\documentclass{ecai} 

\usepackage{latexsym}
\usepackage{amssymb}
\usepackage{amsmath}
\usepackage{amsthm}
\usepackage{booktabs}
\usepackage{enumitem}
\usepackage{graphicx}
\usepackage{color}
\usepackage{hyperref}
\usepackage{subfigure}

\newcommand\leftopen[2]{\ensuremath{(#1,#2]}}
\usepackage{bm}
\usepackage{xcolor}
\usepackage{outlines}
\usepackage{enumerate}

\usepackage{xspace}
\newlist{inlisti}{enumerate*}{1}
\setlist[inlisti]{label=(\roman*)}
\newlist{inlista}{enumerate*}{1}
\setlist[inlista]{label=(\alph*)}

\usepackage{mathtools,nccmath}
\usepackage{dsfont} 
\newcommand{\Ind}[1]{\mathds{1}_{#1}}
\usepackage{textcomp}
\newcommand{\cucbfd}{FCUCB-D\xspace}
\newcommand{\ctsfd}{FCTS-D\xspace}
\newcommand{\cucbfdop}{OP-FCUCB-D\xspace}
\newcommand{\ctsfdop}{OP-FCTS-D\xspace}

\usepackage{multirow}
\usepackage{multicol}
\DeclareMathOperator*{\Ex}{\mathbb{E}}

\newtheorem{theorem}{Theorem}
\newtheorem{lemma}[theorem]{Lemma}

\newtheorem{assumption}[theorem]{Assumption}
\newtheorem{remark}[theorem]{Remark}

\makeatletter
\newcommand{\proofpart}[2]{%
	\par
	\addvspace{\medskipamount}%
	\noindent\emph{Part #1: #2}\par\nobreak
	\addvspace{\smallskipamount}%
	\@afterheading
}
\makeatother

\newcommand{\BibTeX}{B\kern-.05em{\sc i\kern-.025em b}\kern-.08em\TeX}

\usepackage{algorithm}
\usepackage{algorithmic}
\usepackage{hyperref}

\begin{document}


\begin{frontmatter}


\paperid{1024} 


\title{Merit-based Fair Combinatorial Semi-Bandit with Unrestricted Feedback Delays}

\author[1]{\fnms{Ziqun}~\snm{Chen}}
\author[1]{\fnms{Kechao}~\snm{Cai}\thanks{Corresponding Author. Email: caikch3@mail.sysu.edu.cn}}
\author[1]{\fnms{Zhuoyue}~\snm{Chen}}
\author[1]{\fnms{Jinbei}~\snm{Zhang}}
\author[2]{\fnms{John C.S.}~\snm{Lui}}

\address[1]{Sun Yat-sen University, Shenzhen, China}
\address[2]{The Chinese University of Hong Kong, Hong Kong, China}

\begin{abstract}
We study the stochastic combinatorial semi-bandit problem with
unrestricted feedback delays under merit-based fairness constraints.
This is motivated by applications such as crowdsourcing, and online
advertising, where immediate feedback is not immediately available and
fairness among different choices (or arms) is crucial. 
We consider two types of unrestricted feedback delays: reward-independent
delays where the feedback delays are independent of the rewards, and
reward-dependent delays where the feedback delays are correlated with the
rewards. 
Furthermore, we introduce merit-based fairness constraints to ensure a fair
selection of the arms.
We define the reward regret and the fairness regret and present new bandit
algorithms to select arms under unrestricted feedback delays based on their
merits.
We prove that our algorithms all achieve sublinear expected reward regret and
expected fairness regret, with a dependence on the quantiles of the delay
distribution.
We also conduct extensive experiments using synthetic and real-world data
and show that our algorithms can fairly select arms with different feedback
delays.

\end{abstract}
\end{frontmatter}


\section{Introduction}
\label{sec:cmabdelayfairness:introduction}

In the stochastic combinatorial multi-armed bandit (CMAB) problem with
semi-bandit feedback, a learner can select more than one arm at each round and
can receive feedback from each selected arm. 
However, in practice, the feedback is not readily available in many
real-world applications.
For example, consider the task assignment problem in a crowdsourcing platform
where arms represent the workers and feedback (reward) represents the payoff of
selecting a worker.
Each completed task yields a payoff based on the quality of the worker. 
The payoff may be delayed since each task requires a certain amount of time to
complete.
This differs from the typical bandit settings where the learner can receive the
feedback immediately after selecting an arm.
As another example, in online advertising, the customers usually take hours or
even days to make a purchase after clicking an ad~\cite{chapelle2014modeling}. 

In general, the feedback delays in the bandit problems may be
\emph{unrestricted} with unbounded support or expectations.
Previous studies on stochastic delayed bandit problems relied on various
assumptions regarding the delay distribution such as bounded
expectation~\cite{joulani2013online,mandel2015queue}, identical delay
distribution across arms~\cite{vernade2017stochastic}, and the prior knowledge
of delay distribution~\cite{gael2020stochastic}, and none of them can address
unrestricted delays.
In this paper, we consider two different unrestricted delay settings, depending
on the relationship between delays and rewards.
The first is the \emph{reward-independent delay} setting, where the delay of the
feedback from an arm is independent of the reward of the arm.
The second is the \emph{reward-dependent delay}
setting, where the delay of the feedback of each
arm is correlated with the reward of the arm.
The reward-dependent delay is motivated by the applications mentioned earlier:
in crowdsourcing, the time the worker takes to complete the assigned task is
tied to the payoff as tasks with more payoff can take longer to finish; in
online advertising, the delay after collecting the revenue from an ad click
often depends on the purchase price paid by the customer.
Such a setting is challenging as the feedback would provide a biased estimation
of the expected reward of an arm.
Take an arm with a Bernoulli reward as an example. 
If the feedback delay associated with reward $1$ is smaller than the feedback
delay associated with reward $0$, the learner would receive reward $1$ earlier
and more frequently than reward $0$. 
As a result, the observed empirical average reward of the arm would deviate from
the actual mean reward and bias towards reward $1$.
In some cases, the directions of such deviations may be opposite between
different arms.
When the fraction of unobserved feedback is large, the observed empirical
average reward of the good arm may be much smaller than that of the bad arm,
which adds another dimension of complexity to the problem.





In addition, ensuring fairness among the arms is another critical concern in
many bandit problems. 
While existing works mainly focus on maximizing the cumulative rewards, there is
a growing recognition that such a unitary consideration can be problematic as it
ignores the interests of arms, resulting in an unfair selection of
arms~\cite{mansoury2022exposure}.
Consider a bandit algorithm that tries to maximize the reward by assigning tasks
to workers in a crowdsourcing platform, the algorithm will learn which worker
has the highest quality and constantly assign the task to that worker, even if
other workers are almost equally good.
This will result in a winner-takes-all allocation where many skillful workers
will not receive sufficient tasks, and therefore lose interest in the platform. 
Thus, to build a sustainable platform, a good algorithm must ensure
fairness among workers and guarantee that workers with similar skill levels have
similar probabilities of receiving tasks.
Similarly, in online advertising, the ad publishers wish to ensure fairness
among ads and guarantee that all ads have some opportunities to be displayed.
This approach not only enhances the platform's appeal to advertisers but also
sustains a diverse range of content on the website.

\textbf{Main contributions.}
In this paper, we formulate a combinatorial semi-bandit problem to maximize the
cumulative reward while ensuring merit-based fairness among arms with
unrestricted feedback delays.

We define the \emph{merit} of an arm as a function of its expected reward and
impose \emph{merit-based fairness} constraints to ensure each arm is selected
with a probability proportional to its merit under feedback delays.
In particular, we do not make any assumptions on the delay distributions and
allow for unbounded support and expectation of delays.
We propose four different fair algorithms for both reward-independent and
reward-dependent delay settings and define reward regret and fairness regret to
measure their performance.
Specifically, in the reward-independent delay setting, 
we propose an algorithm (\cucbfd)
based on Upper Confidence Bound (UCB) and a computation-efficient algorithm (\ctsfd) based on Thompson Sampling (TS) to
ensure merit-based fairness among arms.
In the more challenging reward-dependent delay setting, we propose~\cucbfdop
and~\ctsfdop algorithms
using both \emph{optimistic and pessimistic estimates} of the delayed unobserved
rewards to accommodate the estimation biases.

We prove that our proposed algorithms all achieve sublinear upper bounds for
both expected fairness regret and expected reward regret, scaling with the quantile of delay
distributions.
We further conduct experiments using synthetic and real-world data. 
Our experiment results show that our algorithms outperform other algorithms by
fairly selecting arms according to the merits of the arms while maximizing the
cumulative reward under different types of feedback delays.

\section{Related work}
\label{sec:cmabdelayfairness:related-work}
The CMAB problems have been extensively
studied~\cite{slivkins2019introduction,lattimore2020bandit}.
Many works extend the combinatorial semi-bandit to various settings, such as
general nonlinear reward~\cite{chen2016combinatorialb}, probabilistically
triggered arms~\cite{chen2016combinatoriala,wang2017improving}, etc.
Their algorithmic designs either follow the principle of optimism in the face of
uncertainty such as the UCB
algorithm~\cite{auer2002finite}, or posterior sampling such as the TS algorithm~\cite{thompson1933likelihood}.

\textbf{Delayed bandit} Delayed feedback has drawn lots of attention since
\citet{dudik2011efficient} first introduced it in
stochastic bandit problems. 
Most studies make various assumptions on the delay distributions.
For instance, \citet{joulani2013online} explore
the impact of delay in both the stochastic and adversarial settings under the
assumption that the expectations of the delays are bounded. 
\citet{mandel2015queue} develop a bandit model with
bounded delays.
Besides, \citet{vernade2017stochastic} study
the delayed bandit with partially observable feedback, where the learner cannot
differentiate between the non-received reward and the zero reward.
They assume that delays are the same for all arms and have a bounded
expectation.
\citet{gael2020stochastic} also consider partially
observed feedback and study heavy-tailed delay distributions which might have infinite
expectations. 
Nevertheless, they assume the parameter of delay distributions is known to the learner.
There has also been an emerging interest in bandit problems with unrestricted
delays.
The recent work \cite{wu2022thompson} proves a sublinear regret upper bound for the TS
algorithm with arbitrary delay distributions.
\citet{lancewicki2021stochastic}
introduce reward-dependent feedback delays and design algorithms based on successive
elimination with no fairness concerns.

\textbf{Bandit with fairness constraints} \citet{joseph2016fairness, joseph2018meritocratic} study fairness learning in
bandit problems, introducing the notion of meritocratic fairness, where a better
arm is always no less likely to be selected than a worse arm.
However, their approach favors the arm with the highest expected reward and
ignores the merits of other arms.
\citet{schumann2022group} partition arms into
groups based on specific features. 
They introduce a group fairness notion, preventing the learner from favoring one
arm over another based on group information.
Other studies~\cite{li2019combinatorial,patil2021achieving,steiger2022learning}
investigate fairness guarantees in bandit problems to ensure that each arm must
be selected for a predetermined required fraction over all rounds.
\citet{liu2017calibrated} impose a smoothness constraint to achieve calibrated
fairness where the probability of selecting an arm equals the probability of it
having the highest reward. 
Our model subsumes their setting by introducing a more general merit function,
with the objective of guaranteeing that each arm receives a selection fraction
proportional to its merit.
This concept of merit-based fairness has been explored in the single-play
bandit~\cite{wang2021fairness} and combinatorial contextual
bandit~\cite{jeunen2021top} where the goal is to ensure that similar arms obtain
comparable treatment.

Our work differs from previous studies by considering two types of unrestricted
feedback delays, namely, reward-independent delays and reward-dependent delays,
in combinatorial semi-bandit bandit problems.
Moreover, our algorithms not only ensure the maximization of cumulative reward
but also guarantee the selection of each arm with a probability proportional to
its merit, all without assuming any specific delay distributions.

\section{Fair CMAB with General Feedback Delays}
\label{sec:cmabdelayfairness:model}
Let $[K]:=\{1,2,...,K\}$ denote the set of $K$ arms and $[T]:=\{1,2,...,T\}$.
A learner will interact with the arms sequentially over $T$ rounds.
At each round $t\in[T]$, each arm $a\in [K]$ is associated with:
\begin{inlisti}
\item a reward $R_{t,a}\in[0,1]$ that follows an unknown distribution
$\nu_{a}$ with mean $\mu_{a}$;
\item an unknown delay $D_{t,a} \in \mathbb{N}$ such that the reward of arm
$a$ can only be revealed to the learner at the end of the round $t+D_{t,a}$.
\end{inlisti}
At round $t$, the learner selects a subset $A_t$ of $L$ $(L\le K)$
arms from $[K]$
receives possibly delayed feedback $Y_{t,a}$ 
from each arm $a\in [K]$.
Essentially, $Y_{t,a}$ is the aggregated rewards from arm $a$ in previous rounds
and can be expressed as follows:
\begin{equation}
	\label{eq:received-feedback}
	Y_{t,a}=\sum_{s=1}^{t}R_{s,a}\Ind{\{ D_{s,a} = t-s \}}\Ind{ \{a \in A_{s} \}},
\end{equation}
where $\Ind{\{ \cdot \}}$ is the indicator function.
The term $\Ind{\{ D_{s,a} = t-s \}}$ in~\eqref{eq:received-feedback} takes
account of the delays $D_{s,a}$ for $s\le t$.
We note that neither the delay $D_{s,a}$ nor the round number $s$ (the original
time of the reward) can be deduced from the feedback $Y_{t,a}$.  
Let $N_{t,a}=\sum_{s:s<t} \Ind{\{a \in A_{s}\}}$ denote the number of rounds
that arm $a$ has been selected up to round $t-1$, and $M_{t,a}=\sum_{s:s+D_{s,a}
< t} \Ind{ \{a \in A_{s}\}}$ denote the number of delayed feedbacks that the
learner can receive from arm $a$ up to round $t-1$.
As the feedback may be delayed, we have $M_{t,a} \leq N_{t,a}$.
Thus, at the beginning of round $t$, the empirical average reward of arm $a$ can be
expressed as:
$\hat{\mu}_{t,a}=\frac{1}{M_{t,a}\vee 1}\sum_{s^{\prime}:s^{\prime} <
t}Y_{s^{\prime},a}$, where $m\vee n = \max\{m, n\}$.
Note that we do not assume that the delays follow any
particular distribution and even allow $D_{t,a}$ being infinite, in which case
the reward from arm $a$ would never be received.
Specifically, we introduce a quantile function to describe the distribution of
the delays for each arm. 
For an arm $a$ with a delay $D_{a}$, we define the quantile function $d_{a}(q)$ as
\begin{equation}
d_{a}(q) = \min \left\lbrace \zeta \in \mathbb{N} \;|\; \mathbb{P} [D_a \leq \zeta] \geq q  \right\rbrace,
\end{equation}
where the quantile $q \in \leftopen{0}{1}$ and $d^{*}(q) = \max_a{d_a(q)}$.

Finally, we consider a merit function $f(\cdot) > 0$ that maps the expected
reward of an arm to a positive merit value. We have two assumptions on the merit
function $f(\cdot)$.
\begin{assumption}
 \label{Minimum-Merit}
The merit of each arm is bounded such that 
\begin{inlisti}
\item 
$\exists$ $\lambda > 0$ and $ \min_{\mu}f(\mu) 
 \geq \lambda $, \label{Minimum-Merit:minmu}
\item 
$\forall \mu_1,\mu_{2} \in [0,1] , \frac{f(\mu_1)}{f(\mu_{2})} \leq
\frac{K-1}{L-1}$ for $L>1$.
\label{Minimum-Merit:boundedrange}
\end{inlisti}
\end{assumption}

\begin{assumption}
	\label{Lipschitz-Continuity}
	The merit function $f$ is $M$-Lipschitz continuous, 
	i.e., there exists a positive constant $M > 0$, such that $\forall \mu_1,\mu_2 \in [0,1], \left| f(\mu_1)-f(\mu_2) \right| \leq M\left|\mu_1-\mu_2  \right| $.
\end{assumption}

To ensure merit-based fairness among the arms, we enforce a constraint
that the probability $p_{a}$ of selecting arm $a$ is proportional to its merit
$f(\mu_{a})$. Formally, we have 
\begin{equation}
\label{fairness-constraint}
	\frac{p_a}{f(\mu_a)}=\frac{p_{a^{\prime}}}{f(\mu_{a^{\prime}})}, \quad \forall a\neq a^{\prime}, a,a^{\prime}\in[K].
\end{equation}
Fairness criteria in various applications can be tailored by selecting different
$f(\cdot)$.
For instance, setting $f(\cdot)$ as a threshold function would grant higher
merits to arms whose expected rewards exceed a predefined threshold.

We now show that there is a unique optimal fair policy that fulfills the
fairness constraints in~\eqref{fairness-constraint} in the following theorem.
\begin{theorem}
\label{the:Optimal-Fair-Policy}
For any $\mu_a, a\in[K]$ and any choice of merit function $f(\cdot) > 0$, there
exist a unique optimal fair policy $\bm{p}^*=\left\lbrace p^*_1,p^*_2,...,p^*_K
\right\rbrace$ such that
\begin{equation}
\label{eq:Optimal-Fair-Policy}
p_a^*=\frac{Lf(\mu_a)}{\sum_{a^{\prime}=1}^K f(\mu_{a^{\prime}})}, \quad \forall a \in [K],
\end{equation}
that satisfies the merit-based fairness constraints
in~\eqref{fairness-constraint}.
\end{theorem} We refer the interested readers to Appendix A for the
proofs of all the theorems.
Theorem~\ref{the:Optimal-Fair-Policy} implies that the optimal fair policy is no
longer selecting a fixed optimal set of $L$ arms as in classical bandit
problems, but a probability distribution on all the possible sets $A_t\subseteq [K]$,
$|A_t|=L$. 
To be more specific, we characterize an arm selection algorithm with a
probabilistic selection vector $\bm{p}_t=\{p_{t,1},p_{t,2},...,p_{t,K} \}$
where $p_{t,a} \in [0,1]$ is the probability of selecting arm $a \in [K]$ at round $t$,
and $\sum_{a=1}^Kp_{t,a}=L$ since only $L$ arms can be selected at each round.
To measure the gap of cumulative reward between the optimal fair policy and a
bandit algorithm, we define the \emph{reward regret} of an algorithm as follows:
\begin{equation}
\label{eq:reward-regret}
\mathrm{RR}_T=\sum_{t=1}^{T}\max\left\lbrace \sum_{a=1}^Kp^{*}_a\mu_{a}-\sum_{a=1}^K p_{t,a}\mu_{a}, 0\right\rbrace .
\end{equation}
We use the reward regret to quantify the speed of reward optimization of an
algorithm.
Specifically, we only consider the non-negative part at each round
in~\eqref{eq:reward-regret} as a less fair algorithm could yield a larger reward
than the optimal fair policy and cause a negative reward gap.
Moreover, we also require a measure to quantify its fairness guarantee.
In this work, we define the \emph{fairness regret} that measures the cumulative
1-norm distance between the optimal fair policy $\bm{p}^{*}$ and the selection
vector $\bm{p}_{t}$ of an algorithm as follows:
\begin{equation}
	\label{eq:fairness-regret}
	\mathrm{FR}_T= \sum_{t=1}^{T}\sum_{a=1}^K|p_a^*-p_{t,a}|.
\end{equation}
The fairness regret measures the overall violation of the merit-based fairness
constraints. 
Our objective is to design algorithms that have both \emph{sublinear expected
reward regret} and \emph{sublinear expected fairness regret} with respect to the
number of rounds $T$, where the expectations are taken over the randomness in
both the arm selections and the rewards. 
By doing so, we can approach the optimal fair policy and maximize the cumulative
reward while ensuring merit-based fairness among all the arms in the long run.

It is important to point out that both Assumption~\ref{Minimum-Merit} and
Assumption~\ref{Lipschitz-Continuity} are necessary for designing bandit
algorithms as stated in the following theorem and remark.
\begin{theorem}
\label{the:Lower-bound-fr-without-Assumption}
For any bandit algorithm, if either
Assumption~\ref{Minimum-Merit}~\ref{Minimum-Merit:minmu} or
Assumption~\ref{Lipschitz-Continuity} does not hold, the lower bound of the
fairness regret is linear; in other words,
there exists a CMAB instance with
linear expected fairness regret $O(T)$.
\end{theorem}

\begin{remark} Assumption~\ref{Minimum-Merit}~\ref{Minimum-Merit:boundedrange}
ensures that the selection probability $p_{t,a}$ in the form of
$Lf(\cdot)/\sum_{a=1}^Kf(\cdot)$ is constrained in $[0,1]$.
\end{remark}

In the following sections, we introduce fair bandit algorithms under two types of
feedback delays, \emph{reward-independent feedback delays} and
\emph{reward-dependent feedback delays}.

\section{Algorithms for Reward-independent Delays}
\label{sec:furthercmabdelayfairness:algorithm}
In this section, we first consider that the feedback delays are independent of
the rewards of arms. 
We design two bandit algorithms to ensure merit-based
fairness under the reward-independent delays.

\subsection{\cucbfd Algorithm}
\label{ssec:fdts:cucbfd-algorithm}

Algorithm~\ref{alg:fairness-delay-ucb-type} shows the details of our \emph{Fair CUCB
with reward-independent feedback Delays} (\cucbfd) algorithm, which follows the
principle of optimism in the face of uncertainty without requiring any prior
knowledge of delay distributions.
At each round $t$, we first calculate the average rewards of all arms based on
the received delayed feedback.
With the average rewards, we construct a confidence region $\mathcal{C}_t$ (see
Line~\ref{alg:fairness-delay-ucb-type:confidence_region}) using both UCB estimates $U_{t,a}$ and LCB
(Lower Confidence Bound) estimates $B_{t,a}$ of all arms, where the vector
$\tilde{\bm{\mu}}:=(\tilde{\mu}_a)_{a\in[K]}$ and $c_{t,a}$ denotes the
confidence radius of each arm $a$.
We clip $U_{t,a}$ to $1$ and $B_{t,a}$ to $0$ since the rewards have support on
$[0, 1]$.
Then we find a vector $\tilde{\bm{\mu}}_t$ in the confidence region
$\mathcal{C}_{t}$ that maximizes the expected reward of a fair policy as shown in Line~\ref{alg:fairness-delay-ucb-type:mu_t}.
Specifically,
according to Theorem~\ref{the:Optimal-Fair-Policy}, we construct the probability of selecting arm $a$ as
$\frac{Lf(\tilde{\mu}_a)}{\sum_{a^{\prime}=1}^Kf(\tilde{\mu}_{a^{\prime}})}$ to satisfy the
merit-based fairness constraints, which is limited to the interval
$[0,1]$ under Assumption~\ref{Minimum-Merit}~\ref{Minimum-Merit:boundedrange}.
Different from the conventional bandit algorithms such as
CUCB~\cite{chen2013combinatorial} which deterministically selects $L$ arms at
each round, our algorithm selects $L$ arms stochastically with the selection
vector $\bm{p}_t$ to ensure fairness.
In particular, we incorporate a randomized rounding scheme (RRS)
from~\cite{gandhi2006dependent}. 
RRS takes a probabilistic selection vector $\bm{p}_t$ ($\sum_{a=1}^Kp_{t,a}=L$) as
input and generates a set of arms $A_t$ such that $\Ex[\Ind{ \{a\in A_{t}\} }]=p_{t,a}$.
Finally, we receive delayed feedback from all arms.

We present the expected fairness regret and reward regret upper
bounds of FCUCB-D in the following theorem.

\begin{theorem}
	\label{the:fairness-reward-regret-ucb-type}
    Suppose that 
    $\forall t > \lceil K/L \rceil,  a\in[K], R_{t,a}\in [0,1]$ and feedback
    delays are reward-independent.
Set $c_{t,a}=\sqrt{\frac{ \log (4LKT)}{M_{t,a} \vee 1}}$.
When $T>K$, 
	the expected fairness regret of \cucbfd is upper bounded as:
	$$\mathbb{E}\left[\mathrm{FR}_T\right] =\widetilde{O}\left( \min_{q\in (0,1]}\left\lbrace \frac{ML}{\lambda}\left( \frac{K}{q}\sqrt{T}+Ld^{*}(q)\right) \right\rbrace \right),$$ 
and the expected reward regret of \cucbfd is upper bounded as:
	$$  \mathbb{E}\left[\mathrm{RR}_T\right]=\widetilde{O}\left( \min_{q\in (0,1]}\left\lbrace \frac{K}{q}\sqrt{T}+Ld^{*}(q)\right\rbrace \right),$$
where $\widetilde{O}$ hides the polylogarithmic factors in $T$.
\end{theorem}

In Theorem~\ref{the:fairness-reward-regret-ucb-type}, the factor
$\tfrac{ML}{\lambda}$ in $\mathbb{E}\left[\mathrm{FR}_T\right]$ comes from
Assumption~\ref{Minimum-Merit} and Assumption~\ref{Lipschitz-Continuity} on the
merit function $f(\cdot)$.
Note that both upper bounds are valid for any quantile $q\in\leftopen{0}{1}$,
and one can optimize the bounds by selecting the optimal $q$. 

Our~\cucbfd algorithm differs from the single-played FairX-UCB
algorithm~\cite{wang2021fairness} as it addresses a more challenging
combinatorial semi-bandit problem involving feedback delays.
Moreover, our theoretical results accommodate
unbounded delays since
the upper bounds depend on the quantiles of the delay distribution instead of
the expectation of the delays as in~\cite{joulani2013online,pike2018bandits,steiger2022learning}.

\begin{algorithm}[!tb]
\caption{\textbf{F}air \textbf{CUCB} with reward-independent feedback \textbf{D}elays (\cucbfd)}
\label{alg:fairness-delay-ucb-type}
\textbf{Input}: $f(\cdot)$, $T$, $L$, $K$\\
\textbf{Init}: \makebox[0pt][l]{Select each arm in $[K]$ once with $\lceil K/L \rceil$ rounds.} 
\begin{algorithmic}[1]
\FOR{$t= \lceil K/L \rceil +1 $ to $T$}
\FOR{$a \in [K]$}
\STATE $ M_{t,a}=\sum_{s:s+D_{s,a} < t} \Ind{ \{a \in A_{s}\}}$
\STATE $\hat{\mu}_{t,a}=\frac{1}{M_{t,a}\vee 1}\sum_{s^{\prime}:s^{\prime} <
t}Y_{s^{\prime},a}$
\STATE $U_{t, a}=\min\{\hat{\mu}_{t,a}+c_{t,a}, 1\}$
\STATE $B_{t, a}=\max\{\hat{\mu}_{t,a}-c_{t,a}, 0\}$
\ENDFOR
\STATE \label{alg:fairness-delay-ucb-type:confidence_region} $\mathcal{C}_t = \{ \tilde{\bm{\mu}} | \forall a\in[K], \,  \tilde{\mu}_{a}  \in [B_{t,a}, U_{t,a}]\}$
\STATE \label{alg:fairness-delay-ucb-type:mu_t} $\tilde{\bm{\mu}}_t = \mathop{\arg\max}_{\tilde{\bm{\mu}} \in \mathcal{C}_t} \sum_{a=1}^K \frac{Lf(\tilde{\mu}_a)}{\sum_{a^{\prime}=1}^K f(\tilde{\mu}_{a^{\prime}})} \tilde{\mu}_a$
\STATE \label{alg:fairness-delay-ucb-type:selection_probability}  $\text{Compute } p_{t,a}=\frac{Lf(\tilde{\mu}_{t,a})}{\sum_{a^{\prime}=1}^K  f(\tilde{\mu}_{t, a^{\prime}})} \text{ for } a\in [K] $ 
\STATE Select arms in $A_t= \text{RRS}(L, \bm{p}_t)$
\STATE Receive delayed feedback $Y_{t,a}$ from $a \in [K]$
\ENDFOR
\end{algorithmic}
\end{algorithm}

\subsection{\ctsfd Algorithm}
\label{ssec:fdts:ctsfd-algorithm}

The computational complexity of \cucbfd may be high for a large $K$.
In particular, Line~\ref{alg:fairness-delay-ucb-type:mu_t} in
Algorithm~\ref{alg:fairness-delay-ucb-type} could involve a non-convex
constrained optimization problem, which requires a complex optimization solver
for finding the optimal solution. 

To tackle this problem, we incorporate a TS-based method in our
algorithm design and propose the \emph{Fair CTS with reward-independent feedback
Delays} (\ctsfd) algorithm without invoking an optimization solver.
The details of \ctsfd are described in
Algorithm~\ref{alg:fairness-delay-thompson-type}.
Initially, the algorithm starts with a prior distribution
$\bm{\mathcal{Q}}_1:=(\mathcal{Q}_{1,a})_{a\in[K]}$ where $\mathcal{Q}_{1,a}$
represents the learner's prior belief about the reward of arm $a$.
At each round $t$, for each arm $a$, we generate a sample $\tilde{\mu}_{t,a}$ as
the reward estimate from the posterior distribution $\mathcal{Q}_{t,a}$ (see
Line~\ref{alg:fairness-delay-Thompson-type:posterior_sample}) and compute the
selection probability $p_{t,a}$. 
Then we select $L$ arms using the selection probability distribution
$\bm{p}_{t}$ via the \textsc{RRS} described in
Algorithm~\ref{alg:fairness-delay-ucb-type}. 
Finally, we update the posterior distribution $\bm{\mathcal{Q}}_{t}:=\left(Q_{t,a}\right)_{a\in[K]}$ using the
received delayed feedback $\bm{Y}_{t}:=\left(Y_{t,a}\right)_{a\in[K]}$ at
Line~\ref{alg:fairness-delay-thompson-type:update}.

Based on the Bayesian setting and given the prior reward distributions, we
derive the following theorem on the expected fairness/reward regret of \ctsfd.

\begin{theorem}
	\label{the:fairness-reward-regret-thompson-sam-type}
	$\forall a\in[K]$, given a uniform prior on $\mu_{a}$ and suppose that $\forall t\in [T]$,
    $R_{t,a}$ is Bernoulli distributed
    and the feedback delays are
    reward-independent. 
When $T>K$,
the expected fairness regret of \ctsfd is upper bounded as: $$  \mathbb{E}\left[\mathrm{FR}_T\right]=\widetilde{O}\left( \min_{q\in (0,1]}\left\lbrace \frac{ML}{\lambda}\left( \frac{K}{q}\sqrt{T}+Ld^{*}(q)\right) \right\rbrace \right),$$
and the expected reward regret of \ctsfd is upper bounded as:
$$ \mathbb{E}\left[\mathrm{RR}_T\right]  = \widetilde{O}\left( \min_{q\in (0,1]}\left\lbrace \frac{K}{q}\sqrt{T}+Ld^{*}(q)\right\rbrace \right),$$
where $\widetilde{O}$ hides the polylogarithmic factors in $T$.
\end{theorem}

\begin{algorithm}[!tb]
	\caption{\textbf{F}air \textbf{CTS} with reward-independent feedback \textbf{D}elays (\ctsfd)}
	\label{alg:fairness-delay-thompson-type}
	\textbf{Input}:  $f(\cdot)$, $T$, $L$, $K$, $\bm{\mathcal{Q}}_1$
	\begin{algorithmic}[1]
		\FOR{$t=1$ to $T$}
  \FOR{$a \in [K]$}
		\STATE Sample $\tilde{\mu}_{t,a}$ from posterior $\mathcal{Q}_{t,a}$
    \label{alg:fairness-delay-Thompson-type:posterior_sample}
		\STATE Compute $p_{t,a}=\frac{Lf(\tilde{\mu}_{t,a})}{\sum_{a^{\prime} =1}^K f(\tilde{\mu}_{t, a^{\prime}})}$
  \ENDFOR
		\STATE Select arms in $A_t=$ \textsc{RRS}$(L, \bm{p}_t)$
		\STATE Receive delayed feedback $Y_{t,a}$ from $a \in [K]$
		\STATE $\bm{\mathcal{Q}}_{t+1} =$ Update$(\bm{\mathcal{Q}}_{t},\bm{Y}_t)$ \label{alg:fairness-delay-thompson-type:update}
		\ENDFOR
\end{algorithmic}
\end{algorithm}

Note that the expected fairness regret and reward regret upper bounds of the
\ctsfd are in the same order as the expected fairness/reward regret of the
\cucbfd.
They are also dependent on the quantiles of the delay distribution.
Nevertheless, \ctsfd avoids solving the optimization problem by using the
Bayesian posterior sampling method, thus it is more computationally efficient
than \cucbfd.

\section{Algorithms for Reward-dependent Delays}
\label{sec:fdts:reward-depend-feedb}

We now consider a more challenging reward-dependent delay setting where the
feedback delay of each arm is correlated with the received reward at the same
round.
In other words, the two random variables are drawn from a joint distribution
over delays and rewards.
Then we propose another two bandit algorithms to maximize the cumulative reward
and ensure merit-based fairness among arms under the reward-dependent feedback
delays.

\subsection{\cucbfdop Algorithm}

In the reward-dependent delay setting, the key challenge arises as the empirical
average reward of each arm is no longer an unbiased estimator of the expected
reward.
This issue occurs when the feedback delays associated with high rewards
distribute differently from the feedback delays associated with low rewards.
Thus, the empirical average rewards would be quite different from the actual
expected rewards, given that high rewards and low rewards are received with
differently distributed delays.
In this context, our previous algorithms, \cucbfd and \ctsfd, that require
unbiased reward estimates, are no longer applicable.

To address such biases in the delayed feedback, we introduce a novel variant of
\cucbfd, named \emph{Optimistic-Pessimistic Fair CUCB with reward-dependent
feedback Delays} (\cucbfdop), detailed in
Algorithm~\ref{alg:fairness-delay-op-ucb-type}.
We leverage both the observed rewards and the optimistic-pessimistic estimates
of delayed unobserved rewards.
Specifically, in calculating the UCB of an arm, we adopt optimistic estimates,
assuming all delayed unobserved rewards attain the maximal value of $1$ at
Line~\ref{alg:fairness-delay-ucb-type-op:o}.
Conversely, in calculating the LCB, we adopt pessimistic estimates,
presuming all the delayed unobserved rewards are the minimal value of $0$ at
Line~\ref{alg:fairness-delay-ucb-type-op:p}.
Subsequently, we construct an expanded confidence region $\mathcal{C}^{\pm}_t$
using the optimistic UCB $U^+_{t, a}$ and pessimistic LCB $B^-_{t, a}$ of all arms at
Line~\ref{alg:fairness-delay-ucb-type-op:confidence_region}.
This approach ensures that the actual expected reward of an arm falls within the
expanded confidence region with high probability.

\begin{algorithm}[!tb]
	\caption{\textbf{O}ptimistic-\textbf{P}essimistic \textbf{F}air \textbf{CUCB} with reward-dependent feedback \textbf{D}elays (\cucbfdop)}
	\label{alg:fairness-delay-op-ucb-type}
	\textbf{Input}: $f(\cdot)$, $T$, $L$, $K$\\
	\textbf{Init}: \makebox[0pt][l]{Select each arm in $[K]$ once with $\lceil K/L \rceil$ rounds. }
	\begin{algorithmic}[1]
 		\FOR{$t=\lceil K/L \rceil+1$ to $T$}
     \FOR{$a \in [K]$}
		\STATE $ M_{t,a}=\sum_{s:s+D_{s,a} < t} \Ind{ \{a \in A_{s}\}}$
		\STATE $ N_{t,a}=\sum_{s:s<t} \Ind{\{a \in A_{s}\}}$
		\STATE $ \hat{\mu}^{+}_{t,a}=\frac{N_{t,a}-M_{t,a}}{N_{t,a} }+\frac{1}{N_{t,a} }\sum_{s^{\prime}:s^{\prime} <
t}Y_{s^{\prime},a}$ \label{alg:fairness-delay-ucb-type-op:o}
		\STATE $ \hat{\mu}^{-}_{t,a}=\frac{1}{N_{t,a}} \sum_{s^{\prime}:s^{\prime} <
t}Y_{s^{\prime},a}$ \label{alg:fairness-delay-ucb-type-op:p}
		\STATE $U^+_{t, a}=\min\{\hat{\mu}^+_{t,a}+c_{t,a}, 1\}$ 
		\STATE $B^-_{t, a}=\max\{\hat{\mu}^-_{t,a}-c_{t,a}, 0\}$
  \ENDFOR
		\STATE  $\mathcal{C}^{\pm}_t = \{ \tilde{\bm{\mu}} | \forall a\in[K],   \tilde{\mu}_{a} \! \in  [B^-_{i,t},  U^+_{i,t}\} $  \label{alg:fairness-delay-ucb-type-op:confidence_region}
		\STATE $\tilde{\bm{\mu}}_t = \mathop{\arg\max}_{\tilde{\bm{\mu}} \in \mathcal{C}^{\pm}_t} \sum_{a=1}^K \frac{Lf(\tilde{\mu}_a)}{\sum_{a^{\prime}=1}^K f(\tilde{\mu}_{a^{\prime}})} \tilde{\mu}_a$ \label{alg:fairness-delay-ucb-type-op:max}
		\STATE $\text{Compute } p_{t,a}=\frac{Lf(\tilde{\mu}_{t,a})}{\sum_{a^{\prime}=1}^K  f(\tilde{\mu}_{t, a^{\prime}})} \text{ for } a\in [K] $ 
		\STATE Select arms in $A_t=$ \textsc{RRS}$(L, \bm{p}_t)$
		\STATE Receive delayed feedback $Y_{t,a}$ from $a \in [K]$
		\ENDFOR
	\end{algorithmic}
\end{algorithm}

We present the upper bounds on the expected fairness regret and reward regret of
\cucbfdop in the following theorem.

\begin{theorem}
	\label{the:fairness-reward-regret-op-ucb-type}
	Suppose that $\forall t>\lceil K/L \rceil, a \in [K],$ $R_{t,a}\in [0,1]$ and feedback delays are reward-dependent.
	For any $\delta \in (0,1)$,
	set $c_{t,a}=\sqrt{\frac{ \log (6LKT)}{N_{t,a}}}$. 	
	When $T>K$, 
	the expected fairness regret of \cucbfdop is upper bounded as:
	$$\mathbb{E}\left[\mathrm{FR}_T\right] =\widetilde{O}\left( \min_{q\in (0,1]}\left\lbrace \frac{MLK}{\lambda}\left((1-q)T+ d^{*}(q)\sqrt{T}\right) \right\rbrace \right),$$	
and the expected reward regret of \cucbfdop is upper bounded as:
	$$  \mathbb{E}\left[\mathrm{RR}_T\right] =\widetilde{O}\left( \min_{q\in (0,1]}\left\lbrace L(1-q)T+ Kd^{*}(q)\sqrt{T} \right\rbrace \right),$$
where $\widetilde{O}$ hides the polylogarithmic factors in $T$.
\end{theorem}

Compared to~\cucbfd, the regret analysis for~\cucbfdop is more challenging since
we must consider the entire feedback rather than just the observed ones. 
Moreover, \cucbfdop could have biased estimates of the actual expected reward
using the optimistic-pessimistic estimates, while \cucbfd always has the
unbiased ones.
Therefore, it would be reasonable to expect that \cucbfdop has larger reward
regret and fairness regret than \cucbfd.
In Theorem~\ref{the:fairness-reward-regret-op-ucb-type}, we show the two regret
upper bounds minimized over the quantile $q \in (0,1]$.
In particular, \cucbfdop achieves sublinear expected reward regret and expected fairness regret
upper bounds $O(T^{\kappa})$ by setting the quantile $q\geq 1-T^{\kappa-1}$,
$\kappa<1$.

\subsection{\ctsfdop Algorithm}
To avoid solving the computationally expensive optimization problem in~\cucbfdop
at Line~\ref{alg:fairness-delay-ucb-type-op:max} in
Algorithm~\ref{alg:fairness-delay-op-ucb-type}, we further propose a TS-based algorithm, named \emph{Optimistic-Pessimistic Fair CTS with
reward-dependent feedback Delays} (\ctsfdop), described in
Algorithm~\ref{alg:fairness-delay-op-thompson-type}.

In \ctsfdop, we also consider both the observed rewards and the delayed
unobserved rewards by constructing an optimistic posterior distribution
$\bm{\mathcal{Q}}_{t}^{+}:=(\mathcal{Q}_{t,a}^{+})_{a \in [K]}$ and a
pessimistic posterior distribution
$\bm{\mathcal{Q}}_{t}^{-}:=(\mathcal{Q}_{t,a}^{-})_{a \in [K]}$.
When updating the optimistic posterior, all the delayed unobserved rewards are
treated as the maximal value of $1$.
In updating the pessimistic posterior, all the delayed unobserved rewards are
considered as the minimal value of $0$.
At each round $t$, we sample an optimistic estimate $\tilde{\mu}^{+}_{t,a}$ from
the optimistic posterior $\mathcal{Q}_{t,a}^{+}$, and a pessimistic estimate
$\tilde{\mu}^{-}_{t,a}$ from the pessimistic posterior $\mathcal{Q}_{t,a}^{-}$
for $a\in[K]$.
Using the average of $\tilde{\mu}^{+}_{t,a}$ and $\tilde{\mu}^{-}_{t,a}$, we can
then compute the selection probability for each arm $a$ at
Line~\ref{alg:fairness-delay-Thompson-type-op:selection-pro}.
This equal weighting of the optimistic and pessimistic estimates facilitates the
analysis of the gap between the optimal fair policy and \ctsfdop.

We use the expected fairness regret and the expected reward regret to measure the performance of \ctsfdop. 
We prove the upper bounds of the regrets in the following theorem.

\begin{algorithm}[!tb]
	\caption{\textbf{O}ptimistic-\textbf{P}essimistic \textbf{F}air \textbf{CTS} with reward-dependent feedback \textbf{D}elays (\ctsfdop)}
	\label{alg:fairness-delay-op-thompson-type}
	\textbf{Input:}  $f(\cdot)$, $T$, $L$, $K$, $\bm{\mathcal{Q}}_{1}$
	\begin{algorithmic}[1]
		\FOR{$t=1$ to $T$}
  \FOR{$a \in [K]$}
		\STATE Sample $\tilde{\mu}^{+}_{t,a}$ from optimistic posterior  $\mathcal{Q}_{t,a}^{+}$
        \STATE Sample $\tilde{\mu}^{-}_{t,a}$ from pessimistic posterior $\mathcal{Q}_{t,a}^{-}$
        \label{alg:fairness-delay-Thompson-type-op:posterior_sample}
		\STATE Compute
$p_{t,a}= \frac{Lf\left((\tilde{\mu}^{+}_{t,a}+\tilde{\mu}^{-}_{t,a})/2\right)}{\sum_{a^{\prime}=1}^K f\left((\tilde{\mu}^{+}_{t,a^{\prime}}+\tilde{\mu}^{-}_{t,a^{\prime}})/2\right)}$
\label{alg:fairness-delay-Thompson-type-op:selection-pro}
\ENDFOR
		\STATE Select arms in $A_t=$ \textsc{RRS}$(L, \bm{p}_t)$
		\STATE Receive delayed feedback $Y_{t,a}$ from $a \in [K]$
		\STATE $\bm{\mathcal{Q}}_{t+1}^{+} =$ Update$(\bm{\mathcal{Q}}_{t}, \bm{Y}_t)^+$
        \STATE $\bm{\mathcal{Q}}_{t+1}^{-} =$ Update$(\bm{\mathcal{Q}}_{t}, \bm{Y}_t)^-$
		\ENDFOR
	\end{algorithmic}
\end{algorithm}

\begin{theorem}
	\label{the:fairness-reward-regret-op-thompson-sam-type}
	$\forall a\in[K]$, given a uniform prior on $\mu_{a}$ and suppose that $\forall t\in [T]$,
    $R_{t,a}$ is Bernoulli distributed
    and feedback delays are
    reward-dependent. 	
	The expected fairness regret of \ctsfdop is upper bounded as: $$\mathbb{E}\left[\mathrm{FR}_T\right]=\widetilde{O}\!\left( \min_{q\in (0,1]}\left\lbrace \frac{MLK}{\lambda}  \left( (1-q)T+\frac{d^{*}(q)}{q}\sqrt{T}\right) \right\rbrace \right),$$
 and the expected reward regret of \ctsfdop is upper bounded as:
        $$\mathbb{E}\left[\mathrm{RR}_T\right]=\widetilde{O}\left( \min_{q\in (0,1]}\left\lbrace \frac{MLK}{\lambda}\left( (1-q)T+\frac{d^{*}(q)}{q}\sqrt{T}\right) \right\rbrace \right),$$
        where $\widetilde{O}$ hides the polylogarithmic factors in $T$.
\end{theorem}

According to Theorem~\ref{the:fairness-reward-regret-op-thompson-sam-type},
\ctsfdop achieves sublinear expected reward regret and expected fairness regret upper bounds
$O(T^{\kappa})$ if the quantile $q\geq 1-T^{\kappa-1}$, $\kappa<1$.
Compared to \ctsfd, both the expected reward regret and the expected fairness
regret of \ctsfdop depend on the constants $\lambda$ and $M$ described in
Assumption~\ref{Minimum-Merit} and Assumption~\ref{Lipschitz-Continuity}. 
This is because \ctsfdop does not have the accurate posterior distribution of the
rewards due to the reward-dependent feedback delays, and we derive its expected reward regret from
its expected fairness regret using Assumption~\ref{Minimum-Merit}
and Assumption~\ref{Lipschitz-Continuity}. 
  
\begin{remark} While our~\cucbfdop and~\ctsfdop algorithms are primarily
tailored for the reward-dependent delay setting, they are versatile enough to be
applied to CMAB problems with reward-independent delays. 
However, this application may lead to potentially larger reward regret and fairness
regret due to the biases in the optimistic-pessimistic estimates.
\end{remark}

\section{Experiments}
\label{sec:cmabdelayfairness:experiments}

Here, we conduct experiments\footnote{Source code available at \href{https://github.com/MLCL-SYSU/FairCMAB-Delays}{https://github.com/MLCL-SYSU/FairCMAB-Delays} (Accessed 29-July-2024)} using both synthetic and real-world data to
demonstrate the effectiveness of our algorithms.
We also discuss several interesting observations derived from the experiment
results.  

\textbf{Experiments using synthetic data.}
We consider a CMAB problem with $K=7$ arms where the learner selects $L=3$ arms
at each round. 
The rewards of each arm follow a Bernoulli distribution with mean in
$\bm{\mu}=\{0.3,0.5,0.7,0.9,0.8,0.6,0.4\}$. 
We examine the impact of delays on the expected fairness/reward regret of our algorithms
with several feedback delay settings (i.e., fixed
delays~\cite{dudik2011efficient}, geometric delays~\cite{vernade2017stochastic},
$\alpha$-Pareto delays~\cite{gael2020stochastic}, packet-loss delays and biased
delays~\cite{lancewicki2021stochastic}) considered in prior work.
We use the merit function $f(\mu)=1+2\mu^c$ to calculate the merit of an arm
with expected reward $\mu$ under Assumption~\ref{Minimum-Merit}
and Assumption~\ref{Lipschitz-Continuity}, where the parameter $c$ controls the gradients
of the merit function.
We set $c=4$ in the following and conduct additional experiments using merit
functions with different $c$ in Appendix B in the supplementary material, 
where we find that the regret gap between different bandit algorithms widens as the parameter $c$ increases. Moreover, in Appendix B, we provide the running time for the same number of rounds of \cucbfd and \ctsfd (and their corresponding OP versions) to demonstrate the effectiveness of TS-type algorithms in avoiding solving optimization problems.
All results are averaged over $100$ runs.

We first examine the fairness of different algorithms under the geometric delays
with the success probability parameter equal to $0.05$.
In this case, the feedback delays can be arbitrarily long but the expectation of
the delays is finite.

For comparison, we implement three other CMAB algorithms, CUCB-D, MP-TS-D, and
FGreedy-D which are adapted from CUCB~\cite{chen2013combinatorial},
MP-TS~\cite{komiyama2015optimal} and $\epsilon$-Greedy, respectively, to account
for feedback delays.
In particular, FGreedy-D selects $L$ arms uniformly at random in the exploration
phase and selects $L$ arms with probability
$\tfrac{Lf(\hat{\mu}_{t,a})}{\sum_{a^{\prime}=1}^K f(\hat{\mu}_{t, a^{\prime}})}$ via
\textsc{RRS} in the exploitation phase. 
FairX-UCB and FairX-TS proposed in~\cite{wang2021fairness} are limited in
applicability to our setting since they can only select a single arm at each round without
accounting for feedback delays. 
Additionally, other fair bandit algorithms use different fairness metrics,
making them unsuitable for direct comparison with our algorithms.

\begin{figure}[!t]
	\centering
\subfigure[Arm selection fractions]{
	\centering
	\includegraphics[width=0.47\textwidth]{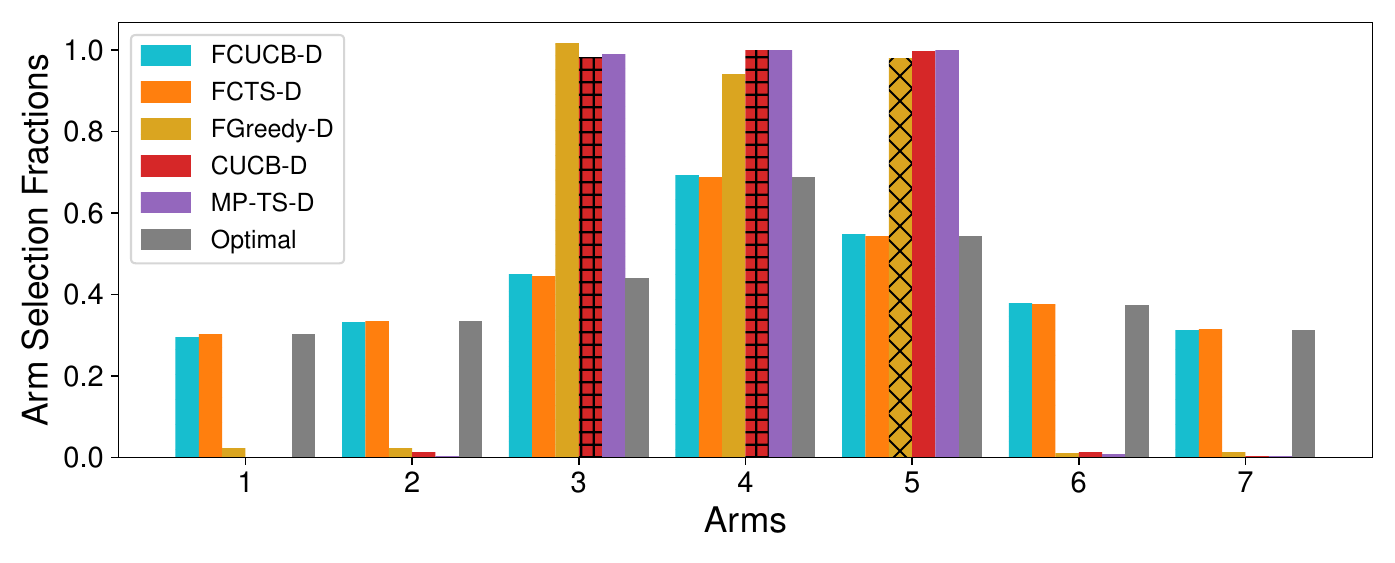}
	\label{fig:average_selection}
	}
	\subfigure[Reward Regret]{
	\centering
	\includegraphics[width=0.23\textwidth]{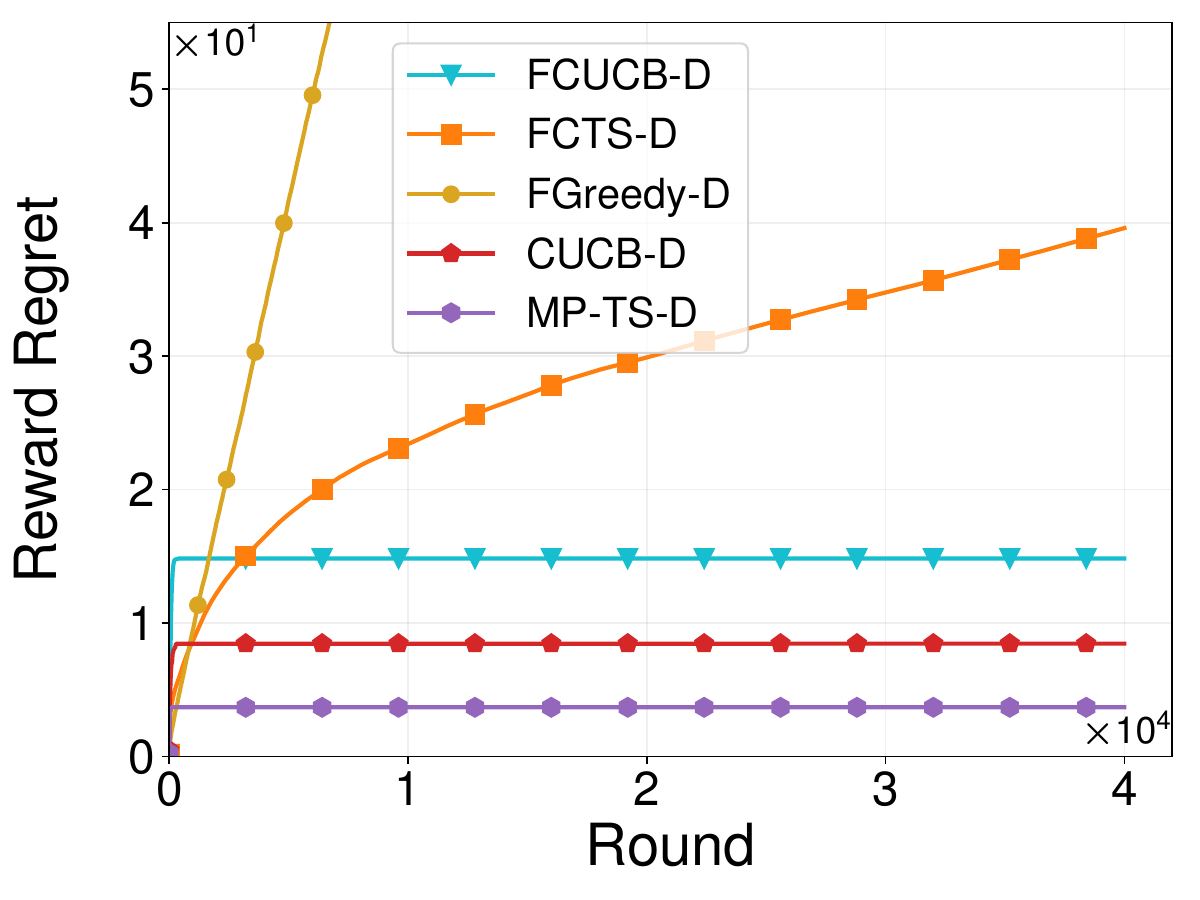}
	\label{fig:RR_GD}
	}
	\subfigure[Fairness Regret]{
		\centering
  		\includegraphics[width=0.23\textwidth]{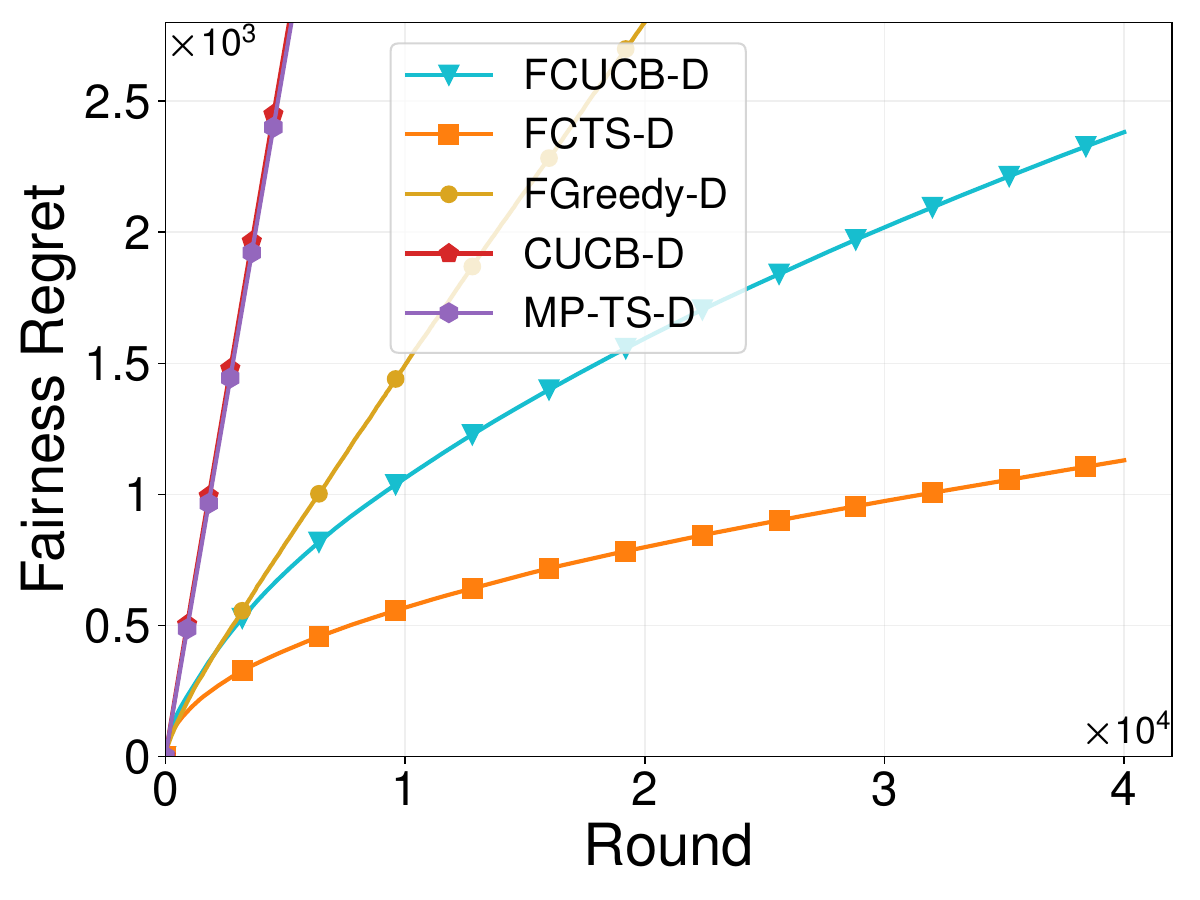}
		\label{fig:FR_GD}
	}
	\caption{Comparison of different bandit algorithms under geometric feedback delays.}
    	\label{fig:experiments_on_geometric_delays}
    \vspace{16pt}
\end{figure}

\begin{figure*}[!t]
	\centering
	\subfigure[Fixed delays]{
		\centering
		\includegraphics[width=0.232\textwidth]{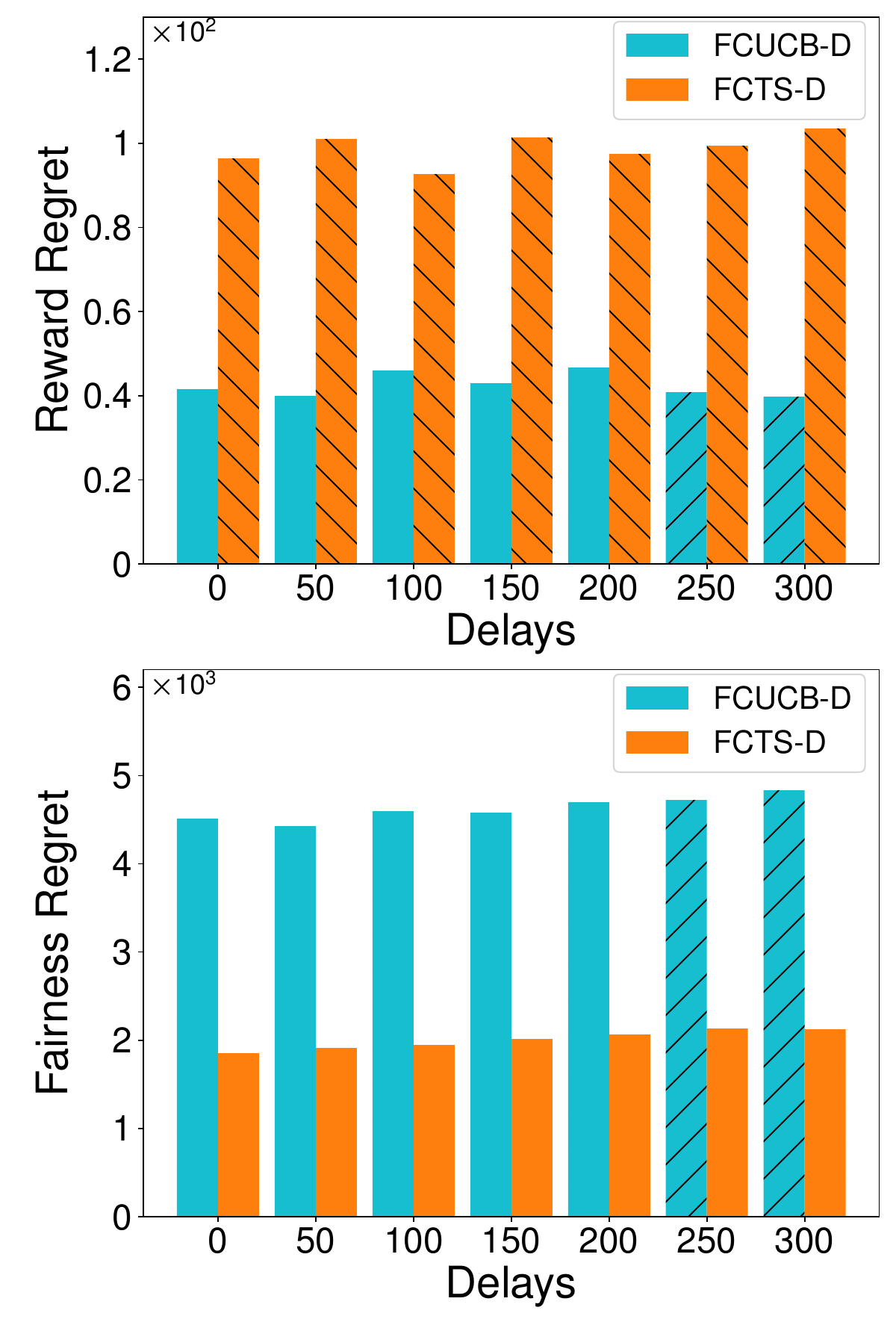}
		\label{fig:fixed_delays}
	}
	\subfigure[$\alpha$-Pareto delays]{
		\centering
		\includegraphics[width=0.232\textwidth]{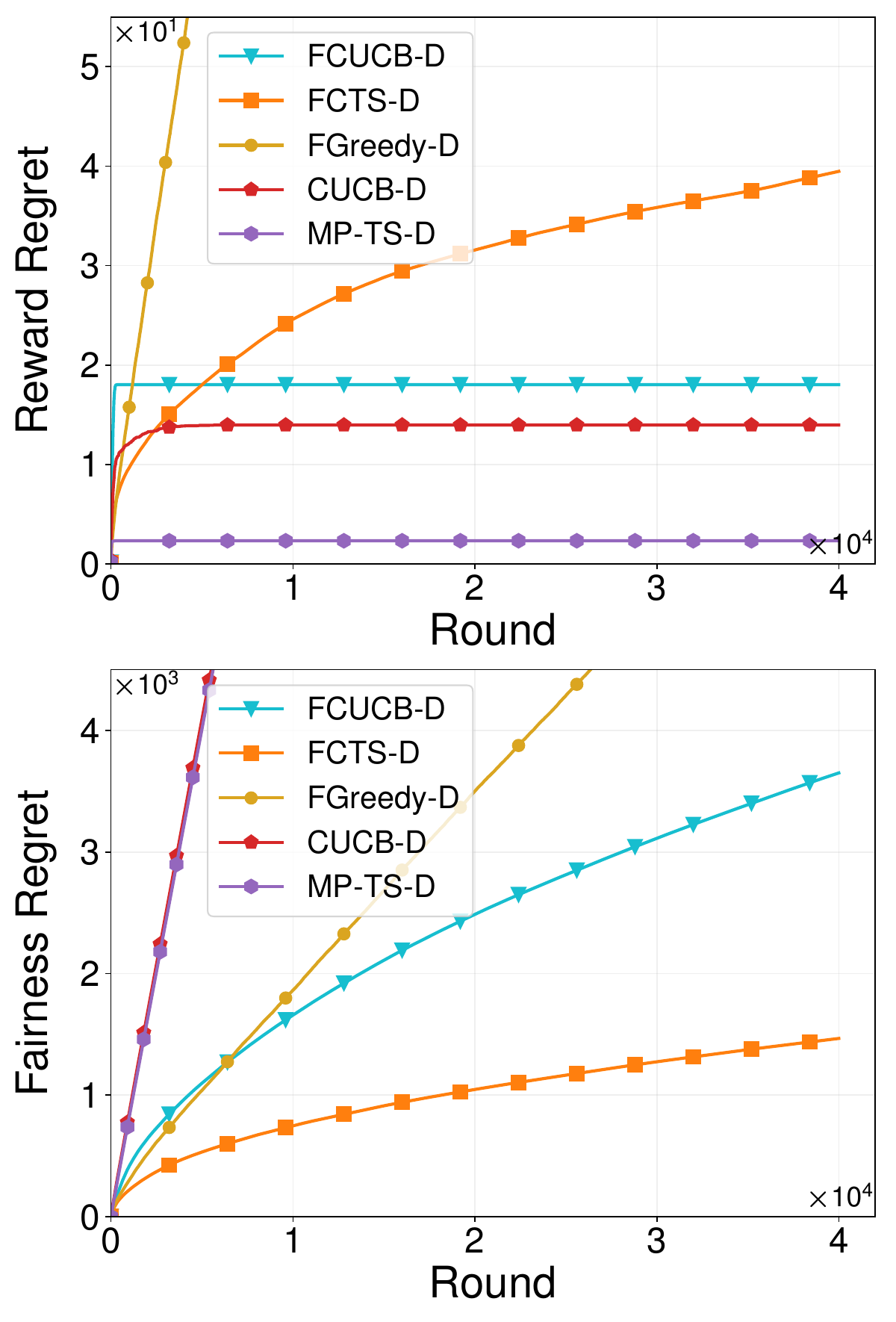}
		\label{fig:pareto_delays}
	}
	\subfigure[Packet-loss delays]{
		\centering
		\includegraphics[width=0.232\textwidth]{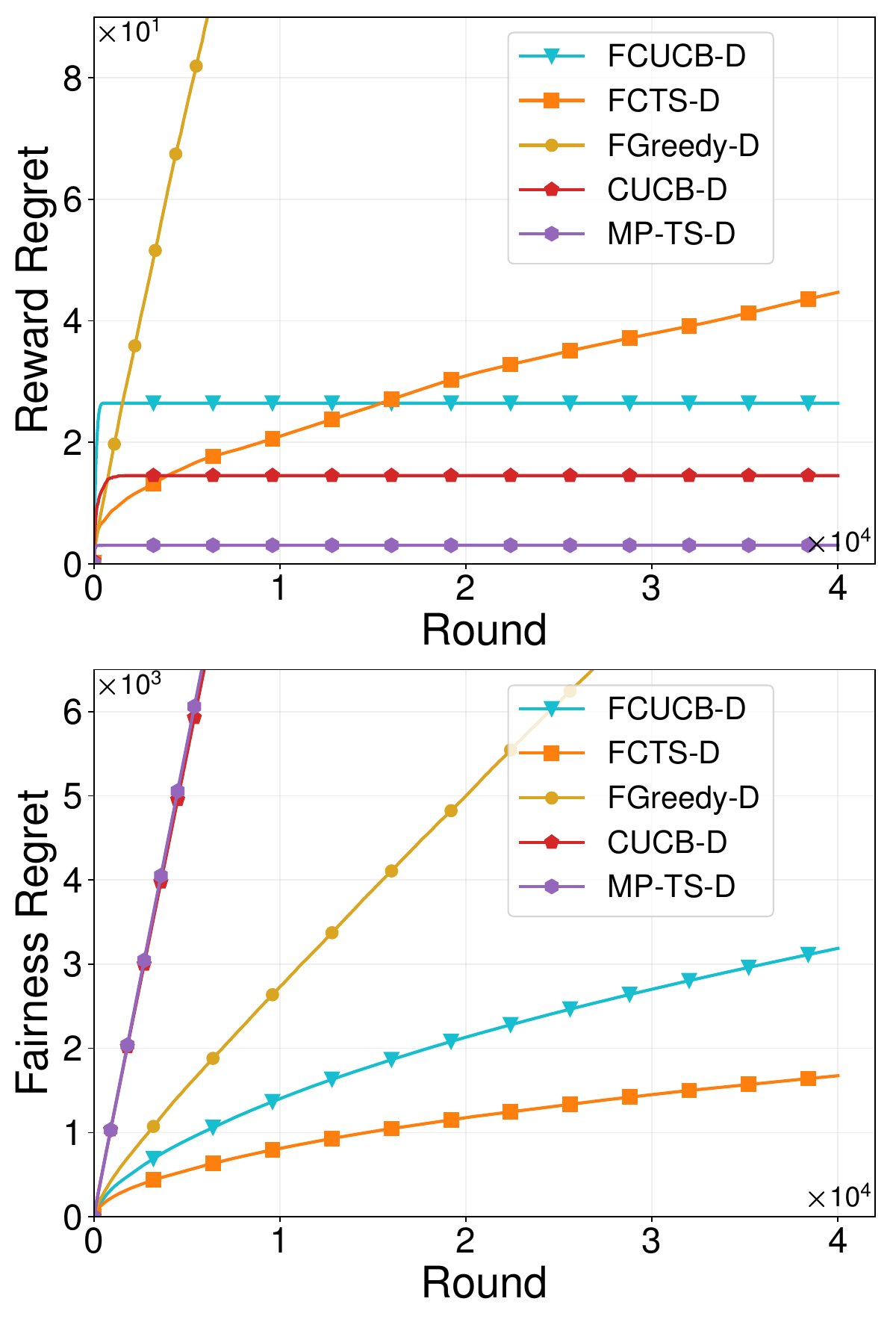}
		\label{fig:packet_loss_delays}
	}	
	\subfigure[Biased delays]{
		\centering
		\includegraphics[width=0.232\textwidth]{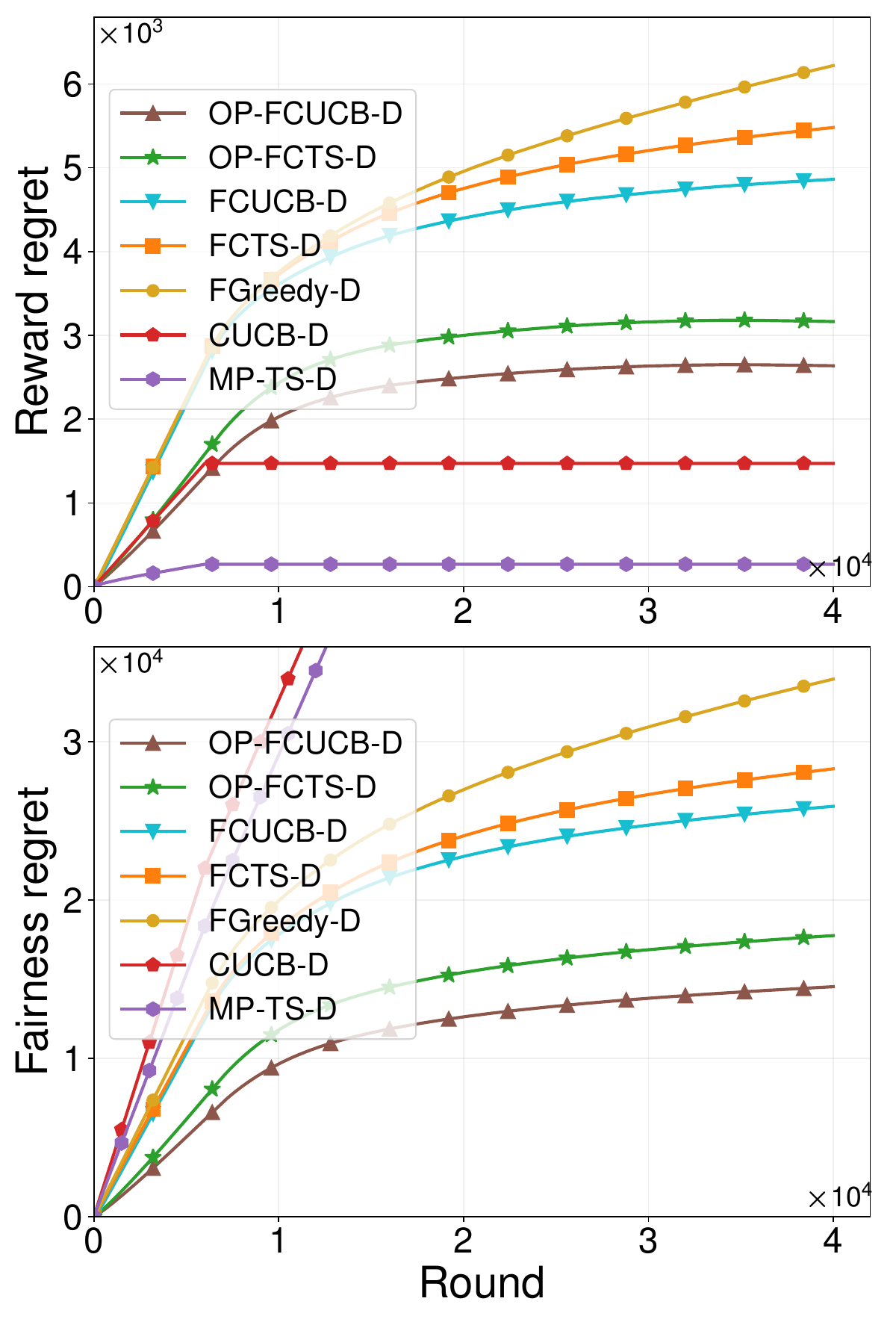}
		\label{fig:baised_delays}              
	}
	\caption{Experiment results of different bandit algorithms under different
		types of feedback delays.}
	\label{fig:experiments_on_different_delays}
     \vspace{10pt}
\end{figure*}

\begin{figure}[!t]
\centering
\subfigure[Reward Regret]{
\centering
\includegraphics[width=0.2302\textwidth]{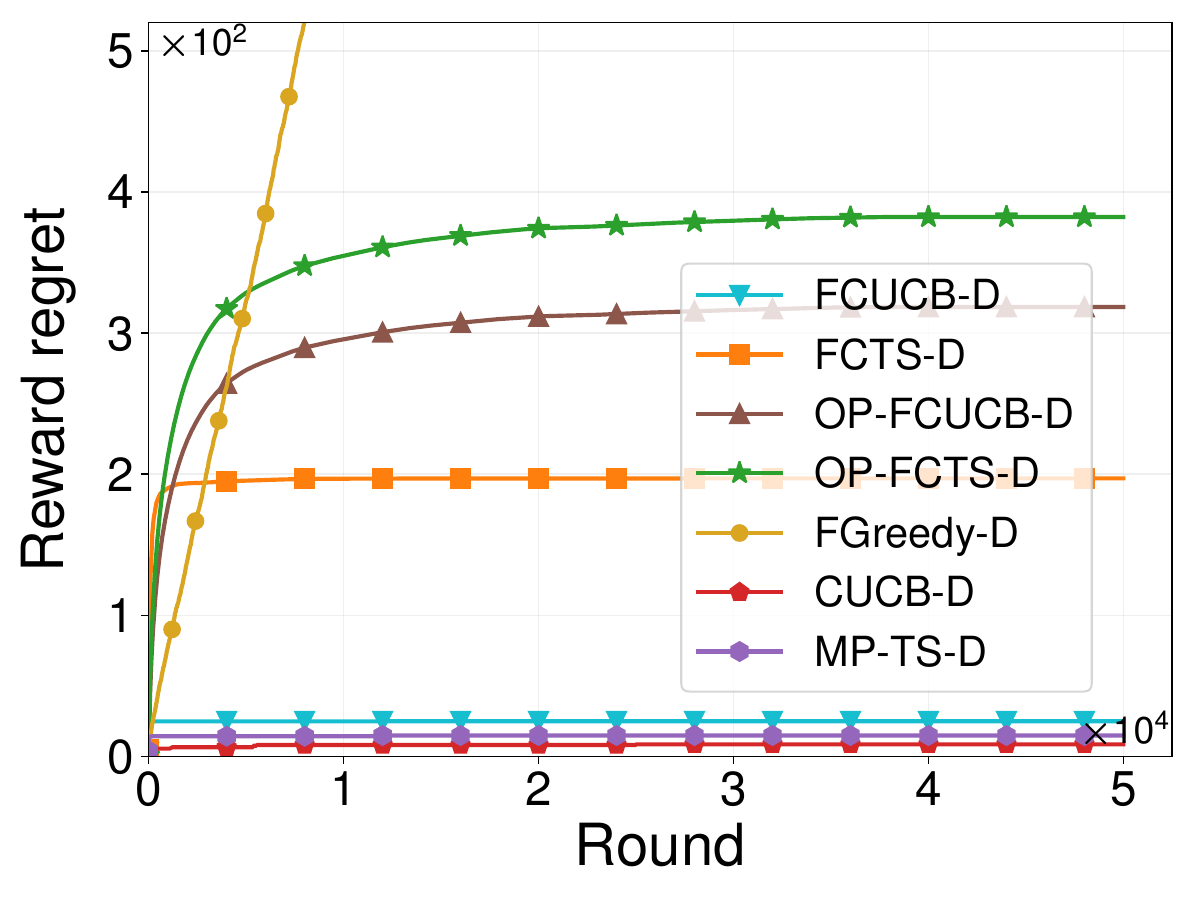}
		\label{fig:RR_RD}
	}
	\subfigure[Fairness Regret]{
		\centering
		\includegraphics[width=0.2302\textwidth]{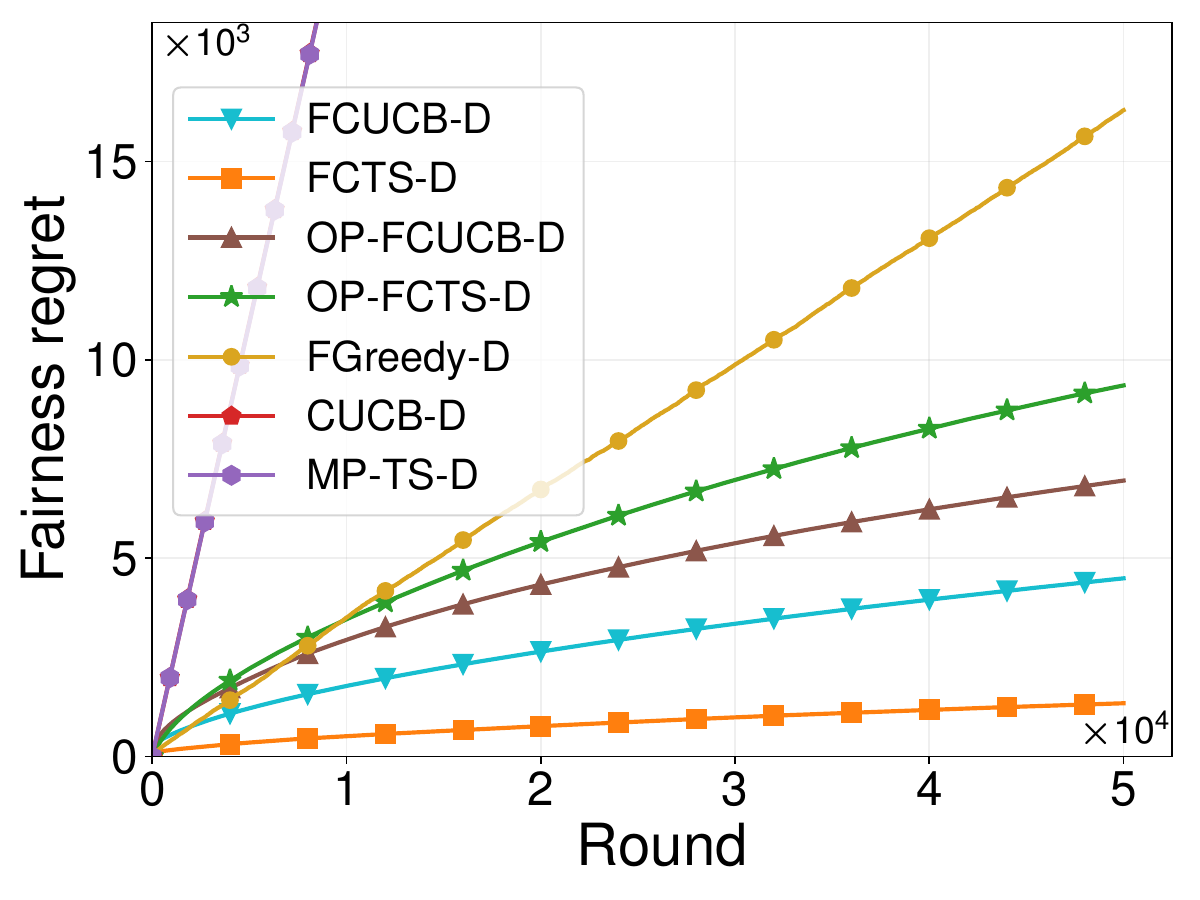}
		\label{fig:FR_RD}
	}
	\caption{Experiment results using the real-world conversion log dataset.}
    \label{fig:experiments_on_dataset}
    \vspace{20pt}
\end{figure}

Figure~\ref{fig:average_selection} illustrates the average arm selection
fractions of CUCB-D, MP-TS-D, FGreedy-D, the optimal fair policy, and our
\cucbfd and \ctsfd.
Each bar corresponds to the fraction of times an arm is chosen over $T=4 \times
10^4$ rounds by a specific algorithm.
As shown in Figure~\ref{fig:average_selection}, CUCB-D and MP-TS-D are
unfair by mainly selecting the arms (arm $3,4,5$) with high rewards, neglecting
the potential merits of other arms. 
FGreedy-D tends to select arms uniformly randomly since it randomly explores the
arms in the exploration phase.
In contrast, both \cucbfd and \ctsfd can converge to the optimally fair policy. 
This observation shows the effectiveness of our algorithms in achieving
merit-based fairness, ensuring that each arm receives a selection allocation
proportional to its merit.

In Figure~\ref{fig:RR_GD} and~\ref{fig:FR_GD}, it is evident that CUCB-D and MP-TS-D consistently
exhibit smaller reward regret and larger fairness regret when
compared to \cucbfd and \ctsfd.
This observation suggests that CUCB-D and MP-TS-D attain high rewards but substantially violate the merit-based fairness constraints.
Moreover, the reward/fairness regrets of our algorithms are smaller than
FGreedy-D and increase sublinearly in $T$, aligning with the bounds we derived
in Theorem~\ref{the:fairness-reward-regret-ucb-type}
and Theorem~\ref{the:fairness-reward-regret-thompson-sam-type}.
In particular, \cucbfd outperforms \ctsfd in reward regret.
However, this advantage comes at the cost of incurring higher fairness regret.

Then we evaluate the performance of different algorithms under different feedback
delay settings by changing the delay distributions.
Figure~\ref{fig:fixed_delays} shows the reward regret and fairness regret of
\cucbfd and \ctsfd under different fixed delays after $T=10^5$ rounds.
As shown in Figure~\ref{fig:fixed_delays}, the reward regrets and the fairness
regrets of our algorithms are quite close under different fixed delays.
The reason is that the algorithms can increase exploration of the merits of all
the arms by receiving the possible delayed rewards from the wrongly selected
arms at each round, and thus they do not incur much regret.
This indicates that our algorithms are not sensitive to fixed delays that are
not excessively large.

Next, we show the fairness/reward regret of different bandit algorithms under $\alpha$-Pareto
delays and packet-loss delays in Figure~\ref{fig:pareto_delays} and
Figure~\ref{fig:packet_loss_delays}, respectively.
These types of delays pose additional challenges as their expected values may be
infinite.
In the case of $\alpha$-Pareto delays, the delays of each arm $a$ follow the
Pareto Type I distribution with the tail index $\alpha_a$. 
A smaller $\alpha_a$ indicates a heavier tail of the delay distribution, and
when $\alpha_a \leq 1$, the delays have an infinite expectation. 
We uniformly sample $\alpha_a$ from the interval $\leftopen{0}{1}$ for each arm
$a$ to model delays with infinite expectations.
In the packet-loss delays, the delay is $0$ with probability $p$ and infinite
otherwise.
We uniformly sample the probabilities $p$ from interval $\leftopen{0.3}{0.8}$
for each arm.
Remarkably, compared to other algorithms, \cucbfd and \ctsfd can achieve both sublinear fairness and reward regret upper bounds across various delay distributions with
infinite expectations.

Finally, we examine the performance of different algorithms under the reward-dependent
(biased) delays.
We note that OPSE~\cite{lancewicki2021stochastic} is also tailored to handle
reward-dependent delays; however, we refrain from comparing it with our
algorithms as it eliminates the bad arms, resulting in substantial fairness
regret. 
We set the reward-dependent delays as follows: the good arms (arm $3,4,5$) have
a fixed delay of $6,000$ rounds for reward $1$ and $0$ round for reward $0$, and
the bad arms (arm $1,2,6,7$) have a fixed delay of $6,000$ rounds for reward $0$
and $0$ round for reward $1$.
In this setting, as the reward $1$ from a bad arm could be received earlier than
the reward $1$ from a good arm, the empirical average reward of a bad arm would
be larger than that of a good arm at the beginning.
In Figure~\ref{fig:baised_delays}, we observe that \cucbfdop and \ctsfdop
significantly outperform \cucbfd and \ctsfd in both reward regret and fairness
regret since \cucbfd and \ctsfd are not aware of the biases in the empirical
average rewards.
This shows the effectiveness of the optimistic-pessimistic estimates in
\cucbfdop and \ctsfdop.

\textbf{Experiments using real-world data.}
We conduct additional experiments on our algorithms using the conversion log
dataset~\cite{tallis2018reacting} that contains data on users' interactions with
a small sample of ads.
Each row in the dataset corresponds to a user clicking on an ad, including a
conversion indicator denoting whether the user makes a purchase after clicking
the ad, as well as the time between the click and the purchase.

We select the top-10 ($K = 10$) clicked ads from the dataset and allocate them
to three regions for ad placement ($L = 3$). 
Each ad is treated as an arm, where a user's click and purchase represent the reward of an arm, and the delay between the click and the purchase serves as the
feedback delay. 
We determine the conversion rate of each ad by normalizing the number of conversions using min-max scaling.
Then we generate the reward for each arm using a Bernoulli distribution, with the
mean given by the corresponding conversion rate.
Since the dataset lacks information on the click rate of ads,
we assumed a click rate of $5\%$ for each ad.
We compute the time between page visits based on this assumed click rate and the
number of ad clicks in the last week provided in the dataset. 
Then we can derive the delay (in page visits) of the purchase by dividing the
time between the click and the purchase by the time between page visits.
We use the merit function of the form $f(\mu) = 1+3.5\mu^c$ with parameter $c=4$ and run
simulations for $T=5 \times 10^4$ page visits. 
All results are averaged over $100$ runs.

Figure~\ref{fig:experiments_on_dataset} shows the experiment results on
fairness/reward regret for different algorithms.
We observe that our algorithms achieve sublinear bounds on fairness/reward
regret and exhibit a better tradeoff between reward and fairness on the
conversion log dataset, in comparison to other algorithms.
In particular, \cucbfd and \ctsfd outperform \cucbfdop and \ctsfdop in terms of
reward regret and fairness regret.
The rationale behind this is that: the dataset only provides the delay of
ads with successful conversions (click and then purchase, reward 1); For an ad
with no successful conversion (click without purchase, reward 0), we determine
its feedback delay by randomly sampling from the delays of the ads with
successful conversions.
This approach makes the 
 ads' feedback delays independent of the ads' rewards.
Thus, in such a reward-independent delay setting, \cucbfdop and \ctsfdop still
take the unobserved feedback of the ads into account and incur larger reward
regret and fairness than \cucbfd and \ctsfd.


\section{Conclusion \& Future Work}
\label{sec:cmabdelayfairness:conclusion}
In this paper, we propose a novel combinatorial semi-bandit setting with
merit-based fairness constraints and two types of unrestricted feedback delays:
reward-independent delays and reward-dependent delays.
We employ UCB, Thompson Sampling, and optimistic-pessimistic estimates and
design novel algorithms that 
achieve both sublinear expected reward regret and sublinear expected
fairness regret.
Our extensive simulation results using both synthetic dataset and real-world
dataset show that our algorithms fairly select arms according to the merits of
the arms under different feedback delays.

For future research, it is interesting to eliminate the assumption that the
learner is aware of the independence/dependence between rewards and delays. 
The goal would be to design a single algorithm capable of accommodating both
reward-independent and reward-dependent delays.
Another interesting direction is to derive the matching lower bounds of reward regret and
fairness regret for our algorithms.


\begin{ack}
This work was supported in part by National Key R\&D
Program of China under Grant 2022YFB2902700, NSF China (Grant No. 
62202508, 62071501), and Shenzhen Science and Technology Program (Grant
20220817094427001, JCYJ20220818102011023, ZDSYS20210623091807023).
\end{ack}



\bibliography{mybibfile}
\onecolumn

\appendix

\section{Proofs of the Theorems}
\label{appendix:proofs-theorems}

\subsection{Proof of Theorem~\ref{the:Optimal-Fair-Policy}}
\label{Proof_of_DF-CUCB_RR}

\begin{proof}
	According to~\eqref{fairness-constraint}, the optimal fair policy $\bm{p}^*$ satisfies the following merit-based fairness constraints:
	\begin{equation}
		\frac{p_a^*}{f(\mu_a)}=\frac{p_{a^{\prime}}^*}{f(\mu_{a^{\prime}})}, \quad \forall a\neq a^{\prime}, a,a^{\prime}\in[K],
	\end{equation}
	which correspond to $K-1$ linearly independent equations of $\bm{p}^*$. 
	Moreover, because only $L$ arms can be selected at each round, there is an additional linear equation 
	$\sum_{a} p^*_a=L$ that is linearly independent of the other $K-1$ ones. Then we have $K$ linearly independent equations on $K$ unknowns in $\bm{p}^* =\left \lbrace p^*_1,p^*_2,...,p^*_K \right \rbrace$.
	Therefore, the optimal fair policy $\bm{p}^*$ is unique. By solving this
    system of linear equations, we have
 \begin{equation}
     p_a^*=\frac{Lf(\mu_a)}{\sum_{a^{\prime}=1}^K f(\mu_{a^{\prime}})}, \quad \forall a \in [K].
 \end{equation}
 This completes the proof of Theorem~\ref{the:Optimal-Fair-Policy}.
\end{proof}

\subsection{Proof of~Theorem \ref{the:Lower-bound-fr-without-Assumption}}
\label{Proof_of_Lower-bound-fr-without-Assumption} 
\begin{proof} 
We first prove that \emph{the lower bound on fairness regret is linear
without Assumption~\ref{Minimum-Merit}~\ref{Minimum-Merit:minmu}} by constructing
two CMAB instances and a $1$-Lipschitz merit function $f(\cdot)$. 
For any bandit algorithm, we show that the sum of the expected fairness regrets
of the two CMAB instances increases linearly in $T$ 
in the absence of feedback delays.
Consequently, we conclude that any bandit algorithm will incur a linear regret
in $T$ for at least one of the two CMAB instances under reward-dependent delays
or reward-independent delays.

The two instances can be defined as $x^1=(\nu^1_1, \nu^1_2, \nu^1_3)$ and $x^2=(\nu^2_1, \nu^2_2, \nu^2_3)$, where $\nu_a$ is the reward distribution of arm $a$. 
Each instance consists of three arms and the learner selects a subset $A_t$ of $L=2$ arms at each round $t$. 
We assume that the reward of each arm in the two instances follows a Bernoulli distribution.
The expected rewards of three arms in the first instance are $3\eta, 2\eta, 2\eta$, 
i.e., $\nu^1_1=\mathrm{Bernoulli}(3\eta), \nu^1_2=\mathrm{Bernoulli}(2\eta), \nu^1_3=\mathrm{Bernoulli}(2\eta)$, 
and the expected rewards of three arms in the second instance are $2\eta, 2\eta, 2\eta$, 
i.e., $\nu^2_1=\mathrm{Bernoulli}(2\eta), \nu^2_2=\mathrm{Bernoulli}(2\eta), \nu^2_3=\mathrm{Bernoulli}(2\eta)$, 
where $\eta \in (0,1/3]$ .
The merit function $f(\cdot)$ is defined as an identity function. i.e., $f(\mu)=\mu $, $\mu \in \left[0,1\right]$.
Therefore, the optimal fair policy for the first instance is $\bm{p}^{*,1}=\left\lbrace 6/7, 4/7, 4/7\right \rbrace$, 
and the optimal fair policy for the second instance is $\bm{p}^{*,2}=\left\lbrace2/3, 2/3, 2/3\right \rbrace$. 
For any bandit algorithm $\pi$, the learner selects the arms stochastically according to a selection policy $\bm{p}_t$ at each round $t$ 
based on the history $\mathcal{H}_t$, which consists of all the previous selection vectors, selected arm sets, and received feedback.
We have $a \sim \bm{p}_{t}$, $R_{t,a} \sim \nu_a$ for $a \in A_t$.
Then we can derive the lower bound of the expected fairness regret for the two instances as follows.
For the first instance $x^1$, we have
\begin{equation}
\begin{aligned}
\mathbb{E}\left[\frac{1}{T} \mathrm{FR}_T^1\right] & 
=\mathbb{E}\left[\frac{1}{T} \sum_{t=1}^T\left(\left|p_{t,1}-\frac{6}{7}\right|+\left|p_{t,2}-\frac{4}{7}\right|+\left|p_{t,3}-\frac{4}{7}\right|\right)\right] \\
& \geq \mathbb{E}\left[\left|\frac{1}{T} \sum_{t=1}^T p_{t,1}-\frac{6}{7}\right|+\left|\frac{1}{T} \sum_{t=1}^T p_{t,2}-\frac{4}{7}\right|+\left|\frac{1}{T} \sum_{t=1}^T p_{t,3}-\frac{4}{7}\right|\right]\\
&=2 \mathbb{E}\left[\left|\frac{1}{T} \sum_{t=1}^T p_{t,1}-\frac{6}{7}\right|\right].
\end{aligned}
\end{equation}
    Similarly, for the second instance $x^2$, we can derive the fairness regret lower bound as follows.
\begin{equation}
\mathbb{E}\left[\frac{1}{T} \mathrm{FR}_T^2\right] 
\geq 2 \mathbb{E}\left[\left|\frac{1}{T} \sum_{t=1}^T p_{t,1}-\frac{2}{3}\right|\right].
\end{equation}
When there is no feedback delay, we consider an arm selection trace during the
$T$ rounds as $h =(\bm{p}_{1},A_1,\bm{R}_{1},...,\bm{p}_{T},A_{T},\bm{R}_{T})$,
where $\bm{R}_{t}:=\left(R_{t,a}\right)_{a\in[K]}$.
Denote $\mathbb{H}^1, \mathbb{H}^2$ as the distributions of $h$ for first CMAB
instance $x^1, x^2$
using the algorithm $\pi$, respectively.
Then we have
\begin{equation}
\label{ineq:sum-two-MAB-instance-part-I}
\begin{aligned}
\mathbb{E}\left[\frac{1}{T}\mathrm{FR}_T^1 \right]+\mathbb{E}\left[\frac{1}{T} \mathrm{FR}_T^2 \right] & \geq \frac{4}{21} \mathbb{P}^1\left(\frac{1}{T} \sum_{t=1}^T p_{t,1}\leq \frac{16}{21}\right)+\frac{4}{21} \mathbb{P}^2\left(\frac{1}{T} \sum_{t=1}^T p_{t,1} > \frac{16}{21}\right) \\
& \stackrel{(a)}{\geq} \frac{2}{21} \exp \left(-\mathrm{KL}\left(\mathbb{H}^1, \mathbb{H}^2\right)\right), 
\end{aligned}
\end{equation}
where $(a)$ follows from the Bretagnolle-Huber inequality~\cite{bretagnolle1978estimation}.
We can derive the upper bound of the KL divergence between $\mathbb{H}^1$ and
$\mathbb{H}^2$, $\mathrm{KL}\left(\mathbb{H}^1, \mathbb{H}^2\right)$, as follows:
    
\begin{equation}
\label{ineq:upper-bound-KL-part-I}
\begin{aligned}
\operatorname{KL}\left(\mathbb{H}^1, \mathbb{H}^2\right) 
& =\mathbb{E}_{h \sim \mathbb{H}^1}\left[\log \frac{\mathbb{H}^1(h)}{\mathbb{H}^2(h)}\right]
\leq \mathbb{E}_{h \sim \mathbb{H}^1}\left[\sum_{t=1}^T\sum_{a\in A_t} \log \frac{\nu_{a}^1\left(R_{t,a}\right)}{\nu_{a}^2\left(R_{t,a}\right)}\right]
=\sum_{t=1}^T \mathbb{E}_{\bm{p}_t \sim \pi^1} \mathbb{E}_{a \sim \bm{p}_t} \left[\operatorname{KL}\left(\nu_{a}^1, \nu_{a}^2\right) \right]\\
& =\sum_{t=1}^T \mathbb{E}_{\bm{p}_t \sim \pi^1}\left[p_{t,1} \operatorname{KL}\left(\nu_1^1, \nu_1^2\right)\right]\\
& =\sum_{t=1}^T \mathbb{E}_{\bm{p}_t \sim \pi^1}\left[p_{t,1} \left(3\eta\log \frac{3}{2} + (1-3\eta) \log \frac{1-3\eta}{1-2\eta}\right)\right] \\
& \leq  3T\eta\log \frac{3}{2},
\end{aligned}
\end{equation}
where $\bm{p}_t \sim \pi^1$ means that $\bm{p}_t$ 
is sampled from the process of the algorithm $\pi^1$ applied to the first CMAB instance.

Combining~\eqref{ineq:upper-bound-KL-part-I} with~\eqref{ineq:sum-two-MAB-instance-part-I} 
and setting $\eta=1 / 3T$, we have
\begin{equation}
\begin{aligned}
\mathbb{E}\left[\frac{1}{T}\mathrm{FR}_T^1 \right]+\mathbb{E}\left[\frac{1}{T} \mathrm{FR}_T^2 \right]  
&\geq \frac{2}{21} \exp \left(-3T\eta\log \frac{3}{2} \right)
\geq 0.06,
\end{aligned}
\end{equation}
which implies that at least one of the two CMAB instances incurs linear expected
fairness regret without feedback delays.
Therefore, we infer that under reward-dependent delays or reward-independent
delays, at least one of the two CMAB instances incurs linear expected fairness
regret, given that feedback delays tend to increase the fairness regret.

Next, we prove that \emph{the lower bound on fairness regret is linear without
Assumption~\ref{Lipschitz-Continuity}} by constructing two CMAB instances and a
merit function $f(\cdot)$ where $\min_{\mu} f(\mu) = 1$.
For any bandit algorithm, we demonstrate that the expected fairness
regrets of the two CMAB instances grow linearly with respect to $T$ in the absence of feedback delays. Thus, any bandit algorithm
will result in linear regret in $T$ for at least one of the two CMAB instances
under either reward-dependent delays or reward-independent delays.

The two instances can be defined as $x^1=(\nu^1_1, \nu^1_2, \nu^1_3)$ and $x^2=(\nu^2_1, \nu^2_2, \nu^2_3)$, where $\nu_a$ is the reward distribution of arm $a$. 
Each instance consists of three arms and the learner selects a subset $A_t$ of $L=2$ arms at each round $t$. 
We assume that the reward of each arm in the two instances follows a Bernoulli distribution.
The expected rewards of three arms in the first instance are $2\eta, \eta, \eta$, 
i.e., $\nu^1_1=\mathrm{Bernoulli}(2\eta), \nu^1_2=\mathrm{Bernoulli}(\eta), \nu^1_3=\mathrm{Bernoulli}(\eta)$, 
and the expected rewards of three arms in the second instance are $\eta, \eta, \eta$, 
i.e., $\nu^2_1=\mathrm{Bernoulli}(\eta), \nu^2_2=\mathrm{Bernoulli}(\eta), \nu^2_3=\mathrm{Bernoulli}(\eta)$, 
where $\eta \in (0,1/2)$ .
We use the merit function $f(\cdot)$ with the form $f(\mu)=M \mu+1$
where $\mu \in [0,1] $ and $M>0$ is a positive constant to be defined later.
Therefore, the optimal fair policy for the first instance is 
$\bm{p}^{*,1}=\left\lbrace 
(4\eta M+2)/(4\eta M+3), (2\eta M+2)/(4\eta M+3),(2\eta M+2)/(4\eta M+3)
\right \rbrace$, 
and the optimal fair policy for the second instance is 
$\bm{p}^{*,2}=\left\lbrace 
2/3,2/3,2/3
\right \rbrace$. 
For any bandit algorithm $\pi$, the learner selects the arms with a probabilistic selection vector $\bm{p}_t$ at each round $t$ 
based on the observation and decision history $\mathcal{H}_t$. 
Then we have $a \sim \bm{p}_{t}$, $R_{t,a} \sim \nu_a$ for $a \in A_t$.

For any algorithm, we can lower bound the expected fairness regret for the two instances as follows.
For the first instance $x^1$, we have
\begin{equation}
\begin{aligned}
\mathbb{E}\left[\frac{1}{T} \mathrm{FR}_T^1\right] & 
=\mathbb{E}\left[\frac{1}{T} \sum_{t=1}^T\left(\left|p_{t,1}-\frac{4\eta M+2}{4\eta M+3}\right|+\left|p_{t,2}-\frac{2\eta M+2}{4\eta M+3}\right|+\left|p_{t,3}-\frac{2\eta M+2}{4\eta M+3}\right|\right)\right] \\
& \geq \mathbb{E}\left[\left|\frac{1}{T} \sum_{t=1}^T p_{t,1}-\frac{4\eta M+2}{4\eta M+3}\right|+\left|\frac{1}{T} \sum_{t=1}^T p_{t,2}-\frac{2\eta M+2}{4\eta M+3}\right|+\left|\frac{1}{T} \sum_{t=1}^T p_{t,3}-\frac{2\eta M+2}{4\eta M+3}\right|\right]\\
&\geq 2 \mathbb{E}\left[\left|\frac{1}{T} \sum_{t=1}^T p_{t,1}-\frac{4\eta M+2}{4\eta M+3}\right|\right].
\end{aligned}
\end{equation}
    Similarly, for the second instance $x^2$, we can derive the fairness regret lower bound as follows,
\begin{equation}
\mathbb{E}\left[\frac{1}{T} \mathrm{FR}_T^2\right] 
\geq 2 \mathbb{E}\left[\left|\frac{1}{T} \sum_{t=1}^T p_{t,1}-\frac{2}{3}\right|\right].
\end{equation}

When there is no feedback delay, we consider an arm selection trace during the $T$
rounds as $h =(\bm{p}_{1},A_1,\bm{R}_{1},...,\bm{p}_{T},A_{T},\bm{R}_{T})$.
Denote $\mathbb{H}^1$ as the distribution of $h$ when the algorithm $\pi$ is applied to the first MAB instance $x^1$, 
while $\mathbb{H}^2$ as the distribution of $h$ when the algorithm $\pi$ is applied to the second MAB instance $x^2$.
Then we have
\begin{equation}
\label{ineq:sum-two-MAB-instance}
\begin{aligned}
\mathbb{E}\left[\frac{1}{T}\mathrm{FR}_T^1 \right]+\mathbb{E}\left[\frac{1}{T} \mathrm{FR}_T^2 \right] 
& \geq \frac{4M\eta}{12M\eta+9} \mathbb{P}^1\left(\frac{1}{T} \sum_{t=1}^T p_{t,1}\leq\frac{10M\eta+6}{12M\eta+9}\right)
+\frac{4M\eta}{12M\eta+9} \mathbb{P}^2\left(\frac{1}{T} \sum_{t=1}^T p_{t,1} > \frac{10M\eta+6}{12M\eta+9}\right) \\
& \stackrel{(a)}{\geq} \frac{2M\eta}{12M\eta+9} \exp \left(-\mathrm{KL}\left(\mathbb{H}^1, \mathbb{H}^2\right)\right),
\end{aligned}
\end{equation}
where $(a)$ follows from the Bretagnolle-Huber inequality~\cite{bretagnolle1978estimation}.
Then we can upper bound the KL divergence $\mathrm{KL}\left(\mathbb{H}^1,
\mathbb{H}^2\right)$ between $\mathbb{H}^1$ and $\mathbb{H}^2$ as follows,
    
\begin{equation}
\label{ineq:upper-bound-KL}
\begin{aligned}
\operatorname{KL}\left(\mathbb{H}^1, \mathbb{H}^2\right) 
& =\mathbb{E}_{h \sim \mathbb{H}^1}\left[\log \frac{\mathbb{H}^1(h)}{\mathbb{H}^2(h)}\right]
\leq \mathbb{E}_{h \sim \mathbb{H}^1}\left[\sum_{t=1}^T\sum_{a\in A_t} \log \frac{\nu_{a}^1\left(R_{t,a}\right)}{\nu_{a}^2\left(R_{t,a}\right)}\right]
=\sum_{t=1}^T \mathbb{E}_{\bm{p}_t \sim \pi^1} \mathbb{E}_{a \sim \bm{p}_t} \left[\operatorname{KL}\left(\nu_{a}^1, \nu_{a}^2\right) \right]\\
& =\sum_{t=1}^T \mathbb{E}_{\bm{p}_t \sim \pi^1}\left[p_{t,1} \operatorname{KL}\left(\nu_1^1, \nu_1^2\right)\right]\\
& =\sum_{t=1}^T \mathbb{E}_{\bm{p}_t \sim \pi^1}\left[p_{t,1} \left(2\eta\log 2 + (1-2\eta) \log \frac{1-2\eta}{1-\eta}\right)\right] \\
& \leq 2T\eta\log 2,
\end{aligned}
\end{equation}
where $\bm{p}_t \sim \pi^1$ means that $\bm{p}_t$ 
is sampled from the process of the algorithm $\pi^1$ applied to the first MAB instance.

According to~\eqref{ineq:sum-two-MAB-instance} and~\eqref{ineq:upper-bound-KL}, 
set $\eta=\frac{1}{2T}, M=T$, we have
\begin{equation}
\begin{aligned}
\mathbb{E}\left[\frac{1}{T}\mathrm{FR}_T^1 \right]+\mathbb{E}\left[\frac{1}{T} \mathrm{FR}_T^2 \right]  
&\geq \frac{2M\eta}{12M\eta+9} \exp \left(-2T\eta\log 2\right)
\geq 0.03,
\end{aligned}
\end{equation}
which implies that at least one of the two CMAB instances incurs linear expected
fairness regret without feedback delays.
Therefore, we can conclude that, under reward-dependent delays or
reward-independent delays, at least one of the two CMAB instances experiences a
linear expected fairness regret, as feedback delays tend to amplify the fairness
regret.

This completes the proof of Theorem~\ref{the:Lower-bound-fr-without-Assumption}.
\end{proof}

\subsection{Proof of Theorem \ref{the:fairness-reward-regret-ucb-type}}

\begin{proof} We first prove the expected reward regret upper bound and then
prove the expected fairness regret upper bound of \cucbfd.

\proofpart{1}{Proof of the Expected Reward Regret Upper Bound of \cucbfd}
We first prove the following lemmas.
\begin{lemma}
\label{Lemma:hoeffding_ineq}
	For any $\delta\in (0,1)$,  with probability at least $1-\frac{\delta}{2}$, $\forall t>\lceil \frac{K}{L} \rceil$, $a\in[K]$, the expected reward vector $\bm{\mu} \in \mathcal{C}_t$.  
\end{lemma}
\begin{proof}
	According to Hoeffding's inequality, for $ t>\lceil \frac{K}{L} \rceil$, $a\in[K]$, with probability at least $1-\frac{\delta}{2KT}$,
	\begin{equation}
 \label{ineq:Hoeffding}
		\left| \hat{\mu}_{t,a}-\mu_a\right| \leq \sqrt{\frac{ \log (4KT/\delta)}{2(M_{t,a} \vee 1)}}.
	\end{equation}
Using union bound, by the result of~\eqref{ineq:Hoeffding} and $\mu_a \in \left[0,1\right]$, with probability at least $1-\frac{\delta}{2}$, $\forall t>\lceil \frac{K}{L}\rceil, a\in [K]$,  $\mu_a \in \left[ B_{t,a},  U_{t,a}\right]$ where $c_{t,a}=\sqrt{\frac{ \log (4KT/\delta)}{2(M_{t,a} \vee 1)}}$
.
This completes the proof of Lemma~\ref{Lemma:hoeffding_ineq}.
\end{proof}

Define an event as $\mathcal{F}=\left\lbrace \exists t\in[T], a\in[K]: N_{t,a} \geq \frac{24\log T}{q}, M_{t+d^{*}(q),a} < \frac{qN_{t,a}}{2}  \right\rbrace.$
\begin{lemma}
	\label{lower-bound-received-observation-event}
	When $T>K$, the probability of event $\mathcal{F}$ being true is $\mathbb{P}\left[\mathcal{F}\right]\leq 1/T $.
\end{lemma}
\begin{proof}
	We first recall the following lemma in~\cite{lancewicki2021stochastic}.
	\begin{lemma}[Lemma $2$ in \cite{lancewicki2021stochastic}]
	At round $t$, for any arm $a \in [K]$ and quantile $q \in (0,1]$, it holds that,
	\begin{equation}
		\mathbb{P}\left[ M_{t+d^{*}(q)} < \frac{qN_{t,a}}{2}    \right] \leq \exp(-\frac{qN_{t,a}}{8}).
	\end{equation}
    \label{lemma:lancewicki}
\end{lemma}
	By Lemma~\ref{lemma:lancewicki}, we have
	\begin{equation}
		\begin{aligned}
			 \mathbb{P}\left[\exists t, a\in[K]: N_{t,a} \geq \frac{24\log T}{q}, M_{t+d^{*}(q),a} < \frac{qN_{t,a}}{2}\right] 
			& \leq \sum_{a=1}^K \sum_{t: N_{t,a} \geq \frac{24\log T}{q}} \mathbb{P}\left[M_{t+d^{*}(q),a} < \frac{qN_{t,a}}{2}\right] \\
			& \leq \sum_{a=1}^K \sum_{t: N_{t,a} \geq \frac{24\log T}{q}} \exp \left(-\frac{qN_{t,a}}{8} \right) \\
			& \leq T K \exp \left(-\frac{q}{8} \cdot \frac{24 \log T}{q}\right) \leq \frac{1}{T}.
		\end{aligned}
\end{equation}
This completes the proof of Lemma~\ref{lower-bound-received-observation-event}.
\end{proof} Lemma~\ref{lower-bound-received-observation-event} implies that
there is a lower bound for the number of observed feedback for each arm $a$ when
$N_{t,a}$ is large enough.

\begin{lemma}
	\label{Lemma:martingale_seq}
	For any $\delta\in (0,1)$, with probability at least $1-\frac{\delta}{2}$,
    it holds that 
	\begin{equation}
	\left| \sum_{t=1}^{T}\mathbb{E}_{a\sim\bm{p}_t}\left[  \sqrt{\frac{1}{M_{t,a}\vee 1}} \right]  -\sum_{t=1}^{T}\sum_{a\in A_t} \sqrt{\frac{1}{M_{t,a}\vee 1} }\right| \leq L\sqrt{2T \log \frac{4}{\delta}}.
\end{equation}
\end{lemma}

\begin{proof}
We first construct a martingale difference sequence 
 \begin{equation}
	\sum_{a\in A_t}\sqrt{\frac{1}{M_{t,a}\vee 1} }-\mathbb{E}_{a\sim\bm{p}_t}\left[ \sqrt{\frac{1}{M_{t,a}\vee 1} }\right].
\end{equation}
Then we have
\begin{equation}
	\left| \sum_{a\in A_t}\sqrt{\frac{1}{M_{t,a}\vee 1} }-\mathbb{E}_{a\sim\bm{p}_t}\left[ \sqrt{\frac{1}{M_{t,a}\vee 1} }\right]\right| < L.
\end{equation}
By Azuma-Hoeffding's inequality, with probability at least
$1-\frac{\delta}{2}$, we have
\begin{equation}
	\left| \sum_{t=1}^{T}\mathbb{E}_{a\sim\bm{p}_t}\left[ \sqrt{\frac{1}{M_{t,a}\vee 1} }\right]  -\sum_{t=1}^{T}\sum_{a\in A_t}\sqrt{\frac{1}{M_{t,a}\vee 1} }\right| \leq L\sqrt{2T \log \frac{4}{\delta}}.
\end{equation} 
This completes the proof of Lemma \ref{Lemma:martingale_seq}.
\end{proof}

Based on the lemmas above,  the reward regret can be bounded as follows:
\begin{equation}
 \begin{aligned}
 	\label{eq:poof-reward_regret-ucb}
 	\mathrm{RR}_T&=\sum_{t=1}^{T}\max\left\lbrace \sum_{a=1}^Kp^{*}_a\mu_{a}-\sum_{a=1}^K p_{t,a}\mu_{a}, 0\right\rbrace 
 	\\&\stackrel{(a)}{\leq}  \left( \frac{K}{L} +1\right) L+\sum_{t=\lceil \frac{K}{L} \rceil+1}^{T}\sum_{a=1}^K (p_{t,a}\tilde{\mu}_{t,a}-p_{t,a}\mu_a) \\
 	&= K+L+\sum_{t=\lceil \frac{K}{L} \rceil+1}^{T}\sum_{a=1}^Kp_{t,a}(\tilde{\mu}_{t,a}-\hat{\mu}_{t,a}+\hat{\mu}_{t,a}-\mu_{a}) \\
 	&\stackrel{(b)}{\leq} K+L+\sum_{t=\lceil \frac{K}{L} \rceil+1}^{T}\sum_{a=1}^K p_{t,a}
 	2 \sqrt{\frac{ \log (4KT/\delta)}{2(M_{t,a} \vee 1)}}  \\
 	&= K+L
 	+\sqrt{2 \log \frac{4KT}{\delta}}\sum_{t=\lceil \frac{K}{L} \rceil+1}^{T} \mathbb{E}_{a\sim\bm{p}_t}\left[ \sqrt{\frac{1}{M_{t,a}\vee 1}} \right],
 \end{aligned}
 \end{equation}
where $(a)$ is from Line~\ref{alg:fairness-delay-ucb-type:mu_t} in Algorithm \ref{alg:fairness-delay-ucb-type},
and $(b)$ is from Lemma \ref{Lemma:hoeffding_ineq}.
For the term $\sum_{t=\lceil \frac{K}{L} \rceil+1}^{T}
\mathbb{E}_{a\sim\bm{p}_t}\left[ \sqrt{\frac{1}{M_{t,a}\vee 1}} \right]$ in \eqref{eq:poof-reward_regret-ucb}, we have
\begin{equation}
	\label{inq:expected-martingale_seq_received_observation}
	\begin{aligned}
		\sum_{t=\lceil \frac{K}{L} \rceil+1}^{T}\mathbb{E}_{a\sim\bm{p}_t}\left[ \sqrt{\frac{1}{M_{t,a}\vee 1}} \right]  &\leq 
		\sum_{t=1}^{T}\mathbb{P}\left[\mathcal{F}\right] \mathbb{E}_{a\sim\bm{p}_t}\left[ \sqrt{\frac{1}{M_{t,a}\vee 1}} \right] 		+\sum_{t=1}^{T}\mathbb{P}\left[\overline{\mathcal{F}}\right] \mathbb{E}_{a\sim\bm{p}_t}\left[ \sqrt{\frac{1}{M_{t,a}\vee 1}} \right]\\
		& \stackrel{(a)}{\leq} 
		\sum_{t=1}^{T}\frac{L}{T} 
		+\sum_{t=1}^{T}\mathbb{P}\left[\overline{\mathcal{F}}\right] \mathbb{E}_{a\sim\bm{p}_t}\left[ \sqrt{\frac{1}{M_{t,a}\vee1}} \right]\\
		& \leq L  +
		\mathbb{P}\left[\overline{\mathcal{F}}\right] \sum_{t=1}^{T} \mathbb{E}_{a\sim\bm{p}_t}\left[ \sqrt{\frac{1}{M_{t,a}\vee1}} \right]	\\
		& \stackrel{(b)}{\leq} L  + \mathbb{E} \left[  \Ind{\{\overline{\mathcal{F}} \}} \left( L\sqrt{2T \log \frac{4}{\delta}} +\sum_{t=1}^{T}\sum_{a\in A_t}\sqrt{ \frac{1}{M_{t,a}\vee1} }\right) \right]  \\
		& \leq L  +  L\sqrt{2T \log \frac{4}{\delta}} +\mathbb{E} \left[     \sum_{t=1}^{T}\sum_{a\in A_t}\Ind{\{\overline{\mathcal{F}} \}}\sqrt{ \frac{1}{M_{t,a}\vee1} } \right] ,
	\end{aligned}
\end{equation}
where $(a)$ is from Lemma~\ref{lower-bound-received-observation-event} and $(b)$
is from Lemma~\ref{Lemma:martingale_seq}. Then for the term $
\sum_{t=1}^{T}\sum_{a\in A_t}\Ind{\{\overline{\mathcal{F}} \}}\sqrt{
\frac{1}{M_{t,a}\vee1} }$, we have
\begin{equation}
	\label{inq:sub-term}
	\begin{aligned}
		&\sum_{t=1}^{T}\sum_{a\in A_t}\Ind{\{\overline{\mathcal{F}} \}}\sqrt{ \frac{1}{M_{t,a}\vee1} } = \sum_{t=1}^{T}\sum_{a\in A_t}\Ind{\{\overline{\mathcal{F}} \}}\left( \Ind{\left\lbrace N_{t,a} < \frac{24\log T}{q} \right\rbrace}  +  \Ind{\left\lbrace N_{t,a} \geq \frac{24\log T}{q} \right\rbrace} \right) \sqrt{ \frac{1}{M_{t,a}\vee1} }\\
		& \leq \sum_{t=1}^{T}\sum_{a\in A_t} \Ind{\left\lbrace N_{t,a} < \frac{24\log T}{q} \right\rbrace}  +   \sum_{t=1}^{T}\sum_{a\in A_t}\Ind{\{\overline{\mathcal{F}} \}}\Ind{\left\lbrace N_{t,a} \geq \frac{24\log T}{q} \right\rbrace}\sqrt{ \frac{1}{M_{t,a}\vee1} }\\
		& \leq \frac{24K\log T}{q} + \sum_{t=1}^{d^{*}(q)}\sum_{a\in A_t}\sqrt{\frac{1}{M_{t,a}\vee1}}+\sum_{t=1}^{T-d^{*}(q)}\sum_{a\in A_{t+d^{*}(q)}}\Ind{\{\overline{\mathcal{F}} \}}\Ind{\left\lbrace N_{t+d^{*}(q),a} \geq \frac{24\log T}{q} \right\rbrace}\sqrt{\frac{1}{M_{t+d^{*}(q),a}}}\\
		&  \leq \frac{24K\log T}{q} + Ld^{*}(q)+\sum_{t=1}^{T-d^{*}(q)}\sum_{a\in A_{t+d^{*}(q)}}\sqrt{\frac{2}{qN_{t,a}}}\\
		&  \leq \frac{24K\log T}{q} + Ld^{*}(q)+K\sqrt{\frac{2}{q}}\int_{0}^{\frac{LT}{K}}\sqrt{\frac{1}{x}}dx\\
		&  \leq \frac{24K\log T}{q} + Ld^{*}(q)+2\sqrt{\frac{2LKT}{q}}.\\
	\end{aligned}
\end{equation}
Plugging~\eqref{inq:sub-term}
into~\eqref{inq:expected-martingale_seq_received_observation}, we have
\begin{equation}
	\label{RR-second-term}
	\sum_{t=\lceil \frac{K}{L} \rceil+1}^{T}\mathbb{E}_{a\sim\bm{p}_t}\left[ \sqrt{\frac{1}{M_{t,a}\vee1}} \right]\leq \sum_{t=1}^{T}\mathbb{E}_{a\sim\bm{p}_t}\left[ \sqrt{\frac{1}{M_{t,a}\vee1}} \right] \leq L  + L\sqrt{2T \log \frac{4}{\delta}} +\frac{24K\log T}{q} + Ld^{*}(q)+2\sqrt{\frac{2LKT}{q}}.
\end{equation}
Based on the results of~\eqref{eq:poof-reward_regret-ucb}-\eqref{RR-second-term}, when $T > K$, with probability at least $1-\delta$, we have
 \begin{align}
	\label{eq:poof-reward_regret-ucb-2}
	\mathrm{RR}_T &\leq K+L
	+\sqrt{2 \log \frac{4KT}{\delta}}\sum_{t=\lceil \frac{K}{L} \rceil+1}^{T} \mathbb{E}_{a\sim\bm{p}_t}\left[ \sqrt{\frac{1}{M_{t,a}\vee 1}} \right]\nonumber \\
	&\leq K+L
	+\sqrt{2 \log \frac{4KT}{\delta}}\left( L  + L\sqrt{2T \log \frac{4}{\delta}} +\frac{24K\log T}{q} + Ld^{*}(q)+2\sqrt{\frac{2LKT}{q}}\right). 
\end{align}
Finally, setting $\delta=\frac{1}{LT}$ and $c_{t,a}=\sqrt{\frac{ \log (4LKT)}{M_{t,a} \vee 1}}$, the expected reward regret can be upper bounded as
\begin{equation}
	\label{eq:poof-reward_regret-ucb-3}
 \begin{aligned}
    \mathbb{E}\left[\mathrm{RR}_T\right]&\leq K+L
	+\sqrt{2 \log \frac{4KT}{\delta}}\left( L  + L\sqrt{2T \log \frac{4}{\delta}} +\frac{24K\log T}{q} + Ld^{*}(q)+2\sqrt{\frac{2LKT}{q}}\right)+LT\delta\\
 &\leq K+L
	+\sqrt{4 \log (4LKT) }\left( L  + L\sqrt{2T \log (4LT)} +\frac{24K\log T}{q} + Ld^{*}(q)+2\sqrt{\frac{2LKT}{q}}\right)+1
 \end{aligned}
\end{equation}

Note that \eqref{eq:poof-reward_regret-ucb-3} holds for any choice of quantile $q \in
(0,1]$.
Thus, we choose the optimal $q$ to obtain the minimum upper bound.
We have
\begin{equation}
 \mathbb{E}\left[\mathrm{RR}_T\right]=\widetilde{O}\left( \min_{q\in (0,1]}\left\lbrace \frac{K}{q}\sqrt{T}+L d^{*}(q)\right\rbrace \right).
\end{equation}

This completes the proof of the reward regret upper bound of \cucbfd.

\proofpart{2}{Proof of the Expected Fairness Regret Upper Bound of \cucbfd}

For any $\delta \in (0,1)$, $t>\lceil K/L \rceil$, with probability $1-\delta$,
\begin{equation}
\begin{aligned}
	\label{eq:instantaneous-fairness-reg}
	&\sum_{a=1}^K\left|p_{t,a}-p_a^*\right|
	=\sum_{a=1}^K \left| \frac{Lf(\tilde{\mu}_{t,a})}{\sum_{a^{\prime}=1}^K f(\tilde{\mu}_{t, a^{\prime}})}-
	\frac{Lf(\mu_{a})}{\sum_{a^{\prime}=1}^K f(\mu_{a^{\prime}})} \right|   \\
	&=\sum_{a=1}^K \frac{L\left|f(\tilde{\mu}_{t, a})\sum_{a^{\prime}=1}^K f(\mu_{ a^{\prime}})-
	f(\mu_{ a})\sum_{a^{\prime}=1}^K f(\tilde{\mu}_{t, a^{\prime}})  \right| }
	{\sum_{a^{\prime}=1}^Kf(\tilde{\mu}_{t, a^{\prime}})\sum_{a^{\prime}=1}^K f(\mu_{a^{\prime}})} \\
	&=\sum_{a=1}^K \frac{L\left|f(\tilde{\mu}_{t, a})\sum_{a^{\prime}=1}^K f(\mu_{ a^{\prime}})
	-f(\mu_{ a})\sum_{a^{\prime}=1}^K f(\mu_{ a^{\prime}})
	+f(\mu_{ a})\sum_{a^{\prime}=1}^K f(\mu_{ a^{\prime}})
	-f(\mu_{ a})\sum_{a^{\prime}=1}^K f(\tilde{\mu}_{t, a^{\prime}})  \right| }
	{\sum_{a^{\prime}=1}^K f(\tilde{\mu}_{t, a^{\prime}})\sum_{a^{\prime}=1}^K f(\mu_{a^{\prime}})} \\
	& \leq \frac{L\sum_{a=1}^K \left| f(\tilde{\mu}_{t, a})- f(\mu_{ a})\right|\sum_{a^{\prime}=1}^K f(\mu_{ a^{\prime}}) 
	+L\sum_{a=1}^K f(\mu_{ a})\sum_{a^{\prime}=1}^K\left| f(\mu_{ a^{\prime}})-f(\tilde{\mu}_{t, a}) \right| }
	{\sum_{a^{\prime}=1}^K f(\tilde{\mu}_{t, a^{\prime}})\sum_{a^{\prime}=1}^K f(\mu_{a^{\prime}})} \\
	&=\frac{2L\sum_{a=1}^K \frac{f(\tilde{\mu}_{t, a})}{f(\tilde{\mu}_{t, a})}\left| f(\tilde{\mu}_{t, a})- f(\mu_{ a})\right| }
	{\sum_{a^{\prime}=1}^K f(\tilde{\mu}_{t, a^{\prime}})} \\
	&\stackrel{(a)}{\leq} \sum_{a=1}^{K}\frac{2Mp_{t,a}L}{\lambda}\left| \tilde{\mu}_{t, a}-\mu_{ a}\right| \\
	&\stackrel{(b)}{\leq} \sum_{a=1}^{K}\frac{4Mp_{t,a}L}{\lambda}\sqrt{\frac{ \log (4KT/\delta)}{2(M_{t,a} \vee 1)}}  \\
	&=\frac{4ML}{\lambda}\sqrt{\frac{1}{2} \log \frac{4KT}{\delta}} \mathbb{E}_{a\sim\bm{p}_t}\left[ \sqrt{\frac{1}{M_{t,a} \vee 1}}\right] ,
\end{aligned}
\end{equation}
where $(a)$ follows from the
Assumption~\ref{Minimum-Merit}~\ref{Minimum-Merit:minmu} and Assumption~\ref{Lipschitz-Continuity}, $(b)$ follows from Lemma \ref{Lemma:hoeffding_ineq}. When $T > K$, with probability at least $1-\delta$, the fairness regret can be upper bounded as follows:
\begin{equation}
\begin{aligned}	\mathrm{FR}_T&=\sum_{t=1}^{T}\sum_{a=1}^K\left| p_a^*-p_{t,a}\right| 
\\&	\leq \left( \frac{K}{L} +1\right) L+\sum_{t=\lceil \frac{K}{L} \rceil+1}^{T}\sum_{a=1}^K\left| p_a^*-p_{t,a}\right|  \\
	&\leq K+L
	+\frac{4ML}{\lambda}\sqrt{\frac{1}{2} \log \frac{4KT}{\delta}} \sum_{t=\lceil \frac{K}{L} \rceil+1}^{T}\mathbb{E}_{a\sim\bm{p}_t}\left[ \sqrt{\frac{1}{M_{t,a} \vee 1}}\right] 
	  \\
	&\stackrel{(a)}{\leq} K+L
	+\frac{4ML}{\lambda}\sqrt{\frac{1}{2} \log \frac{4KT}{\delta}}\left( L \log T + L\sqrt{2T \log \frac{4}{\delta}} +\frac{24K\log T}{q} + Ld^{*}(q)+2\sqrt{\frac{2LKT}{q}}\right) ,
\end{aligned}
\end{equation}
where $(a)$ follows from \eqref{RR-second-term}. Furthermore, setting $\delta=\frac{1}{LT}$, the expected fairness regret can be upper bounded as 
\begin{equation}
\begin{aligned}
	\mathbb{E}\left[\mathrm{FR}_T\right] &\leq K+L
	+\frac{4ML}{\lambda}\sqrt{\frac{1}{2} \log \frac{4KT}{\delta}}\left( L \log T + L\sqrt{2T \log \frac{4}{\delta}} +\frac{24K\log T}{q} + Ld^{*}(q)+2\sqrt{\frac{2LKT}{q}}\right) +LT\delta
 \\&\leq K+L
	+\frac{4ML}{\lambda}\sqrt{ \log (4LKT)}\left( L \log T + L\sqrt{2T \log (4LT)} +\frac{24K\log T}{q} + Ld^{*}(q)+2\sqrt{\frac{2LKT}{q}}\right) +1.
 \end{aligned}
\end{equation}

Finally, we have
\begin{equation}
	\mathbb{E}\left[\mathrm{FR}_T\right] =\widetilde{O}\left( \min_{q\in (0,1]}\left\lbrace \frac{ML}{\lambda}\left( \frac{K}{q}\sqrt{T}+L d^{*}(q)\right) \right\rbrace \right).
\end{equation}

This completes the proof of fairness regret upper bound of \cucbfd. 

Combining Part 1 and Part 2 of the proof, we complete
 the proof of Theorem~\ref{the:fairness-reward-regret-ucb-type}.
\end{proof}

\subsection{Proof of Theorem~\ref{the:fairness-reward-regret-thompson-sam-type}}

\begin{proof}  
We first prove the expected fairness regret upper bound and then
prove the expected reward regret upper bound of \ctsfd.
To take the prior into account, the expectation is simultaneously taken over the draws of $\mu_a$, and the algorithm’s internal randomization over rewards and actions.

\proofpart{1}{Proof of Expected Fairness Regret Upper Bound of \ctsfd}

By assumption, for each arm $a$, the posterior distribution $\mathcal{Q}_{t,a}$ of $\tilde{\mu}_{t,a}$ is a Beta
distribution
$$\mathcal{P}_{t,a}( \cdot\mid \mathcal{H}_{t} ) :=  \mathrm{Beta} (u_{t,a}, v_{t,a}),$$
where $u_{t,a}=1+\sum_{s:s+D_{s,a} < t} \Ind{ \{a \in
A_{s}\}}\Ind{\{R_{t,a}=1\}}$,  $v_{t,a}=1+\sum_{s:s+D_{s,a} < t} \Ind{ \{a \in
A_{s}\}}\Ind{\{R_{t,a}=0\}}$, and
$\mathcal{H}_{t}  = \left(\bm{p}_1, A_1, \bm{Y}_1, ..., \bm{p}_{t-1}, A_{t-1},
\bm{Y}_{t-1} \right)$ denotes the observation and decision history up to round
$t-1$ that consists of all the
previous selection vectors, selected arm sets, and received feedback.
For each arm $a$, we notice that the posterior distribution of $\mu_a$ conditioned on $\mathcal{H}_{t}$ is $P_{t,a}( \cdot\mid \mathcal{H}_{t} )$. 
Therefore, $\bm{\mu}$ and $\tilde{\bm{\mu}}_t$ are identically distributed from $\mathcal{P}_{t,a}( \cdot\mid \mathcal{H}_{t} )$.

First, referring to~\eqref{eq:instantaneous-fairness-reg}, we can derive
the upper bound for the expected per-step fairness regret $fr_t$ as follows:
\begin{equation}
	\mathbb{E}[fr_t] 
	\leq \frac{2L}{\lambda} 
	\mathbb{E}_{\mathcal{H}_{t}}\left[\mathbb{E}_{\bm{\mu},\tilde{\bm{\mu}}_t}\left[\mathbb{E}_{a\sim\bm{p}_t}\left[\left| f(\tilde{\mu}_{t, a})- f(\mu_{ a})\right| \right]| \mathcal{H}_{t}\right]\right] 
	\leq \frac{2ML}{\lambda} 	\mathbb{E}_{\mathcal{H}_{t}}\left[\mathbb{E}_{\bm{\mu},\tilde{\bm{\mu}}_t}\left[\mathbb{E}_{a\sim\bm{p}_t}\left[\left| \tilde{\mu}_{t, a}- \mu_{a}\right| \right]| \mathcal{H}_{t}\right]\right].
\end{equation}

Taking expectation with respect to $\bm{\mu},\tilde{\bm{\mu}}_t$, we have
\begin{equation}
	\mathbb{E}_{\bm{\mu},\tilde{\bm{\mu}}_t}\left[ \mathbb{E}_{a\sim\bm{p}_t}\left[\left| \tilde{\mu}_{t, a}- \mu_{a}\right|\right] \Big| \mathcal{H}_{t}\right]  = \mathbb{E}_{\tilde{\bm{\mu}}_t}\left[ \sum_{a=1}^Kp_{t,a}
	\mathbb{E}_{\bm{\mu}}\left[ \left| \tilde{\mu}_{t, a}- \mu_{a}\right| | \mathcal{H}_{t}, \tilde{\bm{\mu}}_t \right] 
	\Big| \mathcal{H}_{t}\right].
\end{equation}

Conditioned on $\tilde{\bm{\mu}}_t$, $\forall a\in[K]$, we have
\begin{equation}
	\begin{aligned}
\mathbb{E}_{\bm{\mu}}\left[ \tilde{\mu}_{t, a}- \mu_{a} | \mathcal{H}_{t}, \tilde{\bm{\mu}}_t \right] &= \tilde{\mu}_{t, a}- \Ddot{\mu}_{t,a}\\
\mathrm{Var}_{\bm{\mu}}\left[ \tilde{\mu}_{t, a}- \mu_{a} | \mathcal{H}_{t}, \tilde{\bm{\mu}}_t \right]&=\frac{u_{t,a}v_{t,a}}{(u_{t,a}+v_{t,a}+1)(u_{t,a}+v_{t,a})^2}\leq \frac{1}{M_{t,a}+1}\\
 	\end{aligned}
\end{equation}
where $\Ddot{\mu}_{t,a}$ is the expectation of the posterior distribution $\mathcal{P}_{t,a}( \cdot\mid \mathcal{H}_{t} )$ at round $t$.
Therefore, we have
\begin{equation}
	\begin{aligned}
		\mathbb{E}_{\bm{\mu}}\left[ \left| \tilde{\mu}_{t, a}- \mu_{a}\right| | \mathcal{H}_{t}, \tilde{\bm{\mu}}_t \right]  & 
		\stackrel{(a)}{\leq} \sqrt{\mathbb{E}_{\bm{\mu}}\left[ \left| \tilde{\mu}_{t, a}- \mu_{a}\right|^2 | \mathcal{H}_{t}, \tilde{\bm{\mu}}_t \right] }
		=\sqrt{\frac{1}{M_{t,a}+1}+\left( \tilde{\mu}_{t, a}-\Ddot{\mu}_{t,a}        \right)^2} \\
		& \leq \sqrt{\frac{1}{M_{t,a}+1}}+\left|\tilde{\mu}_{t, a}-\Ddot{\mu}_{t,a}\right|,
	\end{aligned}
\end{equation}
where $(a)$ follows the fact that $\mathbb{E}\left[ mn \right] \leq  \sqrt{\mathbb{E}\left[ m^2 \right]} \sqrt{\mathbb{E}\left[ n^2 \right]} $. 

Then we derive the concentration inequality for the term $\left|\tilde{\mu}_{t, a}-\Ddot{\mu}_{t,a}\right|$,
where $\tilde{\mu}_{t, a}$ follows the posterior distribution $\mathcal{P}_{t,a}( \cdot\mid \mathcal{H}_{t} )$ and $\mathbb{E}\left[\tilde{\mu}_{t, a}\right]=\Ddot{\mu}_{t,a}$.
To bound this term,
we recall the theorem in~\cite{marchal2017sub}, 
\begin{lemma}[Theorem 1 in~\cite{marchal2017sub}]
\label{Theorem-1-in-marchal2017sub}
For any $u$, $v$, the Beta distribution $\mathrm{Beta}(u,v)$ is $\sigma^2(u,v)$-sub-Gaussian with optimal proxy variance given by $\sigma^2(u,v)=\frac{1}{4(u+v+1)}.$
\end{lemma} By Lemma~\ref{Theorem-1-in-marchal2017sub},
$\tilde{\mu}_{t, a}$ is $\frac{1}{M_{t,a}+1}$-sub-Gaussian since
$\frac{1}{4(u_{t,a}+v_{t,a}+1)}\leq\frac{1}{M_{t,a}+1}$. 
Then for any $t \in [T]$, $a\in[K]$, with probability at least $1-\frac{\delta^{\prime}}{KT}$, we have

\begin{equation}
	\left|\tilde{\mu}_{t, a}-\Ddot{\mu}_{t,a} \right|
	\leq \sqrt{2\log \frac{2KT}{\delta^{\prime}}}\sqrt{\frac{1}{M_{t,a}+1}}.
\end{equation}

By the union bound over all $a\in[K]$ and all $T$, with probability at least $1-\delta^{\prime}$, we have  
\begin{equation}
\label{ineq:beta-concentration-bound}
	\forall t , a\in [K], \  \left|\tilde{\mu}_{t, a}-\Ddot{\mu}_{t,a} \right|
	\leq \sqrt{2\log \frac{2KT}{\delta^{\prime}}}\sqrt{\frac{1}{M_{t,a}+1}}.
\end{equation}
Denote the event that~\eqref{ineq:beta-concentration-bound} holds at time $t$ as $\mathcal{E}_t$ , i.e., $\mathcal{E}_t=\lbrace\forall  a\in [K], \  \left|\tilde{\mu}_{t, a}-\Ddot{\mu}_{t,a} \right|
	\leq \sqrt{2\log \frac{2KT}{\delta^{\prime}}}\sqrt{\frac{1}{M_{t,a}+1}} \rbrace$,
which only depends on the $\tilde{\bm{\mu}}_t$. We have 
\begin{equation}
	\begin{aligned}
		& \mathbb{E}_{\tilde{\bm{\mu}}_t}\left[ \sum_{a=1}^Kp_{t,a}
		\mathbb{E}_{\bm{\mu}}\left[ \left| \tilde{\mu}_{t, a}- \mu_{a}\right| | \mathcal{H}_{t}, \tilde{\bm{\mu}}_t \right] 
		\Big| \mathcal{H}_{t}\right] \\
		&= \mathbb{E}_{\tilde{\bm{\mu}}_t}\left[\Ind{\left\{\mathcal{E}_t\right\}} \sum_{a=1}^K p_{t,a} \mathbb{E}_{\bm{\mu}}\left[ \left| \tilde{\mu}_{t, a}- \mu_{a}\right| | \mathcal{H}_{t}, \tilde{\bm{\mu}}_t \right] \Big| \mathcal{H}_{t}\right]
		+\mathbb{E}_{\tilde{\bm{\mu}}_t}\left[\Ind{\left\{\overline{\mathcal{E}}_t\right\}} \sum_{a=1}^K p_{t,a} \mathbb{E}_{\bm{\mu}}\left[ \left| \tilde{\mu}_{t, a}- \mu_{a}\right| | \mathcal{H}_{t}, \tilde{\bm{\mu}}_t \right] \Big| \mathcal{H}_{t}\right] \\
		&\leq \mathbb{E}_{\tilde{\bm{\mu}}_t}\left[\mathbb{E}_{a\sim\bm{p}_t}\left[\left(1+\sqrt{2 \log \left(2K T / \delta^{\prime}\right)}\right) \sqrt{\frac{1}{M_{t,a}+1}}\right] \Big|\mathcal{H}_t\right]
		+\underbrace{\mathbb{E}_{\tilde{\bm{\mu}}_t}\left[\Ind{\left\{\overline{\mathcal{E}}_t\right\}} \sum_{a=1}^K p_{t,a} \mathbb{E}_{\bm{\mu}}\left[ \left| \tilde{\mu}_{t, a}- \mu_{a}\right| | \mathcal{H}_{t}, \tilde{\bm{\mu}}_t \right] \Big| \mathcal{H}_{t}\right]}_{\text {term a }}.
	\end{aligned}
    \label{eq:ex-tilde_mu_t}
\end{equation}
For the term a in~\eqref{eq:ex-tilde_mu_t}, note that
\begin{equation}
	\label{eq:instanstaneous-fairness-regret-sub-term}
	\begin{aligned}
		\mathbb{E}_{\bm{\mu}}\left[ \left| \tilde{\mu}_{t, a}- \mu_{a}\right|^2 | \mathcal{H}_{t}, \tilde{\bm{\mu}}_t \right] 
		& \stackrel{(a)}{\leq} 2\left(\tilde{\mu}_{t, a}- \Ddot{\mu}_{t,a}\right)^2+2 \mathbb{E}_{\bm{\mu}}\left[\left( \mu_{a}- \Ddot{\mu}_{t,a}\right)^2 \mid \mathcal{H}_t, \tilde{\bm{\mu}}_t\right] \\
		& \stackrel{(b)}{=} 2\left(\tilde{\mu}_{t, a}- \Ddot{\mu}_{t,a}\right)^2
		+\frac{2}{M_{t,a}+1}
		\leq 2\left(\tilde{\mu}_{t, a}- \Ddot{\mu}_{t,a}\right)^2
		+2,
	\end{aligned}
\end{equation}
where $(a)$ is from the inequality $(m+n)^2 \leq 2m^2 + 2n^2$ and $(b)$ is because 
$\mathbb{E}_{\bm{\mu}}\left[\mu_{a} - \Ddot{\mu}_{t,a}\mid \mathcal{H}_t,
\tilde{\bm{\mu}}_t\right] = 0$ and  $\mathrm{Var}_{\bm{\mu}}\left[\mu_{a} -
\Ddot{\mu}_{t,a}\mid \mathcal{H}_t, \tilde{\bm{\mu}}_t\right] = \frac{u_{t,a}v_{t,a}}{(u_{t,a}+v_{t,a}+1)(u_{t,a}+v_{t,a})^2} \leq
\frac{1}{M_{t,a}+1}$. Next, we can upper bound term a in~\eqref{eq:ex-tilde_mu_t} as follows,
\begin{equation}
	\begin{aligned}
		\text { term } \mathrm{a} & \leq \mathbb{E}_{\tilde{\bm{\mu}}_t}\left[\Ind{\left\{\overline{\mathcal{E}}_t\right\}} \sqrt{\sum_{a=1}^K p_{t,a}^2} \sqrt{\sum_{a=1}^K  \mathbb{E}_{\bm{\mu}}\left[ \left( \tilde{\mu}_{t, a}- \mu_{a}\right)^2 | \mathcal{H}_{t}, \tilde{\bm{\mu}}_t \right]  } \;\Big| \mathcal{H}_{t}\right] \\
		& \leq \mathbb{E}_{\tilde{\bm{\mu}}_t}\left[\Ind{\left\{\overline{\mathcal{E}}_t\right\}} \sqrt{L}\sqrt{\sum_{a=1}^K  \mathbb{E}_{\bm{\mu}}\left[ \left( \tilde{\mu}_{t, a}- \mu_{a}\right)^2 | \mathcal{H}_{t}, \tilde{\bm{\mu}}_t \right]  } \;\Big| \mathcal{H}_t\right] \\
		& \stackrel{(a)}{\leq} 2 \mathbb{E}_{\tilde{\bm{\mu}}_t}\left[\Ind{\left\{\overline{\mathcal{E}}_t\right\}} \sqrt{L}\sqrt{\sum_{a=1}^K\left(\tilde{\mu}_{t, a}- \Ddot{\mu}_{t,a}\right)^2+K} \;\Big|  \mathcal{H}_t\right] \\
		& \leq 2 \mathbb{E}_{\tilde{\bm{\mu}}_t}\left[\Ind{\left\{\overline{\mathcal{E}}_t\right\}} \sqrt{L}\sqrt{\sum_{a=1}^K\left(\tilde{\mu}_{t, a}- \Ddot{\mu}_{t,a}\right)^2} \;\Big|  \mathcal{H}_t\right]
		+ 2 \mathbb{E}_{\tilde{\bm{\mu}}_t}\left[\Ind{\left\{\overline{\mathcal{E}}_t\right\}} \sqrt{LK}\;\Big|  \mathcal{H}_t\right] ,
	\end{aligned}
\end{equation}
where $(a)$ is from~\eqref{eq:instanstaneous-fairness-regret-sub-term}. Then we
further upper bound the term
$\mathbb{E}_{\tilde{\bm{\mu}}_t}\left[\Ind{\left\{\overline{\mathcal{E}}_t\right\}}
\sqrt{\sum_{a=1}^K\left(\tilde{\mu}_{t, a}- \Ddot{\mu}_{t,a}\right)^2} \;\Big|
\mathcal{H}_t\right]$ as follows,
\begin{equation}
	\begin{aligned}
	\mathbb{E}_{\tilde{\bm{\mu}}_t}\left[\Ind{\left\{\overline{\mathcal{E}}_t\right\}} \sqrt{\sum_a\left(\tilde{\mu}_{t, a}- \Ddot{\mu}_{t,a}\right)^2} \;\Big|  \mathcal{H}_t\right] &\stackrel{(a)}{\leq} \sqrt{\mathbb{E}_{\tilde{\bm{\mu}}_t}\left[\Ind{\left\{ \overline{\mathcal{E}}_t\right\}} \;| \mathcal{H}_t\right]} 
	\sqrt{\mathbb{E}_{\tilde{\bm{\mu}}_t}\left[\sum_a\left(\tilde{\mu}_{t, a}- \Ddot{\mu}_{t,a}\right)^2 \;\Big| \mathcal{H}_t\right]}\\
	&\stackrel{(b)}{\leq} \sqrt{\mathbb{E}_{\tilde{\bm{\mu}}_t}\left[\Ind{\left\{ \overline{\mathcal{E}}_t\right\} } \;| \mathcal{H}_t\right]} 
	\sqrt{\sum_a \frac{1}{M_{t,a}+1}}\\
	&\leq \sqrt{\mathbb{E}_{\tilde{\bm{\mu}}_t}\left[\Ind{\left\{ \overline{\mathcal{E}}_t\right\}} \;| \mathcal{H}_t\right]} 
	\sqrt{K},\\
	\end{aligned}
\end{equation}
where $(a)$ is from the fact that $\mathbb{E}\left[ mn \right] \leq  \sqrt{\mathbb{E}\left[ m^2 \right]} \sqrt{\mathbb{E}\left[ n^2 \right]} $ and $(b)$ is 
is because 
$\mathbb{E}_{\bm{\mu}}\left[\tilde{\mu}_{t,a} - \Ddot{\mu}_{t,a}\mid \mathcal{H}_t \right] = 0$ and   $\mathrm{Var}_{\bm{\mu}}\left[\tilde{\mu}_{t,a} - \Ddot{\mu}_{t,a}\mid \mathcal{H}_t \right] = \frac{u_{t,a}v_{t,a}}{(u_{t,a}+v_{t,a}+1)(u_{t,a}+v_{t,a})^2} \leq \frac{1}{M_{t,a}+1}$.
Hence, we have

\begin{equation}
		\text { term } \mathrm{a}  \leq 2\sqrt{\mathbb{E}_{\tilde{\bm{\mu}}_t}\left[\Ind{\left\{ \overline{\mathcal{E}}_t\right\} }\;\Big| \mathcal{H}_t\right]} \sqrt{LK}
		+2 \mathbb{E}_{\tilde{\bm{\mu}}_t}\left[\Ind{\left\{\overline{\mathcal{E}}_t\right\}} \sqrt{LK}\;\Big|  \mathcal{H}_t\right]
		= 2\left( \sqrt{\mathbb{P}(\overline{\mathcal{E}}_t| \mathcal{H}_t)} + \mathbb{P}(\overline{\mathcal{E}}_t| \mathcal{H}_t) \right) \sqrt{LK}.
\end{equation}
   Summing over all time slots, we have 
\begin{equation}
	\begin{aligned}
		&\mathbb{E} \left[\mathrm{FR}_T\right]=\sum_{t=1}^{T} \mathbb{E}[fr_t]\\
		&\leq \frac{2ML}{\lambda} 
		\mathbb{E}_{\mathcal{H}_{t}}\left[ \sum_{t=1}^{T}\mathbb{E}_{\tilde{\bm{\mu}}_t}\left[ \sum_{a=1}^Kp_{t,a}
		\mathbb{E}_{\bm{\mu}}\left[ \left| \tilde{\mu}_{t, a}- \mu_{a}\right| | \mathcal{H}_{t}, \tilde{\bm{\mu}}_t \right] 
		\Big| \mathcal{H}_{t}\right]\right]  \\
		&\leq  \frac{4M L}{\lambda} \sum_{t=1}^T \mathbb{E}_{\mathcal{H}_t}\left[\mathbb{E}_{\tilde{\bm{\mu}}_t}\left[\mathbb{E}_{a \sim \bm{p}_t}\left[\left(1+\sqrt{2 \log \left(2K T / \delta^{\prime}\right)}\right) \sqrt{\frac{1}{M_{t,a}+1}}\right] \Big| \mathcal{H}_t\right]
		+\left( \sqrt{\mathbb{P}(\overline{\mathcal{E}}_t| \mathcal{H}_t)} + \mathbb{P}(\overline{\mathcal{E}}_t| \mathcal{H}_t) \right) \sqrt{LK}\right] \\
		&=  \frac{4 ML}{\lambda}\left(1+\sqrt{2 \log \left(2K T / \delta^{\prime}\right)}\right) \mathbb{E}_{\tilde{\bm{\mu}}_t}\left[\sum_{t=1}^T \mathbb{E}_{a \sim \bm{p}_t} \left[\sqrt{\frac{1}{M_{t,a}+1}}\right]\right]
		\\&+\frac{4 ML}{\lambda} \sqrt{LK}\left(\sum_{t=1}^T \mathbb{E}_{\mathcal{H}_t}\left[\sqrt{\mathbb{P}\left(\overline{\mathcal{E}}_t \mid \mathcal{H}_t\right)}+\mathbb{P}\left(\overline{\mathcal{E}}_t \mid \mathcal{H}_t\right)\right]\right) .
	\end{aligned}
\end{equation}
For the term $\sum_{t=1}^T
\mathbb{E}_{\mathcal{H}_t}\left[\sqrt{\mathbb{P}\left(\overline{\mathcal{E}}_t
\mid \mathcal{H}_t\right)}+\mathbb{P}\left(\overline{\mathcal{E}}_t \mid
\mathcal{H}_t\right)\right]$ in above, we have
\begin{equation}
	\begin{aligned}
		\sum_{t=1}^T \mathbb{E}_{\mathcal{H}_t}\left[\sqrt{\mathbb{P}\left(\overline{\mathcal{E}}_t \mid \mathcal{H}_t\right)}+\mathbb{P}\left(\overline{\mathcal{E}}_t \mid \mathcal{H}_t\right)\right]
		&=\sum_{t=1}^T \mathbb{E}_{\mathcal{H}_t} \left[ \sqrt{\mathbb{P}\left(\overline{\mathcal{E}}_t \mid \mathcal{H}_t\right)}\right] +\sum_{t=1}^T \mathbb{P}\left(\overline{\mathcal{E}}_t\mid\mathcal{H}_t \right) 
		\\&\leq \sum_{t=1}^T \mathbb{E}_{\mathcal{H}_t} \left[ \sqrt{\mathbb{P}\left(\overline{\mathcal{E}}_t \mid \mathcal{H}_t\right)}\right] +\delta^{\prime} \\
		& \leq \sqrt{T} \sqrt{\sum_{t=1}^T \mathbb{E}_{\mathcal{H}_t}  \left[\mathbb{P}\left(\overline{\mathcal{E}}_t \mid \mathcal{H}_t\right)\right]}+\delta^{\prime} \leq \sqrt{T \delta^{\prime}}+\delta^{\prime} . 	
	\end{aligned}
\end{equation}
Therefore, the expected fairness regret of \ctsfd can be bounded as
\begin{equation}
	\label{BayesFR-all-terms}
	\begin{aligned}
		&\mathbb{E}\left[\mathrm{FR}_T\right]=\sum_{t=1}^{T} \mathbb{E}[fr_t]\\
		&\leq \frac{4 ML}{\lambda}\left(1+\sqrt{2 \log \left(2K T / \delta^{\prime}\right)}\right) \mathbb{E}_{\tilde{\bm{\mu}}_t}\left[\sum_{t=1}^T \mathbb{E}_{a \sim \bm{p}_t}\left[\sqrt{\frac{1}{M_{t,a}+1}}\right]\right]
		+\frac{4 ML}{\lambda} \sqrt{LK}\left(\sqrt{T \delta^{\prime}}+\delta^{\prime}\right) . 
	\end{aligned}
\end{equation}

Based on Lemma~\ref{Lemma:martingale_seq}, for any $\delta \in (0,1)$, we have
\begin{equation}
\begin{aligned}
	\mathbb{E}_{\tilde{\bm{\mu}}_t}\left[ \sum_{t=1}^{T}\mathbb{E}_{a\sim\bm{p}_t}\left[ \sqrt{\frac{1}{M_{t,a}+1}} \right] \right] \leq \delta LT + L\sqrt{2T \log \frac{4}{\delta}} +\sum_{t=1}^{T}\sum_{a\in A_t}\sqrt{ \frac{1}{M_{t,a}+1} }.
\end{aligned}
\end{equation}
Then we have
\begin{equation}
	\label{expected-martingale_seq_received_observation}
	\begin{aligned}
		\mathbb{E}_{\tilde{\bm{\mu}}_t}\left[ \sum_{t=1}^{T}\mathbb{E}_{a\sim\bm{p}_t}\left[ \sqrt{\frac{1}{M_{t,a}+1}} \right] \right] &= 
		\mathbb{E}_{\tilde{\bm{\mu}}_t}\left[ \sum_{t=1}^{T}\mathbb{P}\left[\mathcal{F}\right] \mathbb{E}_{a\sim\bm{p}_t}\left[ \sqrt{\frac{1}{M_{t,a}+1}} \right] 
		+\sum_{t=1}^{T}\mathbb{P}\left[\overline{\mathcal{F}}\right] \mathbb{E}_{a\sim\bm{p}_t}\left[ \sqrt{\frac{1}{M_{t,a}+1}} \right]\right]\\
		& \stackrel{(a)}{\leq} 
		\mathbb{E}_{\tilde{\bm{\mu}}_t}\left[\sum_{t=1}^{T}\frac{L}{T}  
		+\sum_{t=1}^{T}\mathbb{P}\left[\overline{\mathcal{F}}\right] \mathbb{E}_{a\sim\bm{p}_t}\left[ \sqrt{\frac{1}{M_{t,a}+1}} \right]\right]\\
		& \leq L  +
		\mathbb{P}\left[\overline{\mathcal{F}}\right] \mathbb{E}_{\tilde{\bm{\mu}}_t}\left[\sum_{t=1}^{T} \mathbb{E}_{a\sim\bm{p}_t}\left[ \sqrt{\frac{1}{M_{t,a}+1}} \right]\right]	\\
		& \stackrel{(b)}{\leq} L  + \mathbb{E}_{\tilde{\bm{\mu}}_t} \left[  \Ind{\{\overline{\mathcal{F}} \}} \left( \delta T + L\sqrt{2T \log \frac{4}{\delta}} +\sum_{t=1}^{T}\sum_{a\in A_t}\sqrt{ \frac{1}{M_{t,a}+1} }\right) \right]  \\
		& \leq L  + \delta LT + L\sqrt{2T \log \frac{4}{\delta}} +\mathbb{E}_{\tilde{\bm{\mu}}_t} \left[     \sum_{t=1}^{T}\sum_{a\in A_t}\Ind{\{\overline{\mathcal{F}} \}}\sqrt{ \frac{1}{M_{t,a}+1} } \right] ,
	\end{aligned}
\end{equation}
where $(a)$ is from Lemma~\ref{lower-bound-received-observation-event} and $(b)$
is from Lemma~\ref{Lemma:martingale_seq}. Then for the term
$  \sum_{t=1}^{T}\sum_{a\in A_t}\Ind{\{\overline{\mathcal{F}} \}}\sqrt{
\frac{1}{M_{t,a}+1} }$ in above,
\begin{equation}
	\begin{aligned}
		\sum_{t=1}^{T}\sum_{a\in A_t}\Ind{\{\overline{\mathcal{F}} \}}\sqrt{ \frac{1}{M_{t,a}+1} } &\leq \sum_{t=1}^{T}\sum_{a\in A_t}\Ind{\{\overline{\mathcal{F}} \}}\left( \Ind{\left\lbrace N_{t,a} < \frac{24\log T}{q} \right\rbrace}  +  \Ind{\left\lbrace N_{t,a} \geq \frac{24\log T}{q} \right\rbrace} \right) \sqrt{ \frac{1}{M_{t,a}+1} }\\
		& \leq \sum_{t=1}^{T}\sum_{a\in A_t} \Ind{\left\lbrace N_{t,a} < \frac{24\log T}{q} \right\rbrace}  +   \sum_{t=1}^{T}\sum_{a\in A_t}\Ind{\{\overline{\mathcal{F}} \}}\Ind{\left\lbrace N_{t,a} \geq \frac{24\log T}{q} \right\rbrace}\sqrt{ \frac{1}{M_{t,a}+1} }\\
		& \leq \frac{24K\log T}{q} + \sum_{t=1}^{d^{*}(q)}\sum_{a\in A_t}\sqrt{\frac{1}{M_{t,a}+1}}+\sum_{t=1}^{T}\sum_{a\in A_t}\sqrt{\frac{1}{M_{t+d^{*}(q),a}}}\\
		&  \leq \frac{24K\log T}{q} + Ld^{*}(q)+\sum_{t=1}^{T}\sum_{a\in A_t}\sqrt{\frac{2}{qN_{t,a}}}\\
		&  \leq \frac{24K\log T}{q} + Ld^{*}(q)+K\sqrt{\frac{2}{q}}\int_{0}^{\frac{LT}{K}}\sqrt{\frac{1}{x}}dx\\
		&  \leq \frac{24K\log T}{q} + Ld^{*}(q)+2\sqrt{\frac{2LKT}{q}}.\\
	\end{aligned}
\end{equation}
Let $\delta = \frac{1}{LT}$. We have
\begin{equation}
	\label{BayesFR-first-term}
	\begin{aligned}
		\mathbb{E}_{\tilde{\bm{\mu}}_t}\left[ \sum_{t=1}^{T}\mathbb{E}_{a\sim\bm{p}_t}\left[ \sqrt{\frac{1}{M_{t,a}+1}} \right] \right] &\leq L + \delta LT + L\sqrt{2T \log \frac{4}{\delta}} +\mathbb{E} \left[    \sum_{t=1}^{T}\sum_{a\in A_t}\Ind{\{\overline{\mathcal{F}} \}}\sqrt{ \frac{1}{M_{t,a}+1} } \right]\\ 
		&\leq L+ \frac{24K\log T}{q}  + L\sqrt{2T \log (4LT)}+2\sqrt{\frac{2LKT}{q}} + L d^{*}(q)+ 1 .
	\end{aligned}
\end{equation}
Combining~\eqref{BayesFR-all-terms} and~\eqref{BayesFR-first-term} and setting  $\delta^{\prime}=1/T$,  the expected fairness regret can be upper bounded as follows,
\begin{equation}
	\begin{aligned}
		&\mathbb{E}\left[\mathrm{FR}_T\right]=\sum_{t=1}^{T} \mathbb{E}[fr_t]\\
		&\leq \min_{q\in (0,1]} \left\lbrace \frac{4 ML}{\lambda}\left(1+\sqrt{4 \log \left(2K T \right)}\right) \left(L+ \frac{24K\log T}{q}  + L\sqrt{2T \log (4LT)}+2\sqrt{\frac{2LKT}{q}} + L d^{*}(q)+ 1 \right) \right.\\ 
		& \left.+\frac{4 ML}{\lambda} \sqrt{LK} \left(1+\frac{1}{T}\right) \right\rbrace 
  \\&=\widetilde{O}\left( \min_{q\in (0,1]}\left\lbrace \frac{ML}{\lambda}\left( \frac{K}{q}\sqrt{T}+Ld^{*}(q)\right) \right\rbrace \right).
	\end{aligned}
\end{equation}

This completes the proof of the expected fairness regret upper bound of \ctsfd.

\proofpart{2}{Proof of the Expected Reward Regret Upper Bound of \ctsfd}

We first prove the following lemma which is adapted from Proposition $1$ in~\cite{russo2014learning}.
\begin{lemma}
\label{lemma:BRT-decomposed}
The expected reward regret can be upper bounded as 
	\begin{equation}
		\mathbb{E} \left[\mathrm{RR}_T\right]\leq\sum_{t=1}^T \mathbb{E}\left[\max\left\lbrace\mathbb{E}_{a \sim \bm{p}_t}\left[U_{t, a}-\mu_a\right],0\right\rbrace\right]+\sum_{t=1}^T \mathbb{E}\left[\max\left\lbrace\mathbb{E}_{a \sim \bm{p}^*}\left[\mu_a-U_{t, a}\right],0\right\rbrace\right].
	\end{equation}
\end{lemma}

\begin{proof}
As mentioned earlier, $\bm{\mu}$ and $\tilde{\bm{\mu}}$ are identically distributed conditioned on $\mathcal{H}_t$ at each round $t$. 
In addition, the optimal fairness policy $\bm{p}^*$ and the policy $\bm{p}_t$
depend on $\bm{\mu}$ and $\tilde{\bm{\mu}}$, respectively. 
We conclude that $\bm{p}^*$ and $\bm{p}_t$ are also identically distributed conditioned on $\mathcal{H}_t$. Moreover, $U_{t, a}$ is fully determined by the history $\mathcal{H}_t$. Hence, $\mathbb{E} \left[  \mathbb{E}_{a \sim \bm{p}^*} \left[ U_{t, a} \mid \mathcal{H}_t \right] \right]  =\mathbb{E} \left[  \mathbb{E}_{a \sim \bm{p}_t} \left[ U_{t, a} \mid \mathcal{H}_t \right] \right]$. Then we have
\begin{equation}
	\begin{aligned}
		\mathbb{E}\left[\max\left\lbrace\mathbb{E}_{a \sim \bm{p}^*}\left[\mu_a\right]-\mathbb{E}_{a \sim \bm{p}_t}\left[\mu_a\right], 0\right\rbrace\right] 
		&=  \mathbb{E}_{\mathcal{H}_t}\left[\mathbb{E}\left[\max\left\lbrace\mathbb{E}_{a \sim \bm{p}^*}\left[\mu_a\right]-\mathbb{E}_{a \sim \bm{p}_t}\left[\mu_a\right], 0\right\rbrace \mid \mathcal{H}_t\right]\right] \\
  &=  \mathbb{E}_{\mathcal{H}_t}\left[\mathbb{E}\left[\max\left\lbrace\mathbb{E}_{a \sim \bm{p}^*}\left[\mu_a\right]-\mathbb{E}_{a \sim \bm{p}^*} \left[ U_{t, a}\right] +\mathbb{E}_{a \sim \bm{p}_t} \left[ U_{t, a}\right]-\mathbb{E}_{a \sim \bm{p}_t}\left[\mu_a\right], 0\right\rbrace \mid \mathcal{H}_t\right]\right] \\
		&= \mathbb{E}\left[\max\left\lbrace\mathbb{E}_{a \sim \bm{p}_t}\left[U_{t, a}-\mu_a\right]+\mathbb{E}_{a \sim \bm{p}^*}\left[\mu_a-U_{t, a}\right],0\right\rbrace\right]\\&
  \leq\mathbb{E}\left[\max\left\lbrace\mathbb{E}_{a \sim \bm{p}_t}\left[U_{t, a}-\mu_a\right],0\right\rbrace\right]+\mathbb{E}\left[\max\left\lbrace\mathbb{E}_{a \sim \bm{p}^*}\left[\mu_a-U_{t, a}\right],0\right\rbrace\right].
	\end{aligned}
\end{equation}
Therefore, we have
	\begin{equation}
		\mathbb{E} \left[\mathrm{RR}_T\right]\leq\sum_{t=1}^T \mathbb{E}\left[\max\left\lbrace\mathbb{E}_{a \sim \bm{p}_t}\left[U_{t, a}-\mu_a\right],0\right\rbrace\right]+\sum_{t=1}^T \mathbb{E}\left[\max\left\lbrace\mathbb{E}_{a \sim \bm{p}^*}\left[\mu_a-U_{t, a}\right],0\right\rbrace\right].
	\end{equation}
This completes the proof of Lemma~\ref{lemma:BRT-decomposed}.
\end{proof}
We establish the upper bound of the expected reward regret by introducing the following events.

Denote the first event as $ \mathcal{G}_1=\left \lbrace \forall t \in [T], a \in [K], \mu_a \in [B_{t,a},
U_{t,a}]\right \rbrace$, where the UCB estimate is defined as
$U_{t,a}:=\hat{\mu}_{t,a}+\sqrt{\frac{\log (4KT/\delta)}{2(M_{t,a}\vee 1)}}$ and
the LCB estimate is defined as $B_{t,a}:=\hat{\mu}_{t,a}-\sqrt{\frac{\log
(4KT/\delta)}{2(M_{t,a}\vee 1)}}$.
By Hoeffding's inequality, we have 
$\mathbb{P}\left[|\hat{\mu}_{t,a}-\mu_a|>\sqrt{\frac{\log (4KT/\delta)}{2(M_{t,a}\vee 1)}}\right] \leq \frac{\delta}{2KT}$.
Using union bound, we have $\mathbb{P}(\overline{\mathcal{G}}_1) \leq \frac{\delta}{2}$.

Denote the second event as $ \mathcal{G}_2=\left \lbrace  \left|
\sum_{t=1}^{T}\mathbb{E}_{a\sim\bm{p}_t}\left[ \sqrt{ \frac{1}{M_{t,a} \vee 1} }\right]
-\sum_{t=1}^{T}\sum_{a\in A_t}\sqrt{ \frac{1}{M_{t,a} \vee 1} }\right| \leq L\sqrt{2T
\log (\frac{4}{\delta})}\right \rbrace$. According to Lemma~\ref{Lemma:martingale_seq}, we have
$\mathbb{P}(\overline{\mathcal{G}}_2) \leq \frac{\delta}{2}$.

Then we can decompose $\mathbb{E} \left[\mathrm{RR}_T\right]$ as follows:
\begin{equation}
	\label{BayesRR-decomposed}
	\begin{aligned}
		\mathbb{E} \left[\mathrm{RR}_T\right] 
		& =\mathbb{E}\left[\sum_{t=1}^{T}\max\left\lbrace \sum_{a=1}^Kp^{*}_a\mu_{a}-\sum_{a=1}^K p_{t,a}\mu_{a}, 0\right\rbrace\right] \\
		& =\mathbb{E}\left[\Ind{\left\{\mathcal{G}_1 \text { and } \mathcal{G}_2\right\}} \sum_{t=1}^{T}\max\left\lbrace \sum_{a=1}^Kp^{*}_a\mu_{a}-\sum_{a=1}^K p_{t,a}\mu_{a}, 0\right\rbrace\right]
		+\underbrace{\mathbb{E}\left[\Ind{\left\{\overline{\mathcal{G}}_1 \text { or } \overline{\mathcal{G}}_2\right\} }\sum_{t=1}^{T}\max\left\lbrace \sum_{a=1}^Kp^{*}_a\mu_{a}-\sum_{a=1}^K p_{t,a}\mu_{a}, 0\right\rbrace\right]}_{\text {term } \mathrm{b}} \\
		& \stackrel{(a)}{\leq}\mathbb{E}\left[\Ind{\left\{\mathcal{G}_1 \text { and } \mathcal{G}_2\right\}} \sum_{t=1}^T \max\left\lbrace\mathbb{E}_{a \sim \bm{p}_t}\left[U_{t, a}-\mu_a\right],0\right\rbrace\right]
		+\mathbb{E}\left[\Ind{\left\{\mathcal{G}_1 \text { and } \mathcal{G}_2\right\}} \sum_{t=1}^T \max\left\lbrace\mathbb{E}_{a \sim \bm{p}^*}\left[\mu_a-U_{t, a}\right],0\right\rbrace\right]+\text { term b } \\
		& \stackrel{(b)}{\leq}  \mathbb{E}\left[\Ind{\left\{\mathcal{G}_1 \text { and } \mathcal{G}_2\right\}} \sum_{t=1}^T \mathbb{E}_{a \sim \bm{p}_t}\left[U_{t, a}-\mu_a\right]\right]+\text { term b }, 
	\end{aligned}
\end{equation}
where $(a)$ is from Lemma~\ref{lemma:BRT-decomposed} and $(b)$ is because $\mu_a\leq U_{t, a}$ under event $\mathcal{G}_1$.
We first upper bound the term $\mathbb{E}\left[\Ind{\left\{\mathcal{G}_1 \text {
and } \mathcal{G}_2\right\}} \sum_{t=1}^T \mathbb{E}_{a \sim
\bm{p}_t}\left[U_{t, a}-\mu_a\right]\right]$ as follows:
\begin{equation}
	\label{ineq:BayesRR-sub-term}
	\begin{aligned}
		&\mathbb{E}\left[\Ind{\left\{ \mathcal{G}_1 \text { and } \mathcal{G}_2\right\}} \sum_{t=1}^T \mathbb{E}_{a \sim \bm{p}_t}\left[U_{t, a}-\mu_a\right]\right] 
		 \stackrel{(a)}{\leq} \mathbb{E}\left[\Ind{\left\{ \mathcal{G}_1 \text { and } \mathcal{G}_2\right\} }\sum_{t=1}^T \mathbb{E}_{a \sim \bm{p}_t}\left[U_{t, a}-B_{t, a}\right]\right] \\
		& \leq \sqrt{2 \log \frac{4KT}{\delta}} \mathbb{E}\left[ \Ind{\left\{  \mathcal{G}_1 \text { and } \mathcal{G}_2\right\}} \sum_{t=1}^T \mathbb{E}_{a \sim \bm{p}_t}\left[\sqrt{\frac{1}{M_{t,a} \vee 1}}\right]\right] \\
		& \stackrel{(b)}{\leq} \sqrt{2 \log \frac{4KT}{\delta}}\left( L  + L\sqrt{2T \log \frac{4}{\delta}} +\frac{24K\log T}{q} + Ld^{*}(q)+2\sqrt{\frac{2LKT}{q}} \right)  ,
	\end{aligned}
\end{equation}
where $(a)$ is because $\mu_a\geq B_{t, a}$ under event $\mathcal{G}_1$ and
$(b)$ is from~\eqref{RR-second-term} under event $\mathcal{G}_2$.

For the term b in \eqref{BayesRR-decomposed}, 
we note that $\mathbb{E}\left[\Ind{\left\{\overline{\mathcal{G}}_1 \text { or } \overline{\mathcal{G}}_2\right\rbrace }\right] = \delta$. Then we have
\begin{equation}
	\label{ineq-term-b}
		\text { term b }  =\mathbb{E}\left[\Ind{\left\{\overline{\mathcal{G}}_1 \text { or } \overline{\mathcal{G}}_2\right\} }\sum_{t=1}^{T}\max\left\lbrace \sum_{a=1}^Kp^{*}_a\mu_{a}-\sum_{a=1}^K p_{t,a}\mu_{a}, 0\right\rbrace\right] 
		 \leq LT\delta .
\end{equation}
Combining~\eqref{BayesRR-decomposed},~\eqref{ineq:BayesRR-sub-term}
and~\eqref{ineq-term-b}, and setting $\delta = 1/LT$, the expected reward regret
can be upper bounded as follows,
\begin{equation}
\begin{aligned}
\mathbb{E} \left[\mathrm{RR}_T\right] & \leq  \sqrt{2 \log \frac{4KT}{\delta}}\left( L  + L\sqrt{2T \log \frac{4}{\delta}} +\frac{24K\log T}{q} + Ld^{*}(q)+2\sqrt{\frac{2LKT}{q}} \right) + LT\delta  \\
&\leq  \sqrt{4 \log (4LKT) }\left( L  + L\sqrt{2T \log (4LT) } +\frac{24K\log T}{q} + Ld^{*}(q)+2\sqrt{\frac{2LKT}{q}} \right) + 1,
\end{aligned}
\end{equation}
which holds for any $q\in\leftopen{0}{1}$. Then we can select the optimal $q$ to
minimize the upper bound for the expected reward regret as follows:
\begin{equation}
\mathbb{E} \left[\mathrm{RR}_T\right]=\widetilde{O}\left( \min_{q\in (0,1]}\left\lbrace \frac{K}{q}\sqrt{T}+Ld^{*}(q)\right\rbrace \right).
\end{equation}
This completes the proof of the expected reward regret upper bound of \ctsfd.

Combining Part 1 and Part 2 of the proof, we complete the proof of Theorem~\ref{the:fairness-reward-regret-thompson-sam-type}.

\end{proof}
\subsection{Proof of Theorem~\ref{the:fairness-reward-regret-op-ucb-type}}

We first prove the expected reward regret upper bound and then
prove the expected fairness regret upper bound of \cucbfdop.

\begin{proof}
\proofpart{1}{Proof of the Expected Reward Regret Upper Bound of \cucbfdop}
First, we have the following lemmas. 
\begin{lemma}
\label{Lemma:hoeffding_ineq_all_observations}
	For any $\delta\in (0,1)$,  with probability at least $1-\frac{\delta}{3}$, $\forall t>\lceil \frac{K}{L} \rceil$, $a\in[K]$, the expected reward vector $\bm{\mu} \in \mathcal{C}^{\pm}_t$.  
\end{lemma}
\begin{proof}
	According to Hoeffding's inequality, for any $ t>\lceil \frac{K}{L} \rceil$,
    $a\in[K]$, with probability at least $1-\frac{\delta}{3KT}$, we have
	\begin{equation}
 \label{ineq:Hoeffding_all_obser}
		\left|
\bar{\mu}_{t,a}-\mu_a\right| \leq \sqrt{\frac{\log(6KT /\delta)}{2N_{t,a}}},
\end{equation}
where $\bar{\mu}_{t,a} = \frac{1}{N_{t,a}}\sum_{s:s<t}R_{s,a}\Ind{\{a\in A_s\}} $ is the empirical average of all rewards (including unobserved ones) for arm $a$ up to round $t-1$.

By the result of~\eqref{ineq:Hoeffding_all_obser} and $\mu_a\in\left[0,1\right]$,
then for any $ t>\lceil K/L \rceil$,
    $a\in[K]$, with probability at least $1-\frac{\delta}{3KT}$,
we have 
\begin{equation}
    B^{-}_{t,a} \leq \max\left\lbrace\bar{\mu}_{t,a}-\sqrt{\frac{\log(6KT /\delta)}{2N_{t,a}}},0 \right\rbrace\leq \mu_a\leq \min\left\lbrace\bar{\mu}_{t,a}+\sqrt{\frac{\log(6KT /\delta)}{2N_{t,a}}},1 \right\rbrace \leq U^{+}_{t,a}.
\end{equation}

Using union bound, with probability at least $1-\frac{\delta}{3}$, $\forall t>\lceil K/L \rceil, a\in [K]$,  $\mu_a \in \left[ B^{-}_{t,a},  U^{+}_{t,a}\right]$ where $c_{t,a}=\sqrt{\frac{ \log (6KT/\delta)}{2N_{t,a}}}$.

This completes the proof of Lemma~\ref{Lemma:hoeffding_ineq_all_observations}.
\end{proof}

\begin{lemma} Fix some $q\in(0,1]$, $\forall t\in[T], a\in[K]$, the following inequality
holds with probability at least $1-\frac{\delta}{3}$,
 	\begin{equation}
		M_{t,a}\geq qN_{t-d_a(q),a} - \sqrt{\frac{\log(3KT /\delta)}{2}N_{t,a}}.
	\end{equation}
    \label{lemma:lancewicki-lemma5-adapted}
\end{lemma}

\begin{proof}
When $N_{t-d_a(q),a}=0$ or $N_{t,a}=0$, \eqref{lemma:lancewicki-lemma5-adapted} is true.
When $N_{t-d_a(q),a}\neq0$ and $N_{t,a}\neq0$,
    by definition of the quantile function, we have $\mathbb{P} 
    \left[ 
    D_{s,a} \leq d_a(q) \mid a\in A_s
    \right]
    \geq q$. Hence, by Hoeffding's inequality, we have

\begin{equation}
\begin{aligned}
    &\mathbb{P} 
    \left[ 
    \frac{1}{N_{t-d_a(q),a}}
    \sum_{s=1}^{t-d_a(q)-1} 
    \Ind{\{ D_{s,a} \leq d_a(q), a\in A_s \}}
    \leq 
    q - \sqrt{\frac{\log(3KT /\delta)}{2N_{t-d_a(q),a}}}
    \right]
    \\& \leq \mathbb{P} 
    \left[ 
    \frac{1}{N_{t-d_a(q),a}}
    \sum_{s=1}^{t-d_a(q)-1} 
    \Ind{\{ D_{s,a} \leq d_a(q), a\in A_s \}}
    \leq 
    \mathbb{P} 
    \left[ 
    D_{s,a} \leq d_a(q) \mid a\in A_s
    \right] - \sqrt{\frac{\log(3KT /\delta)}{2N_{t-d_a(q),a}}}
    \right] \leq\frac{\delta}{3KT},    
\end{aligned}
\end{equation}
which means
\begin{equation}
   \mathbb{P} 
    \left[ 
    \sum_{s=1}^{t-d_a(q)-1} 
    \Ind{\{ D_{s,a} \leq d_a(q), a\in A_s \}}
    \leq 
    qN_{t-d_a(q),a} - \sqrt{\frac{\log(3KT /\delta)}{2}N_{t-d_a(q),a}}
    \right]
    \leq \frac{\delta}{3KT}.
\end{equation}
Note that 
\begin{equation}
    M_{t,a} = \sum_{s=1}^{t-1} 
    \Ind{ \{ s+D_{s,a}< t, a \in A_s \} }
    \geq \sum_{s=1}^{t-d_a(q)-1}
    \Ind{\{ D_{s,a}\leq d_a(q), a \in A_s \}}.
\end{equation}
Then for any $a \in [K]$ and $t>\lceil K/L \rceil $, with probability at least $1-\frac{\delta}{3KT}$, we have
\begin{equation}
    M_{t,a} \geq qN_{t-d_a(q),a} - \sqrt{\frac{\log(3KT /\delta)}{2}N_{t-d_a(q),a}} \geq qN_{t-d_a(q),a} - \sqrt{\frac{\log(3KT /\delta)}{2}N_{t,a}}. 
\end{equation}
Apply union bound over all $a\in [K]$ and all $T$, we complete the proof of Lemma~\ref{lemma:lancewicki-lemma5-adapted}.
\end{proof}

\begin{lemma}
Fix some $q\in(0,1]$, $\forall t>\lceil K/L \rceil, a\in[K]$, with probability at least $1-\frac{\delta}{3}$,
\begin{equation}
\hat{\mu}^{+}_{t,a}-\hat{\mu}^{-}_{t,a} \leq \frac{N_{t,a}-N_{t-d_a(q),a}}{N_{t,a}}+1-q+\sqrt{\frac{\log(3KT /\delta)}{2N_{t,a}}}
\end{equation}
    \label{lemma:lancewicki-lemma6-adapted}
\end{lemma}

\begin{proof}
Fix some $q\in(0,1]$, $\forall t>\lceil K/L\rceil, a\in[K]$, we have
\begin{equation}
\begin{aligned}
\hat{\mu}^{+}_{t,a}-\hat{\mu}^{-}_{t,a} & =\frac{N_{t,a}-M_{t,a}}{N_{t,a}} \\
& =\frac{N_{t,a}-N_{t-d_a(q),a}+N_{t-d_a(q),a}-M_{t,a}}{N_{t,a}} \\
& \stackrel{(a)}{\leq} \frac{N_{t,a}-N_{t-d_a(q),a}}{N_{t,a}}+\frac{N_{t-d_a(q),a}(1-q)+\sqrt{\frac{N_{t,a}\log(3KT /\delta)}{2}}}{N_{t,a}} \\
& \leq \frac{N_{t,a}-N_{t-d_a(q),a}}{N_{t,a}}+1-q+\sqrt{\frac{\log(3KT /\delta)}{2N_{t,a}}},
\end{aligned}
\end{equation}
where $(a)$ is true with probability at least $1-\frac{\delta}{3}$ according to Lemma~\ref{lemma:lancewicki-lemma5-adapted}.
\end{proof}

Subsequently, we follow the same steps as in the proof of Lemma~\ref{Lemma:martingale_seq} and have the following lemma.
\begin{lemma}
	\label{Lemma:martingale_seq_all_observations}
    For any $\delta\in (0,1)$, with
    probability at least $1-\frac{\delta}{3}$, it holds that
	\begin{equation}
	\left| \sum_{t=1}^{T}\mathbb{E}_{a\sim\bm{p}_t}\left[  \sqrt{\frac{1}{N_{t,a}\vee 1}} \right]  -\sum_{t=1}^{T}\sum_{a\in A_t} \sqrt{\frac{1}{N_{t,a}\vee 1} }\right| \leq L\sqrt{2T \log \frac{6}{\delta}}.
	\end{equation}
\end{lemma}

Based on the above lemmas, for the reward regret of \cucbfdop, we have,
\begin{equation}
 \begin{aligned}
 	\label{eq:poof-reward_regret-ucb-op}
 	\mathrm{RR}_T&=\sum_{t=1}^{T}\max\left\lbrace \sum_{a=1}^Kp^{*}_a\mu_{a}-\sum_{a=1}^K p_{t,a}\mu_{a}, 0\right\rbrace  
 	\stackrel{(a)}{\leq} \left( \frac{K}{L} +1\right) L+\sum_{t=\lceil \frac{K}{L} \rceil+1}^{T}\sum_{a=1}^K (p_{t,a}\tilde{\mu}_{t,a}-p_{t,a}\mu_a) \\
 	&\leq K+L+\sum_{t=\lceil \frac{K}{L} \rceil+1}^{T}\sum_{a=1}^K(p_{t,a}\tilde{\mu}_{t,a}-p_{t,a}\mu_a) \\
 	&\leq K+L+\sum_{t=\lceil \frac{K}{L} \rceil+1}^{T}\sum_{a=1}^Kp_{t,a}\left(\hat{\mu}^{+}_{t,a}-\hat{\mu}^{-}_{t,a}+2\sqrt{\frac{\log(6KT/\delta)}{2N_{t,a}}}\right) \\
 	&\stackrel{(b)}{\leq} K+L+\sum_{t=\lceil \frac{K}{L} \rceil+1}^{T}\sum_{a=1}^Kp_{t,a}
 	 \left( \frac{N_{t,a}-N_{t-d_a(q),a}}{N_{t,a}}+1-q+3\sqrt{\frac{\log(6KT /\delta)}{2N_{t,a}}}  \right) 
         \\
        &\leq K+L+ L(1-q)T+  \sum_{t=\lceil \frac{K}{L} \rceil+1}^{T}\mathbb{E}_{a\sim\bm{p}_t}
        \left[ \frac{N_{t,a}-N_{t-d_a(q),a}}{N_{t,a}} \right]
        + 
        3\sqrt{  \frac{\log (6KT/\delta)}{2}}
        \sum_{t=\lceil \frac{K}{L} \rceil+1}^{T} 
        \mathbb{E}_{a\sim\bm{p}_t}\left[ \sqrt{\frac{1}{N_{t,a}}} \right]     \\
        &\leq K+L+ L(1-q)T+  d^{*}(q)\sum_{t=1}^{T}\mathbb{E}_{a\sim\bm{p}_t}
        \left[ \sqrt{\frac{1}{N_{t,a}\vee 1}} \right]
        + 
        3\sqrt{  \frac{\log (6KT/\delta)}{2}}
        \sum_{t=1}^{T} 
        \mathbb{E}_{a\sim\bm{p}_t}\left[ \sqrt{\frac{1}{N_{t,a}\vee 1}} \right]      \\
 	&\stackrel{(c)}{\leq} K+L+L(1-q)T
    +\left(d^{*}(q)+3\sqrt{  \frac{\log (6KT/\delta)}{2}}\right)\left( L\sqrt{2T \log \frac{6}{\delta}} +\sum_{t=1}^{T}\sum_{a\in A_t} \sqrt{\frac{1}{N_{t,a}\vee 1}} \right)  
 	 \\
        &\leq K+L+L(1-q)T
    +\left(d^{*}(q)+3\sqrt{  \frac{\log (6KT/\delta)}{2}}\right)\left( L\sqrt{2T \log \frac{6}{\delta}} +K\int_{1}^{\frac{LT}{K}}\sqrt{\frac{1}{x}}dx \right)    
 	 \\
        &\leq K+L+L(1-q)T
    +\left(d^{*}(q)+3\sqrt{  \frac{\log (6KT/\delta)}{2}}\right)\left( L\sqrt{2T \log \frac{6}{\delta}} +2\sqrt{LKT} \right)    
,  
 \end{aligned}
 \end{equation}
where $(a)$ is from Line~\ref{alg:fairness-delay-ucb-type-op:max} in Algorithm~\ref{alg:fairness-delay-op-ucb-type} and Lemma~\ref{Lemma:hoeffding_ineq_all_observations}, $(b)$ follows from Lemma \ref{lemma:lancewicki-lemma6-adapted}, and in $(c)$ we apply  Lemma~\ref{Lemma:martingale_seq_all_observations}.
Therefore, when $T > K$, with probability at least $1-\delta$, we have

\begin{equation}
	\mathrm{RR}_T \leq K+L+L(1-q)T
    +\left(d^{*}(q)+3\sqrt{  \frac{\log (6KT/\delta)}{2}}\right)\left( L\sqrt{2T \log \frac{6}{\delta}} +2\sqrt{LKT} \right).
\end{equation}

Finally, setting $\delta=\frac{1}{LT}$, the expected reward regret can be upper bounded as
\begin{equation}
 \begin{aligned}
    \mathbb{E}\left[\mathrm{RR}_T\right]&\leq K+L+L(1-q)T
    +\left(d^{*}(q)+3\sqrt{  \frac{\log (6KT/\delta)}{2}}\right)\left( L\sqrt{2T \log \frac{6}{\delta}} +2\sqrt{LKT} \right)+LT\delta\\
 &\leq K+L+L(1-q)T
    +\left(d^{*}(q)+3\sqrt{ \log (6LKT)}\right)\left( L\sqrt{2T \log \frac{6}{\delta}} +2\sqrt{LKT} \right)+1.
 \end{aligned}
\end{equation}
By choosing the optimal $q$ to obtain the minimum upper bound, we have
\begin{equation}
	\mathbb{E}\left[\mathrm{RR}_T\right]= \widetilde{O}\left( \min_{q\in (0,1]}\left\lbrace L(1-q)T+ Kd^{*}(q)\sqrt{T} \right\rbrace \right).
\end{equation}

This completes the proof of the expected reward regret upper bound of \cucbfdop.

\proofpart{2}{Proof of the Expected Fairness Regret Upper Bound of \cucbfdop}

For any $\delta \in (0,1)$, $t>\lceil K/L \rceil$, with probability $1-\delta$, we have
\begin{equation}
	\label{eq:instantaneous-fairness-reg-op}
\begin{aligned}
	\sum_{a=1}^K\left|p_{t,a}-p_a^*\right|  	
 	&\stackrel{(a)}{\leq} \sum_{a=1}^{K}\frac{2Mp_{t,a}L}{\lambda}\left( \tilde{\mu}_{t, a}-\mu_{ a}\right) 
	\leq \sum_{a=1}^{K}\frac{2Mp_{t,a}L}{\lambda}\left( \hat{\mu}^{+}_{t,a}-\hat{\mu}^{-}_{t,a}+2\sqrt{\frac{\log(6KT/\delta)}{2N_{t,a}}}\right) \\
	&\stackrel{(b)}{\leq} \sum_{a=1}^{K}\frac{2Mp_{t,a}L}{\lambda}
 \left(\frac{N_{t,a}-N_{t-d_a(q),a}}{N_{t,a}}+1-q+3\sqrt{\frac{\log(6KT/\delta)}{2N_{t,a}}}\right) \\
 &=\frac{2ML}{\lambda}\mathbb{E}_{a\sim\bm{p}_t}
        \left[ \frac{N_{t,a}-N_{t-d_a(q),a}}{N_{t,a}} +1-q \right]
        + 
        \frac{3ML}{\lambda}\sqrt{2 \log \frac{6KT}{\delta}}
        \mathbb{E}_{a\sim\bm{p}_t}\left[ \sqrt{\frac{1}{N_{t,a}}} \right] \\
 &\leq \frac{2ML^2}{\lambda}(1-q)+ \frac{2MLd^{*}(q)}{\lambda}\mathbb{E}_{a\sim\bm{p}_t}
        \left[ \sqrt{\frac{1}{N_{t,a}} } \right]
        + 
        \frac{3ML}{\lambda}\sqrt{2 \log \frac{6KT}{\delta}}
        \mathbb{E}_{a\sim\bm{p}_t}\left[ \sqrt{\frac{1}{N_{t,a}}} \right],
\end{aligned}
    \end{equation}
where $(a)$ is from~\eqref{eq:instantaneous-fairness-reg}, $(b)$ follows from Lemma \ref{lemma:lancewicki-lemma6-adapted}. Then the fairness regret can be upper bounded as follows,
\begin{equation}
\begin{aligned}
	&\mathrm{FR}_T=\sum_{t=1}^{T}\sum_{a=1}^K\left| p_a^*-p_{t,a}\right| 
	\leq \left( \frac{K}{L} +1\right)L +\sum_{t=\lceil \frac{K}{L} \rceil+1}^{T}\sum_{a=1}^K\left| p_a^*-p_{t,a}\right|  \\
	&\leq K+L
	+\frac{2ML^{2}}{\lambda}(1-q)T
    +\left(\frac{2MLd^{*}(q)}{\lambda}+\frac{3ML}{\lambda}\sqrt{2 \log \frac{6KT}{\delta}}\right)\sum_{t=1}^{T}\mathbb{E}_{a\sim\bm{p}_t}
        \left[ \sqrt{\frac{1}{N_{t,a}\vee 1}}\right]
	  \\
  	&\stackrel{(a)}{\leq} K+L
	+\frac{2ML^{2}}{\lambda}(1-q)T
    +\left(\frac{2MLd^{*}(q)}{\lambda}+\frac{3ML}{\lambda}\sqrt{2 \log \frac{6KT}{\delta}}\right)\left( L\sqrt{2T \log \frac{4}{\delta}} +\sum_{t=1}^{T}\sum_{a\in A_t} \sqrt{\frac{1}{N_{t,a}\vee 1}} \right)\ \\     
	& 	\leq K+L
	+\frac{2ML^{2}}{\lambda}(1-q)T
    +\left(\frac{2MLd^{*}(q)}{\lambda}+\frac{3ML}{\lambda}\sqrt{2 \log \frac{6KT}{\delta}}\right)\left( L\sqrt{2T \log \frac{4}{\delta}} +K\int_{1}^{\frac{LT}{K}}\sqrt{\frac{1}{x}}dx \right)    
 	 \\
             & \leq K+L
	+\frac{2ML^{2}}{\lambda}(1-q)T
    +\left(\frac{2MLd^{*}(q)}{\lambda}+\frac{3ML}{\lambda}\sqrt{2 \log \frac{6KT}{\delta}}\right)\left( L\sqrt{2T \log \frac{4}{\delta}} +2\sqrt{LKT} \right),
\end{aligned}
\end{equation}
where $(a)$ is from Lemma~\ref{Lemma:martingale_seq_all_observations}.

Therefore, when $T > K$, with probability at least $1-\delta$, we have

\begin{equation}
	\mathrm{FR}_T \leq K+L
	+\frac{2ML^{2}}{\lambda}(1-q)T
    +\left(\frac{2MLd^{*}(q)}{\lambda}+\frac{3ML}{\lambda}\sqrt{2 \log \frac{6KT}{\delta}}\right)\left( L\sqrt{2T \log \frac{4}{\delta}} +2\sqrt{LKT} \right).
\end{equation}

Finally, setting $\delta=\frac{1}{LT}$, the expected reward regret can be upper bounded as
\begin{equation}
 \begin{aligned}
	\mathbb{E}\left[\mathrm{FR}_T\right] &\leq K+L
	+\frac{2ML^{2}}{\lambda}(1-q)T
    +\left(\frac{2MLd^{*}(q)}{\lambda}+\frac{3ML}{\lambda}\sqrt{2 \log \frac{6KT}{\delta}}\right)\left( L\sqrt{2T \log \frac{4}{\delta}} +2\sqrt{LKT} \right)+LT\delta\\
    & \leq K+L
	+\frac{2ML^{2}}{\lambda}(1-q)T
    +\left(\frac{2MLd^{*}(q)}{\lambda}+\frac{3ML}{\lambda}\sqrt{4 \log (6LKT)}\right)\left( L\sqrt{2T \log (4LT)} +2\sqrt{LKT} \right)+1\\
 & = \widetilde{O}\left( \min_{q\in (0,1]}\left\lbrace \frac{MLK}{\lambda}\left((1-q)T+ d^{*}(q)\sqrt{T}\right) \right\rbrace \right).
 \end{aligned}
\end{equation}

This completes the proof of the expected fairness regret upper bound of \cucbfdop.

Combining Part 1 and Part 2 of the proof, we complete the proof of Theorem~\ref{the:fairness-reward-regret-op-ucb-type}.

\end{proof}
\subsection{Proof of Theorem~\ref{the:fairness-reward-regret-op-thompson-sam-type}}

\begin{proof} For \ctsfdop, we take a different approach and prove the expected
fairness regret upper bound first and then prove the expected reward regret
upper bound using the expected fairness regret. 
To take the prior into account, the expectation is simultaneously taken over the draws of $\mu_a$, and the algorithm’s internal randomization over rewards and actions.

\proofpart{1}{Proof of the Expected Fairness Regret Upper Bound of \ctsfdop}

For each arm $a$, the optimistic posterior $\mathcal{Q}_{t,a}^{+}$ of $\tilde{\mu}^{+}_{t,a}$ at round $t$ is a Beta distribution,
$$\mathcal{P}^{+}_{t,a}( \cdot\mid \mathcal{H}_{t} ) :=  \mathrm{Beta} (u^{+}_{t,a}, v^{+}_{t,a}),$$
where $u_{t,a}^{+}=1+\sum_{s:s+D_{s,a} < t} \Ind{ \{a \in
A_{s}\}}\Ind{\{R_{t,a}=1\}}+(N_{t,a}-M_{t,a})$, $v_{t,a}^{+}=1+\sum_{s:s+D_{s,a}
< t} \Ind{ \{a \in A_{s}\}}\Ind{\{R_{t,a}=0\}}$, and $\mathcal{H}_{t} =
\left(\bm{p}_1, A_1, \bm{Y}_1, ..., \bm{p}_{t-1}, A_{t-1}, \bm{Y}_{t-1} \right)$
denotes the observation and decision history up to round $t-1$, which consists
of all the previous selection vectors, selected arm sets, and received feedback.
Denote $\Ddot{\mu}_{t,a}^{+}$ as the expectation of the distribution $\mathcal{P}^{+}_{t,a}( \cdot\mid \mathcal{H}_{t} )$ at round $t$,

Similarly, 
for each arm $a$, the pessimistic posterior $\mathcal{Q}_{t,a}^{-}$ of $\tilde{\mu}^{-}_{t,a}$ at round $t$ is a Beta distribution,
$$\mathcal{P}^{-}_{t,a}( \cdot\mid \mathcal{H}_{t} ) :=  \mathrm{Beta} (u^{-}_{t,a}, v^{-}_{t,a}),$$
where $u^{-}_{t,a}=1+\sum_{s:s+D_{s,a} < t} \Ind{ \{a \in A_{s}\}}\Ind{\{R_{t,a}=1\}}$ and $v^{-}_{t,a}=1+\sum_{s:s+D_{s,a} < t} \Ind{ \{a \in A_{s}\}}\Ind{\{R_{t,a}=0\}}+(N_{t,a}-M_{t,a})$.
Denote $\Ddot{\mu}_{t,a}^{-}$ as the expectation of the distribution $\mathcal{P}^{-}_{t,a}( \cdot\mid \mathcal{H}_{t} )$ at round $t$.

First, referring to~\eqref{eq:instantaneous-fairness-reg}, we can upper bound the expected 
per-step fairness regret $fr_t$ as follows,
\begin{equation}
\begin{aligned}
	\mathbb{E}[fr_t] 
	&\stackrel{(a)}{\leq} \frac{2L}{\lambda} 	\mathbb{E}_{\mathcal{H}_{t}}\left[\mathbb{E}_{\bm{\mu},\tilde{\bm{\mu}}_t^{+},\tilde{\bm{\mu}}_t^{-}}\left[\mathbb{E}_{a\sim\bm{p}_t}\left[\left| f((\tilde{\mu}^{+}_{t,a}+\tilde{\mu}^{-}_{t,a})/2)- f(\mu_{ a})\right| \right] | \mathcal{H}_{t}\right] \right] 
	\\
         &\stackrel{(b)}{\leq} \frac{ML}{\lambda} 	\mathbb{E}_{\mathcal{H}_{t}}\left[\mathbb{E}_{\bm{\mu},\tilde{\bm{\mu}}_t^{+},\tilde{\bm{\mu}}_t^{-}}\left[\mathbb{E}_{a\sim\bm{p}_t}\left[\left| \tilde{\mu}^{+}_{t,a}+\tilde{\mu}^{-}_{t,a}- 2\mu_{a}\right| \right]| \mathcal{H}_{t}\right]\right]\\
          &\leq \frac{ML}{\lambda} \left\lbrace 	\mathbb{E}_{\mathcal{H}_{t}}\left[\mathbb{E}_{\bm{\mu},\tilde{\bm{\mu}}_t^{+},\tilde{\bm{\mu}}_t^{-}}\left[\mathbb{E}_{a\sim\bm{p}_t}\left[\left| \tilde{\mu}^{+}_{t,a}- \mu_{a}\right| \right]| \mathcal{H}_{t}\right]\right] + \mathbb{E}_{\mathcal{H}_{t}}\left[\mathbb{E}_{\bm{\mu},\tilde{\bm{\mu}}_t^{+},\tilde{\bm{\mu}}_t^{-}}\left[\mathbb{E}_{a\sim\bm{p}_t}\left[\left| \mu_{a}-\tilde{\mu}^{-}_{t,a}\right| \right]| \mathcal{H}_{t}\right]\right]
 \right \rbrace\\
    &\leq \frac{ML}{\lambda} \left\lbrace 
\underbrace{\mathbb{E}_{\mathcal{H}_{t}}\left[\mathbb{E}_{\bm{\mu},\tilde{\bm{\mu}}_t^{+},\tilde{\bm{\mu}}_t^{-}}\left[\mathbb{E}_{a\sim\bm{p}_t}\left[
\left| \tilde{\mu}^{+}_{t,a}-\Ddot{\mu}_{t,a}^{+}\right| 
\right]| \mathcal{H}_{t}\right]\right]}_{\text{term c}}
 + \underbrace{\mathbb{E}_{\mathcal{H}_{t}}\left[\mathbb{E}_{\bm{\mu},\tilde{\bm{\mu}}_t^{+},\tilde{\bm{\mu}}_t^{-}}\left[\mathbb{E}_{a\sim\bm{p}_t}\left[
\left|\Ddot{\mu}_{t,a}^{+}-\mu_a\right| 
\right]| \mathcal{H}_{t}\right]\right]}_{\text{term d}}
\right.
\\&\quad\left.+\underbrace{\mathbb{E}_{\mathcal{H}_{t}}\left[\mathbb{E}_{\bm{\mu},\tilde{\bm{\mu}}_t^{+},\tilde{\bm{\mu}}_t^{-}}\left[\mathbb{E}_{a\sim\bm{p}_t}\left[\left| 
\mu_{a}-\Ddot{\mu}_{t,a}^{-}
\right| \right]| \mathcal{H}_{t}\right]\right]}_{\text{term e}}
 + \underbrace{\mathbb{E}_{\mathcal{H}_{t}}\left[\mathbb{E}_{\bm{\mu},\tilde{\bm{\mu}}_t^{+},\tilde{\bm{\mu}}_t^{-}}\left[\mathbb{E}_{a\sim\bm{p}_t}\left[
 \left| 
\Ddot{\mu}_{t,a}^{-}-\tilde{\mu}^{-}_{t,a}
 \right| \right]| \mathcal{H}_{t}\right]\right]}_{\text{term f}}  \right\rbrace ,
\end{aligned}
\label{eq:frt-terms}
\end{equation}
where $(a)$ follows from Assumption~\ref{Minimum-Merit}~\ref{Minimum-Merit:minmu} and $(b)$ follows from Assumption~\ref{Lipschitz-Continuity}.

For the expectation concerning
$\bm{\mu},\tilde{\bm{\mu}}_t^{+},\tilde{\bm{\mu}}_t^{-}$ in the term c in~\eqref{eq:frt-terms}, we have 
\begin{equation}
\label{term_c}
\begin{aligned}
\mathbb{E}_{\bm{\mu},\tilde{\bm{\mu}}_t^{+},\tilde{\bm{\mu}}_t^{-}}\left[\mathbb{E}_{a\sim\bm{p}_t}\left[
\left| \tilde{\mu}^{+}_{t,a}-\Ddot{\mu}_{t,a}^{+}\right| 
\right]| \mathcal{H}_{t}\right]&=
\mathbb{E}_{\bm{\mu},\tilde{\bm{\mu}}_t^{-}}\left[\sum_{a=1}^K p_{t,a}\mathbb{E}_{\tilde{\bm{\mu}}^{+}_t}\left[
\left| \tilde{\mu}^{+}_{t,a}-\Ddot{\mu}_{t,a}^{+}\right| 
\mid \mathcal{H}_{t}, \bm{\mu},\tilde{\bm{\mu}}_t^{-} \right] \Big| \mathcal{H}_{t} \right]\\
&\stackrel{(a)}{\leq} \mathbb{E}_{\bm{\mu},\tilde{\bm{\mu}}_t^{-}}\left[\sum_{a=1}^K p_{t,a}\sqrt{\mathbb{E}_{\tilde{\bm{\mu}}^{+}_t}\left[
\left| \tilde{\mu}^{+}_{t,a}-\Ddot{\mu}_{t,a}^{+}\right|^2 
\mid \mathcal{H}_{t}, \bm{\mu},\tilde{\bm{\mu}}_t^{-}\right]}\Big| \mathcal{H}_{t}\right]\\
&\stackrel{(b)}{\leq}\mathbb{E}_{\bm{\mu},\tilde{\bm{\mu}}_t^{-}}\left[\sum_{a=1}^Kp_{t,a}\sqrt{\frac{1}{N_{t,a}\vee 1}}\Big| \mathcal{H}_{t}\right]\\&=\mathbb{E}_{\bm{\mu},\tilde{\bm{\mu}}_t^{-}}\left[\mathbb{E}_{a\sim\bm{p}_t}\left[\sqrt{\frac{1}{N_{t,a} \vee 1}}\right]\Big| \mathcal{H}_{t}\right],
\end{aligned}
\end{equation}
where $(a)$ comes from the fact that $\mathbb{E}\left[ mn \right] \leq  \sqrt{\mathbb{E}\left[ m^2 \right]} \sqrt{\mathbb{E}\left[ n^2 \right]} $, $(b)$ is because $\mathbb{E}_{\tilde{\bm{\mu}}^{+}_t}\left[\tilde{\mu}_{t,a}^{+} - \Ddot{\mu}_{t,a}^{+}\mid \mathcal{H}_{t}, \bm{\mu},\tilde{\bm{\mu}}_t^{-}\right] = 0$ and   $\mathrm{Var}_{\tilde{\bm{\mu}}^{+}_t}\left[\tilde{\mu}_{t,a}^{+} - \Ddot{\mu}_{t,a}^{+}\mid \mathcal{H}_{t}, \bm{\mu},\tilde{\bm{\mu}}_t^{-} \right]= \frac{u^{+}_{t,a}v^{+}_{t,a}}{(u^{+}_{t,a}+v^{+}_{t,a}+1)(u^{+}_{t,a}+v^{+}_{t,a})^2} \leq \frac{1}{N_{t,a}\vee 1}$.

Similarly, for the expectation concerning
$\bm{\mu},\tilde{\bm{\mu}}_t^{+},\tilde{\bm{\mu}}_t^{-}$ in the term f in~\eqref{eq:frt-terms}, 
we also have 
\begin{equation}
\label{term_f}
\begin{aligned}
\mathbb{E}_{\bm{\mu},\tilde{\bm{\mu}}_t^{+},\tilde{\bm{\mu}}_t^{-}}\left[\mathbb{E}_{a\sim\bm{p}_t}\left[
 \left| 
\Ddot{\mu}_{t,a}^{-}-\tilde{\mu}^{-}_{t,a}
 \right| \right]| \mathcal{H}_{t}\right] \leq \mathbb{E}_{\bm{\mu},\tilde{\bm{\mu}}_t^{+}}\left[\sum_{a=1}^K p_{t,a}\sqrt{\frac{1}{N_{t,a}\vee 1}}\Big| \mathcal{H}_{t}\right]=\mathbb{E}_{\bm{\mu},\tilde{\bm{\mu}}_t^{+}}\left[\mathbb{E}_{a\sim\bm{p}_t}\left[\sqrt{\frac{1}{N_{t,a}\vee 1}}\right]\Big| \mathcal{H}_{t}\right],
\end{aligned}
\end{equation}
For the term d in~\eqref{eq:frt-terms}, we have

\begin{equation}
\begin{aligned}
\mathbb{E}_{\bm{\mu},\tilde{\bm{\mu}}_t^{+},\tilde{\bm{\mu}}_t^{-}}\left[\mathbb{E}_{a\sim\bm{p}_t}\left[
\left|\Ddot{\mu}_{t,a}^{+}-\mu_a\right| 
\right]| \mathcal{H}_{t}\right]=
\mathbb{E}_{\tilde{\bm{\mu}}_t^{+},\tilde{\bm{\mu}}_t^{-}}\left[\sum_{a=1}^K p_{t,a}\mathbb{E}_{\bm{\mu}}\left[
\left| \Ddot{\mu}_{t,a}^{+}-\mu_a\right| 
\mid \mathcal{H}_{t}, \tilde{\bm{\mu}}_t^{+},\tilde{\bm{\mu}}_t^{-} \right] \Big| \mathcal{H}_{t} \right].
\end{aligned}
\end{equation}

For the term 
$\mathbb{E}_{\bm{\mu}}\left[
\left| \Ddot{\mu}_{t,a}^{+}-\mu_a\right| 
\mid \mathcal{H}_{t}, \tilde{\bm{\mu}}_t^{+},\tilde{\bm{\mu}}_t^{-} \right] $,
$\forall a\in[K]$, $t\in[T]$, with probability at least $1-\frac{\delta^{\prime}}{2}$,
we have
\begin{equation}
\label{ineq:optimistic-bound}
\begin{aligned}
\mathbb{E}_{\bm{\mu}}\left[
\left| \Ddot{\mu}_{t,a}^{+}-\mu_a\right| 
\mid \mathcal{H}_{t}, \tilde{\bm{\mu}}_t^{+},\tilde{\bm{\mu}}_t^{-}\right]
& \leq \sqrt{\mathbb{E}_{\bm{\mu}}\left[
\left| \Ddot{\mu}_{t,a}^{+}-\mu_a\right|^2 
\mid \mathcal{H}_{t}, \tilde{\bm{\mu}}_t^{+},\tilde{\bm{\mu}}_t^{-}\right]}
\stackrel{(a)}{\leq} \sqrt{\frac{1}{M_{t,a}+ 1}+\left( \Ddot{\mu}_{t,a}^{+}-\Ddot{\mu}_{t,a} \right)^2}\\
& \leq \sqrt{\frac{1}{M_{t,a}+ 1}}+\left| \Ddot{\mu}_{t,a}^{+}-\Ddot{\mu}_{t,a} \right|\\
&   \leq \sqrt{\frac{1}{M_{t,a}+ 1}}+ \frac{N_{t,a}-M_{t,a}}{N_{t,a}\vee 1} 
= \sqrt{\frac{1}{M_{t,a}+ 1}}+\hat{\mu}^{+}_{t,a}-\hat{\mu}^{-}_{t,a}\\
& \stackrel{(b)}{\leq} \sqrt{\frac{1}{M_{t,a}+ 1}}+ \frac{N_{t,a}-N_{t-d_a(q),a}}{N_{t,a}\vee 1}+1-q+\sqrt{\frac{\log(2KT /\delta^{\prime})}{2(N_{t,a}\vee 1)}},
\end{aligned}
\end{equation}
where $(a)$ is because $\mathbb{E}_{\bm{\mu}}\left[
 \Ddot{\mu}_{t,a}^{+}-\mu_a 
\mid \mathcal{H}_{t}, \tilde{\bm{\mu}}_t^{+},\tilde{\bm{\mu}}_t^{-}\right] = \mathbb{E}_{\bm{\mu}}\left[
 \Ddot{\mu}_{t,a}^{+}-\Ddot{\mu}_{t,a} 
\mid \mathcal{H}_{t}, \tilde{\bm{\mu}}_t^{+},\tilde{\bm{\mu}}_t^{-}\right]$ and $\mathrm{Var}_{\bm{\mu}}\left[
 \Ddot{\mu}_{t,a}^{+}-\mu_a 
\mid \mathcal{H}_{t}, \tilde{\bm{\mu}}_t^{+},\tilde{\bm{\mu}}_t^{-}\right] = \frac{u_{t,a}v_{t,a}}{(u_{t,a}+v_{t,a}+1)(u_{t,a}+v_{t,a})^2} \leq \frac{1}{M_{t,a}+ 1}$  and $(b)$ is from Lemma~\ref{lemma:lancewicki-lemma6-adapted}.

Similarly, for the term e in~\eqref{eq:frt-terms}, we have
\begin{equation}
\begin{aligned}
\mathbb{E}_{\bm{\mu},\tilde{\bm{\mu}}_t^{+},\tilde{\bm{\mu}}_t^{-}}\left[\mathbb{E}_{a\sim\bm{p}_t}\left[
\left|\mu_{a}-\Ddot{\mu}_{t,a}^{-}\right| 
\right]| \mathcal{H}_{t}\right]=
\mathbb{E}_{\tilde{\bm{\mu}}_t^{+},\tilde{\bm{\mu}}_t^{-}}\left[\sum_{a=1}^K p_{t,a}\mathbb{E}_{\bm{\mu}}\left[
\left| \mu_{a}-\Ddot{\mu}_{t,a}^{-}\right| 
\mid \mathcal{H}_{t}, \tilde{\bm{\mu}}_t^{+},\tilde{\bm{\mu}}_t^{-} \right] \Big| \mathcal{H}_{t} \right].
\end{aligned}
\end{equation}

$\forall a\in[K]$, $t\in[T]$, with probability at least $1-\frac{\delta^{\prime}}{2}$,
we have
\begin{equation}
\label{ineq:pessimistic-bound}
\begin{aligned}
\mathbb{E}_{\bm{\mu}}\left[\left| 
\mu_{a}-\Ddot{\mu}_{t,a}^{-}
\right| \Big| \mathcal{H}_{t}, \tilde{\bm{\mu}}_t^{+},\tilde{\bm{\mu}}_t^{-}\right]\leq
\sqrt{\frac{1}{M_{t,a}+ 1}}+ \frac{N_{t,a}-N_{t-d_a(q),a}}{N_{t,a}\vee 1}+1-q+\sqrt{\frac{\log(2KT /\delta^{\prime})}{2(N_{t,a}\vee 1)}}.
\end{aligned}
\end{equation}

Denote event $\mathcal{E}_t=\left\lbrace\forall  a\in [K], \ \eqref{ineq:optimistic-bound}~\text{and}~\eqref{ineq:pessimistic-bound}~\text{hold.} \right \rbrace$.
We have 

\begin{equation}
\label{term_de}
\begin{aligned}
&\mathbb{E}_{\mathcal{H}_{t}}\left[
\mathbb{E}_{\bm{\mu},\tilde{\bm{\mu}}_t^{+},\tilde{\bm{\mu}}_t^{-}}\left[\mathbb{E}_{a\sim\bm{p}_t}\left[\left|\Ddot{\mu}_{t,a}^{+}-\mu_a\right|+\left| 
\mu_{a}-\Ddot{\mu}_{t,a}^{-}
\right| 
\right]\Big| \mathcal{H}_{t}\right]\right]\\
&=\mathbb{E}_{\mathcal{H}_{t}}\left[
\mathbb{E}_{\bm{\mu},\tilde{\bm{\mu}}_t^{+},\tilde{\bm{\mu}}_t^{-}}\left[\mathbb{E}_{a\sim\bm{p}_t}\left[\left|\Ddot{\mu}_{t,a}^{+}-\mu_a\right|+\left| 
\mu_{a}-\Ddot{\mu}_{t,a}^{-}
\right| 
\right]\Big| \mathcal{H}_{t},\mathcal{E}_t
\right]\mathbb{P}\left(\mathcal{E}_t\right)\right.
\\&\left.\quad+
\mathbb{E}_{\bm{\mu},\tilde{\bm{\mu}}_t^{+},\tilde{\bm{\mu}}_t^{-}}\left[\mathbb{E}_{a\sim\bm{p}_t}\left[\left|\Ddot{\mu}_{t,a}^{+}-\mu_a\right|+\left| 
\mu_{a}-\Ddot{\mu}_{t,a}^{-}
\right| 
\right]\Big| \mathcal{H}_{t},\overline{\mathcal{E}}_t\right]\mathbb{P}\left(\overline{\mathcal{E}}_t\right)
\right]
\\&\leq2\mathbb{E}_{\mathcal{H}_{t}}\left[\mathbb{E}_{\tilde{\bm{\mu}}_t^{+},\tilde{\bm{\mu}}_t^{-}}\left[\mathbb{E}_{a\sim\bm{p}_t}\left[\sqrt{\frac{1}{M_{t,a}+ 1}}+ \frac{N_{t,a}-N_{t-d_a(q),a}}{N_{t,a}\vee 1}+1-q+\sqrt{\frac{\log(2KT /\delta^{\prime})}{2(N_{t,a}\vee 1)}}\right]\Big| \mathcal{H}_t\right]\right]
\\&\quad+
\underbrace{\mathbb{E}_{\mathcal{H}_{t}}\left[\sum_{a}p_{t,a}\mathbb{E}_{\bm{\mu},\tilde{\bm{\mu}}_t^{+},\tilde{\bm{\mu}}_t^{-}}\left[
\left|\Ddot{\mu}_{t,a}^{+}-\mu_a\right|+\left| 
\mu_{a}-\Ddot{\mu}_{t,a}^{-}\right|\Big|\mathcal{H}_{t},\overline{\mathcal{E}}_t\right]\mathbb{P}\left(\overline{\mathcal{E}}_t\right)\right]
}_{\text {term g }} .
\end{aligned}
\end{equation}

For the term g in~\eqref{term_de}, note that
\begin{equation}
	\label{eq:instanstaneous-fairness-regret-sub-term-op}
	\begin{aligned}
 \mathbb{E}_{\bm{\mu},\tilde{\bm{\mu}}_t^{+},\tilde{\bm{\mu}}_t^{-}}\left[
\left|\Ddot{\mu}_{t,a}^{+}-\mu_a\right|^2\Big|\mathcal{H}_{t},\overline{\mathcal{E}}_t\right]		
		& 
  \stackrel{(a)}{\leq} 2\left(\Ddot{\mu}_{t,a}^{+}- \Ddot{\mu}_{t,a}\right)^2
  +2\mathbb{E}_{\bm{\mu}}\left[\left( \mu_{a}- \Ddot{\mu}_{t,a}\right)^2 \mid \mathcal{H}_t,\overline{\mathcal{E}}_t\right]\\
		& 
    \stackrel{(b)}{=} 2\left(\Ddot{\mu}_{t,a}^{+}-\Ddot{\mu}_{t,a}\right)^2
		+\frac{2}{M_{t,a}+ 1}
		\\&\leq 2\left(\Ddot{\mu}_{t,a}^{+}- \Ddot{\mu}_{t,a}\right)^2
		+2,
	\end{aligned}
\end{equation}
where $(a)$ follows the inequality $(m+n)^2 \leq 2m^2 + 2n^2$ and $(b)$ is because 
$\mathbb{E}_{\bm{\mu}}\left[\mu_{a} - \Ddot{\mu}_{t,a}\mid \mathcal{H}_t, \overline{\mathcal{E}}_t\right] = 0$ and   $\mathrm{Var}_{\bm{\mu}}\left[\mu_{a} - \Ddot{\mu}_{t,a}\mid \mathcal{H}_t, \overline{\mathcal{E}}_t\right] = \frac{u_{t,a}v_{t,a}}{(u_{t,a}+v_{t,a}+1)(u_{t,a}+v_{t,a})^2} \leq \frac{1}{M_{t,a}+ 1}$. 
Similarly, we also have
\begin{equation}
	\label{eq:instanstaneous-fairness-regret-sub-term2}
 \mathbb{E}_{\bm{\mu},\tilde{\bm{\mu}}_t^{+},\tilde{\bm{\mu}}_t^{-}}\left[
\left| 
\mu_{a}-\Ddot{\mu}_{t,a}^{-}\right|^2\Big|\mathcal{H}_{t},\overline{\mathcal{E}}_t\right]
\leq2\left(\Ddot{\mu}_{t,a}^{-}- \Ddot{\mu}_{t,a}\right)^2+2,
\end{equation}

Next, we can upper bound term g in~\eqref{term_de} as follows:
\begin{equation}
\label{term_g}
	\begin{aligned}
		\text { term } \mathrm{g} & \leq \mathbb{E}_{\mathcal{H}_{t}}\left[\mathbb{P}\left(\overline{\mathcal{E}}_t\right) \sqrt{\sum_{a=1}^K p_{t,a}^2} \left(\sqrt{\sum_{a=1}^K  \mathbb{E}_{\bm{\mu}}\left[ \left( \Ddot{\mu}_{t,a}^{+}-\mu_a\right)^2 | \mathcal{H}_{t},\overline{\mathcal{E}}_t \right]  }+\sqrt{\sum_{a=1}^K  \mathbb{E}_{\bm{\mu}}\left[ \left( \mu_{a}-\Ddot{\mu}_{t,a}^{-}\right)^2 | \mathcal{H}_{t},\overline{\mathcal{E}}_t \right]  }\right)\right] \\
		& \leq \mathbb{E}_{\mathcal{H}_{t}}\left[\sqrt{L} \mathbb{P}\left(\overline{\mathcal{E}}_t\right) \left(\sqrt{\sum_{a=1}^K  \mathbb{E}_{\bm{\mu}}\left[ \left( \Ddot{\mu}_{t,a}^{+}-\mu_a\right)^2 | \mathcal{H}_{t},\overline{\mathcal{E}}_t \right]  }+\sqrt{\sum_{a=1}^K  \mathbb{E}_{\bm{\mu}}\left[ \left( \mu_{a}-\Ddot{\mu}_{t,a}^{-}\right)^2 | \mathcal{H}_{t},\overline{\mathcal{E}}_t \right]  }\right) \right]\\
		& \stackrel{(a)}{\leq} \mathbb{E}_{\mathcal{H}_{t}}\left[
  \sqrt{2L} \mathbb{P}\left(\overline{\mathcal{E}}_t\right) \left(
  \sqrt{\sum_{a=1}^K\left(\Ddot{\mu}_{t,a}^{+}- \Ddot{\mu}_{t,a}\right)^2  +K }
  +\sqrt{\sum_{a=1}^K\left( \Ddot{\mu}_{t,a}-\Ddot{\mu}_{t,a}^{-}\right)^2  +K }
  \right) \right] \\
		& \leq \mathbb{E}_{\mathcal{H}_{t}}\left[\sqrt{2L}  \mathbb{P}\left(\overline{\mathcal{E}}_t\right) \left(\sqrt{\sum_{a=1}^K\left(\Ddot{\mu}_{t,a}^{+}- \Ddot{\mu}_{t,a}\right)^2} +\sqrt{\sum_{a=1}^K\left( \Ddot{\mu}_{t,a}-\Ddot{\mu}_{t,a}^{-}\right)^2}
		\right)+ 2\sqrt{2LK}\mathbb{P}\left(\overline{\mathcal{E}}_t\right)  \right]\\&\leq 2\sqrt{2LK}\mathbb{P}\left(\overline{\mathcal{E}}_t\right)  ,
	\end{aligned}
\end{equation}
where $(a)$ is from~\eqref{eq:instanstaneous-fairness-regret-sub-term-op} and~\eqref{eq:instanstaneous-fairness-regret-sub-term2}. 

Summing over all time slots, we have
\begin{equation}
	\begin{aligned}
		&\mathbb{E}\left[\mathrm{FR}_T\right]=\sum_{t=1}^{T} \mathbb{E}[fr_t]\\
  &\leq \frac{ML}{\lambda} \left\lbrace 
\mathbb{E}_{\mathcal{H}_{t}}\left[\mathbb{E}_{\bm{\mu},\tilde{\bm{\mu}}_t^{+},\tilde{\bm{\mu}}_t^{-}}\left[\mathbb{E}_{a\sim\bm{p}_t}\left[
\left| \tilde{\mu}^{+}_{t,a}-\Ddot{\mu}_{t,a}^{+}\right| 
\right]| \mathcal{H}_{t}\right]\right]
 + \mathbb{E}_{\mathcal{H}_{t}}\left[\mathbb{E}_{\bm{\mu},\tilde{\bm{\mu}}_t^{+},\tilde{\bm{\mu}}_t^{-}}\left[\mathbb{E}_{a\sim\bm{p}_t}\left[
\left|\Ddot{\mu}_{t,a}^{+}-\mu_a\right| 
\right]| \mathcal{H}_{t}\right]\right]
\right.
\\&\quad\left.+\mathbb{E}_{\mathcal{H}_{t}}\left[\mathbb{E}_{\bm{\mu},\tilde{\bm{\mu}}_t^{+},\tilde{\bm{\mu}}_t^{-}}\left[\mathbb{E}_{a\sim\bm{p}_t}\left[\left| 
\mu_{a}-\Ddot{\mu}_{t,a}^{-}
\right| \right]| \mathcal{H}_{t}\right]\right]
 + \mathbb{E}_{\mathcal{H}_{t}}\left[\mathbb{E}_{\bm{\mu},\tilde{\bm{\mu}}_t^{+},\tilde{\bm{\mu}}_t^{-}}\left[\mathbb{E}_{a\sim\bm{p}_t}\left[
 \left| 
\Ddot{\mu}_{t,a}^{-}-\tilde{\mu}^{-}_{t,a}
 \right| \right]| \mathcal{H}_{t}\right]\right]  \right\rbrace \\
    &\stackrel{(a)}{\leq}
    \frac{2ML}{\lambda}\sum_{t=1}^{T}
  \mathbb{E}_{\mathcal{H}_{t}}\left[ \mathbb{E}\left[\mathbb{E}_{a\sim\bm{p}_t}\left[\sqrt{\frac{1}{N_{t,a}\vee 1}}\right]\Big| \mathcal{H}_{t}\right]
  \right]+\frac{2ML}{\lambda}\sum_{t=1}^{T}
  \mathbb{E}_{\mathcal{H}_{t}}\left[ \mathbb{E}\left[\mathbb{E}_{a\sim\bm{p}_t}\left[\sqrt{\frac{1}{M_{t,a}+ 1}}\right]\Big| \mathcal{H}_{t}\right]
  \right]\\
  &\quad+\frac{2ML\sqrt{2LK}}{\lambda}\sum_{t=1}^{T}
  \mathbb{P}\left(\overline{\mathcal{E}}_t\right)  +\frac{2ML}{\lambda}\sum_{t=1}^{T}\mathbb{E}_{\mathcal{H}_{t}}\left[\mathbb{E}\left[\mathbb{E}_{a\sim\bm{p}_t}\left[ \frac{N_{t,a}-N_{t-d_a(q),a}}{N_{t,a}\vee 1}+1-q+\sqrt{\frac{\log(2KT /\delta^{\prime})}{2(N_{t,a}\vee 1)}}\right]\Big|\mathcal{H}_t\right]\right]\\
&\leq \frac{2ML^2}{\lambda}(1-q)T + \frac{2ML}{\lambda}\left(\sqrt{\log(2KT /\delta^{\prime})/2}+d^{*}(q)+1\right)\sum_{t=1}^{T}\mathbb{E}_{\mathcal{H}_{t}}\left[\mathbb{E}\left[\mathbb{E}_{a\sim\bm{p}_t}
        \left[ \frac{1}{N_{t,a}\vee 1} \right]\Big|\mathcal{H}_t \right]\right] \\&\quad+\frac{2ML}{\lambda}\sum_{t=1}^{T}
  \mathbb{E}_{\mathcal{H}_{t}}\left[ \mathbb{E}\left[\mathbb{E}_{a\sim\bm{p}_t}\left[\sqrt{\frac{1}{M_{t,a}+ 1}}\right]\Big| \mathcal{H}_{t}\right] \right]  +\frac{2ML\sqrt{2LK}}{\lambda}\sum_{t=1}^{T}\mathbb{P}\left(\overline{\mathcal{E}}_t\right)
.
\end{aligned}
\end{equation}
where $(a)$ is derived from  ~\eqref{term_c},~\eqref{term_f},~\eqref{term_de} and~\eqref{term_g}.
For the term $\sum_{t=1}^T\mathbb{E}_{a\sim\bm{p}_t}\left[\sqrt{\frac{1}{N_{t,a}\vee 1}}\right] $, with probability at least
$1-\delta$, we have, 
\begin{equation}
\begin{aligned}   \sum_{t=1}^T\mathbb{E}_{a\sim\bm{p}_t}\left[\sqrt{\frac{1}{N_{t,a}\vee 1}}\right]&\stackrel{(a)}{\leq} L\sqrt{2T \log \frac{2}{\delta}} +\sum_{t=1}^{T}\sum_{a\in A_t}\sqrt{ \frac{1}{N_{t,a}\vee 1} }\leq L\sqrt{2T \log \frac{2}{\delta}} +K\int_{1}^{\frac{LT}{K}}\sqrt{\frac{1}{x}}dx+1\\
    &\leq L\sqrt{2T \log \frac{2}{\delta}} +2\sqrt{LKT}+1,
\end{aligned}
\end{equation}
where $(a)$ is from Lemma~\ref{Lemma:martingale_seq_all_observations}.
Set $\delta=1/LT$, we further have
\begin{equation}
\label{ineq:upper-bound-all-1/N}
\begin{aligned}    \mathbb{E}\left[ \sum_{t=1}^T\mathbb{E}_{a\sim\bm{p}_t}\left[\sqrt{\frac{1}{N_{t,a}\vee 1}}\right]\right] 
    &\leq L\sqrt{2T \log \frac{2}{\delta}} +2\sqrt{LKT}+LT\delta+1\\
    &\leq L\sqrt{2T \log (2LT)} +2\sqrt{LKT}+2.
    \end{aligned}
\end{equation}
Moreover, by the result of~\eqref{BayesFR-first-term},
\begin{equation}
	\label{BayesFR-first-term-op}
	\begin{aligned}
		\mathbb{E}\left[ \sum_{t=1}^{T}\mathbb{E}_{a\sim\bm{p}_t}\left[ \sqrt{\frac{1}{M_{t,a}+1}} \right] \right]  
		\leq L+ \frac{24K\log T}{q}  + L\sqrt{2T \log (4LT)}+2\sqrt{\frac{2LKT}{q}} + L d^{*}(q)+ 1.
	\end{aligned}
\end{equation}
Finally, we have
\begin{equation}
	\begin{aligned}
		&\mathbb{E}\left[\mathrm{FR}_T\right]=\sum_{t=1}^{T} \mathbb{E}[fr_t]\\
&\leq \frac{2ML}{\lambda} \Bigg [L(1-q)T+\left(\sqrt{\log(2KT /\delta^{\prime})/2}+d^{*}(q)+1\right)\left( L\sqrt{2T \log (2LT)} +2\sqrt{LKT}+2\right) 
\\& \quad +L+ \frac{24K\log T}{q}  + L\sqrt{2T \log (4LT)}+2\sqrt{\frac{2LKT}{q}} + L d^{*}(q)+ 1 +\sqrt{2LK}\sum_{t=1}^{T}\mathbb{P}\left(\overline{\mathcal{E}}_t\right) \Bigg ]
\\
&\leq 
\frac{2ML}{\lambda} \Bigg [L(1-q)T+\left(\sqrt{\log(2KT /\delta^{\prime})/2}+d^{*}(q)+1\right)\left( L\sqrt{2T \log (2LT)} +2\sqrt{LKT}+2\right) 
\\& \quad +L+ \frac{24K\log T}{q}  + L\sqrt{2T \log (4LT)}+2\sqrt{\frac{2LKT}{q}} + L d^{*}(q)+ 1+\sqrt{2LK}T\delta^{\prime} \Bigg ]
.
\end{aligned}
\end{equation}
which hold for any quantile $q\in\leftopen{0}{1}$.
Set $\delta^{\prime}=1/T$,  the expected fairness regret can be upper bounded as follows:
\begin{equation}
	\begin{aligned}
		&\mathbb{E}\left[\mathrm{FR}_T\right]=\sum_{t=1}^{T} \mathbb{E}[fr_t]\\
		&\leq \min_{q\in (0,1]} \left\{
\frac{2ML}{\lambda} \Bigg [ L(1-q)T+\left(\sqrt{\log(2KT )}+d^{*}(q)+2\right)\left( L\sqrt{2T \log (2LT)} +2\sqrt{LKT}+1\right) \right. \\& \left.+L+ \frac{24K\log T}{q}  + L\sqrt{2T \log (4LT)}+2\sqrt{\frac{2LKT}{q}} + L d^{*}(q)+ 1 +\sqrt{2LK}\Bigg ]
\right\} \\
 &=\widetilde{O}\left( \min_{q\in (0,1]}\left\lbrace \frac{MLK}{\lambda}\left( (1-q)T+\frac{d^{*}(q)}{q}\sqrt{T}\right) \right\rbrace \right).
	\end{aligned}
\end{equation}
This completes the proof of the expected fairness regret upper bound of \ctsfdop.

\proofpart{2}{Proof of the Expected Reward Regret Upper Bound of \ctsfdop}

In \ctsfdop,
the learner selects arms with the probabilistic vector $\bm{p}_t$, depending on the average of
$\tilde{\mu}^{+}_{t,a}$ and $\tilde{\mu}^{-}_{t,a}$ from each arm $a$ while the
optimal fairness policy $\bm{p}^*$ depends on $\bm{\mu}$.
Because we lack information about unobserved feedback due to reward-dependent delays, 
the optimal fairness policy $\bm{p}^*$ and the policy
$\bm{p}_t$ are not identically distributed conditioned on $\mathcal{H}_t$. 
Hence, Lemma~\ref{lemma:BRT-decomposed} does not apply in this setting.
We derive the expected reward regret bound from the expected fairness regret
under Assumption~\ref{Minimum-Merit} and Assumption~\ref{Lipschitz-Continuity}.

Specifically, by the definition of the expected fairness regret and expected reward regret,
the expected reward regret can be upper bounded using the expected fairness regret as follows:
\begin{equation}
\begin{aligned}
        \mathbb{E}\left[\mathrm{RR}_T\right] &=  \mathbb{E}\left[\sum_{t=1}^{T}\max\left\lbrace \sum_{a=1}^Kp^{*}_a\mu_{a}-\sum_{a=1}^K p_{t,a}\mu_{a}, 0\right\rbrace\right] \stackrel{(a)}{\leq}  \mathbb{E} \left[\sum_{t=1}^{T}\sum_{a=1}^K\left|p^{*}_a-p_{t,a}\right|\right]
    =\mathbb{E}\left[\mathrm{FR}_T\right] \\&= \widetilde{O}\left( \min_{q\in (0,1]}\left\lbrace \frac{MLK}{\lambda}\left( (1-q)T+\frac{d^{*}(q)}{q}\sqrt{T})\right) \right\rbrace \right),
\end{aligned}
\end{equation}
where $(a)$ is because $\mu_a \in \left[0,1\right]$.
This completes the proof of the expected reward regret upper bound of \ctsfdop.

Combining Part 1 and Part 2 of the proof, we complete the proof of Theorem~\ref{the:fairness-reward-regret-op-thompson-sam-type}.
\end{proof}

\section{Additional Experiments}
\label{additional_experiments}
\subsection{Experiments using different merit functions}

We vary the parameter $c$ of the merit function and present the experiment
results on the fairness and reward regret of the different algorithms under various types
of feedback delays in
Figure~\ref{fig:experiments_on_fixed_delays_slope}-\ref{fig:experiments_on_baised_delay_slope}.
The plots indicate that our algorithms 
effectively ensure merit-based fairness among arms while achieving low reward regret for different merit functions. 
As the parameter $c$ increases, making the merit function $f$ steeper, the gap
between different bandit algorithms widens accordingly.
In particular, we note that FGreedy-D has a comparable performance with \cucbfd and \ctsfd under the reward-dependent delay setting in Figure~\ref{fig:experiments_on_baised_delay_slope}.

The reason is that FGreedy-D selects arms uniformly
at random in the exploration phase when facing observed feedback with potential bias, and thus it avoids favoring the sub-optimal arms at the beginning like \cucbfd and \ctsfd.

Finally, we provide the running time of \cucbfd and \ctsfd (and their corresponding OP versions) to demonstrate the effectiveness of our TS-type algorithms. Figure~\ref{fig:running-time} shows that our TS-type algorithms (\ctsfd and \ctsfdop) have shorter running times for the same number of rounds by avoiding solving the optimization problem.

\begin{figure*}[!t]
	\centering
	\subfigure[$c=2$]{
	\centering
	\includegraphics[width=0.23\textwidth]{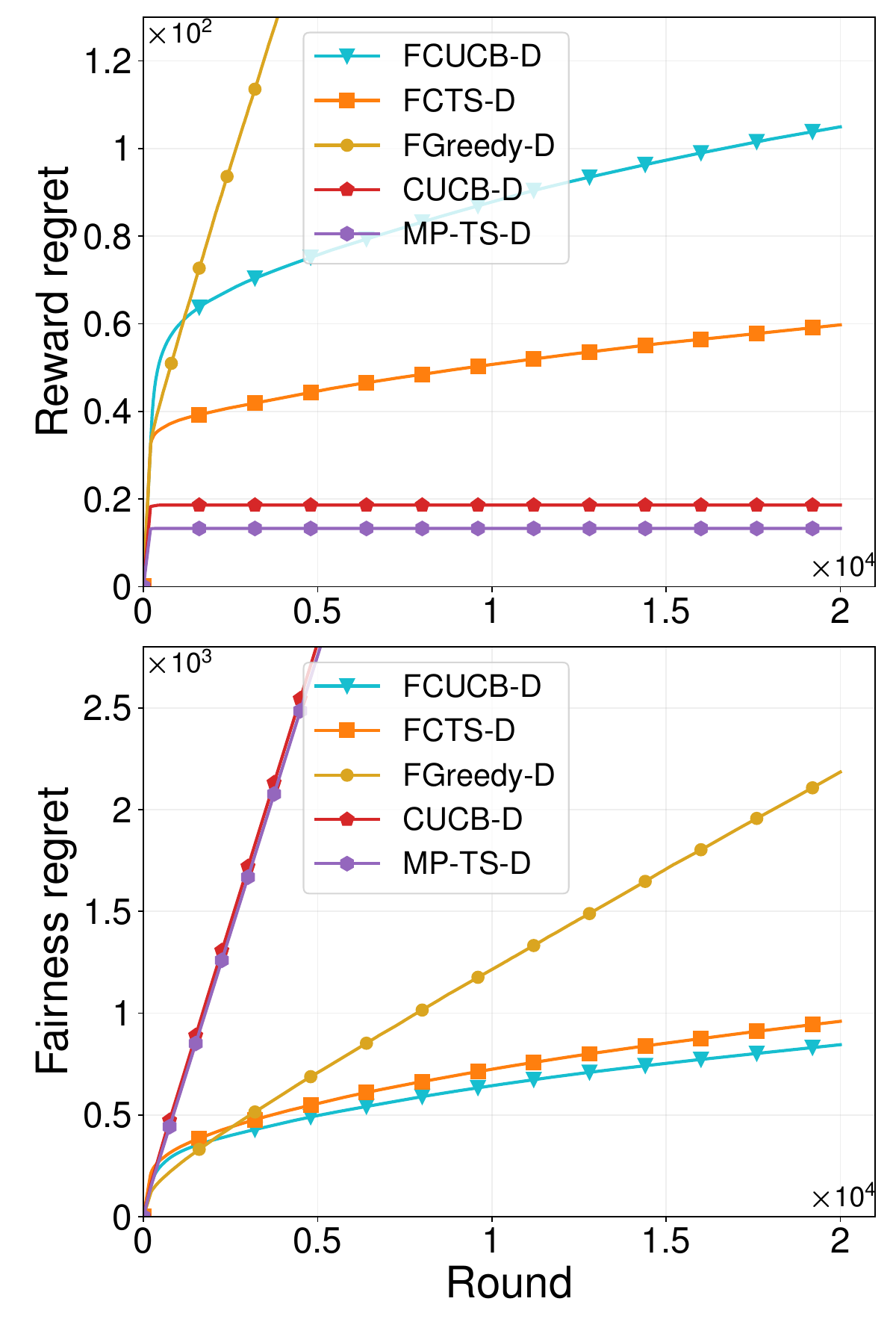}
	\label{fig:fixed_delays_slope_c=2}
	}
	\subfigure[$c=4$]{
		\centering
		\includegraphics[width=0.23\textwidth]{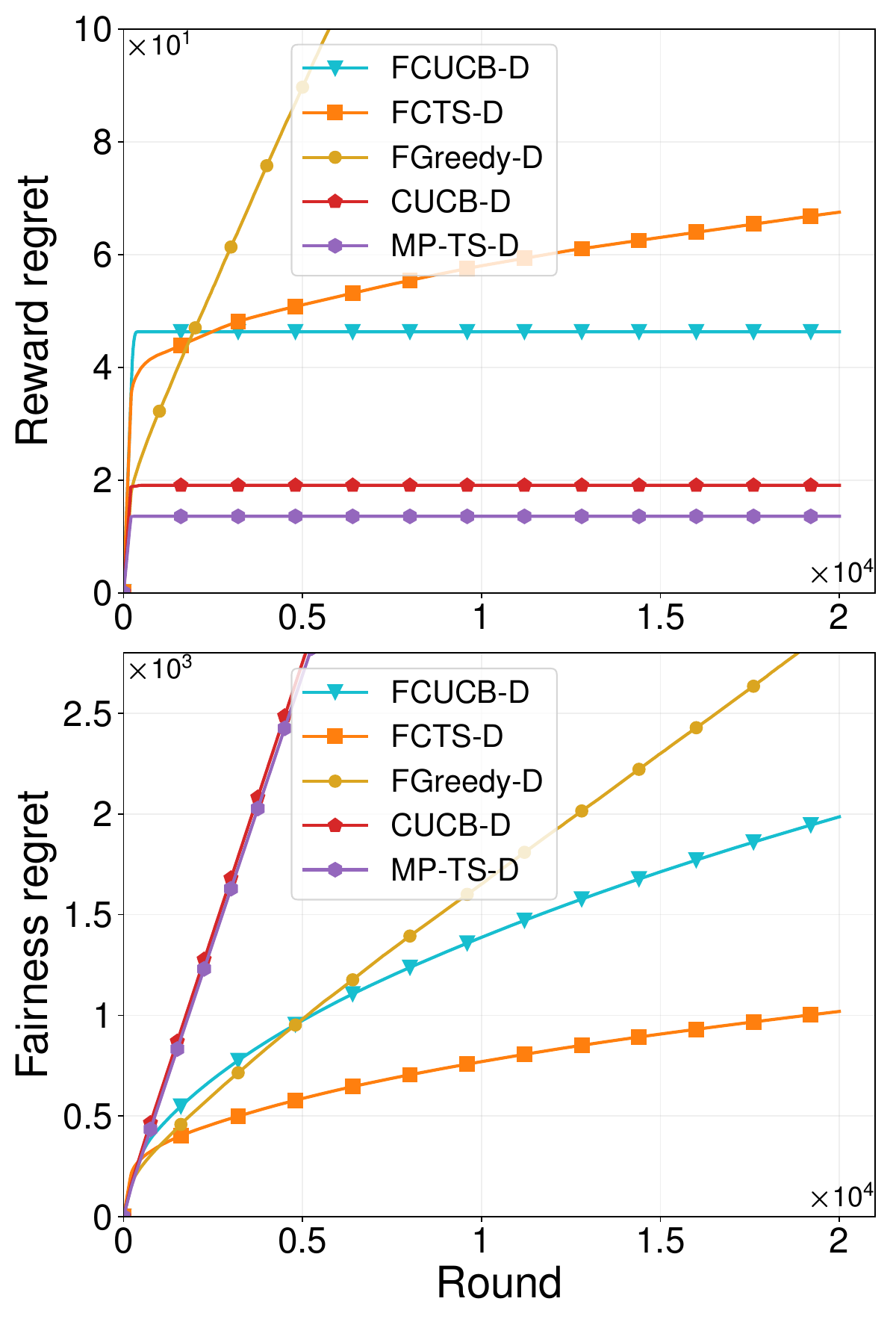}
		\label{fig:fixed_delays_slope_c=4}              
	}
	\subfigure[$c=6$]{
		\centering
		\includegraphics[width=0.23\textwidth]{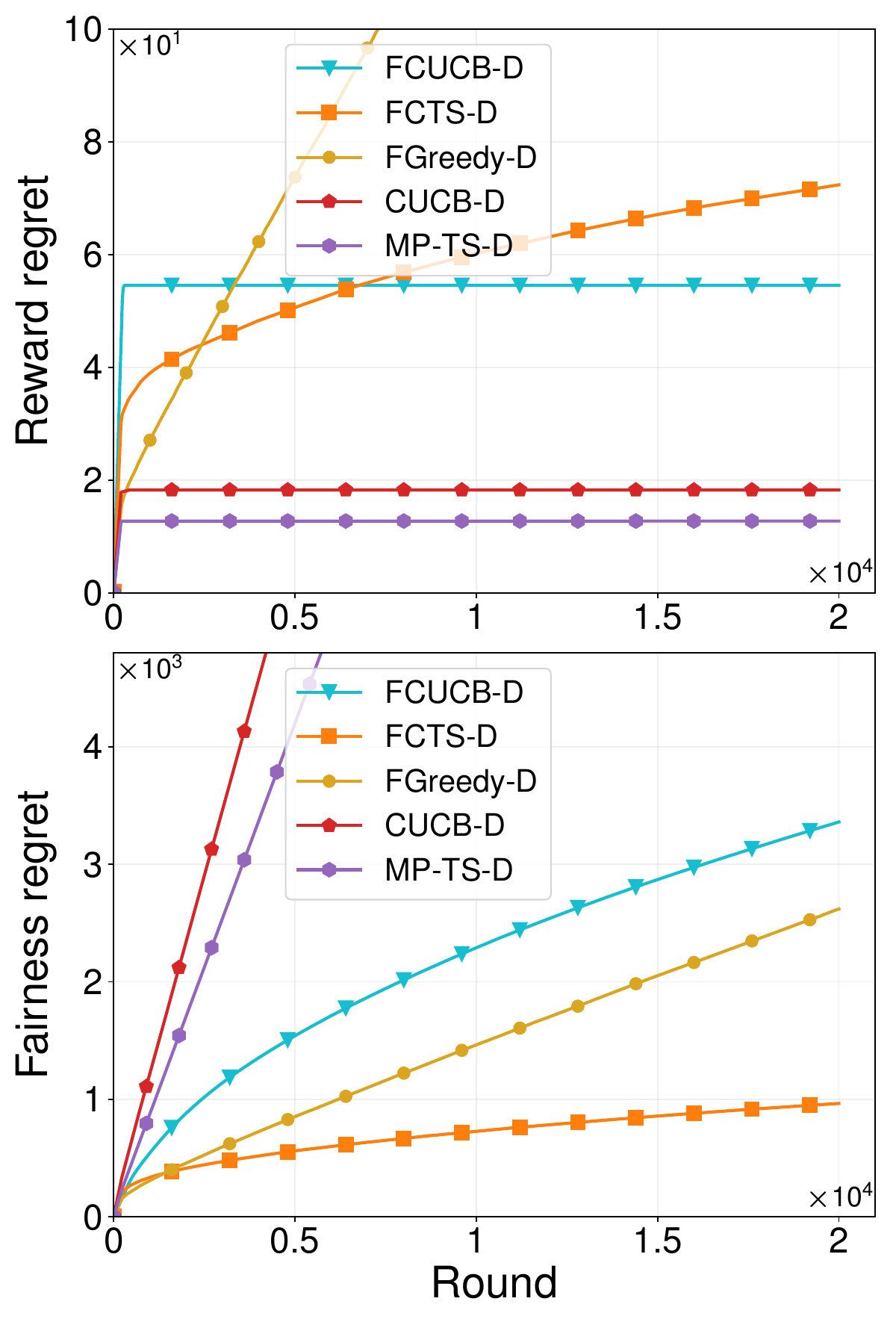}
		\label{fig:fixed_delays_slope_c=6}
	}
 	\subfigure[$c=8$]{
		\centering
		\includegraphics[width=0.23\textwidth]{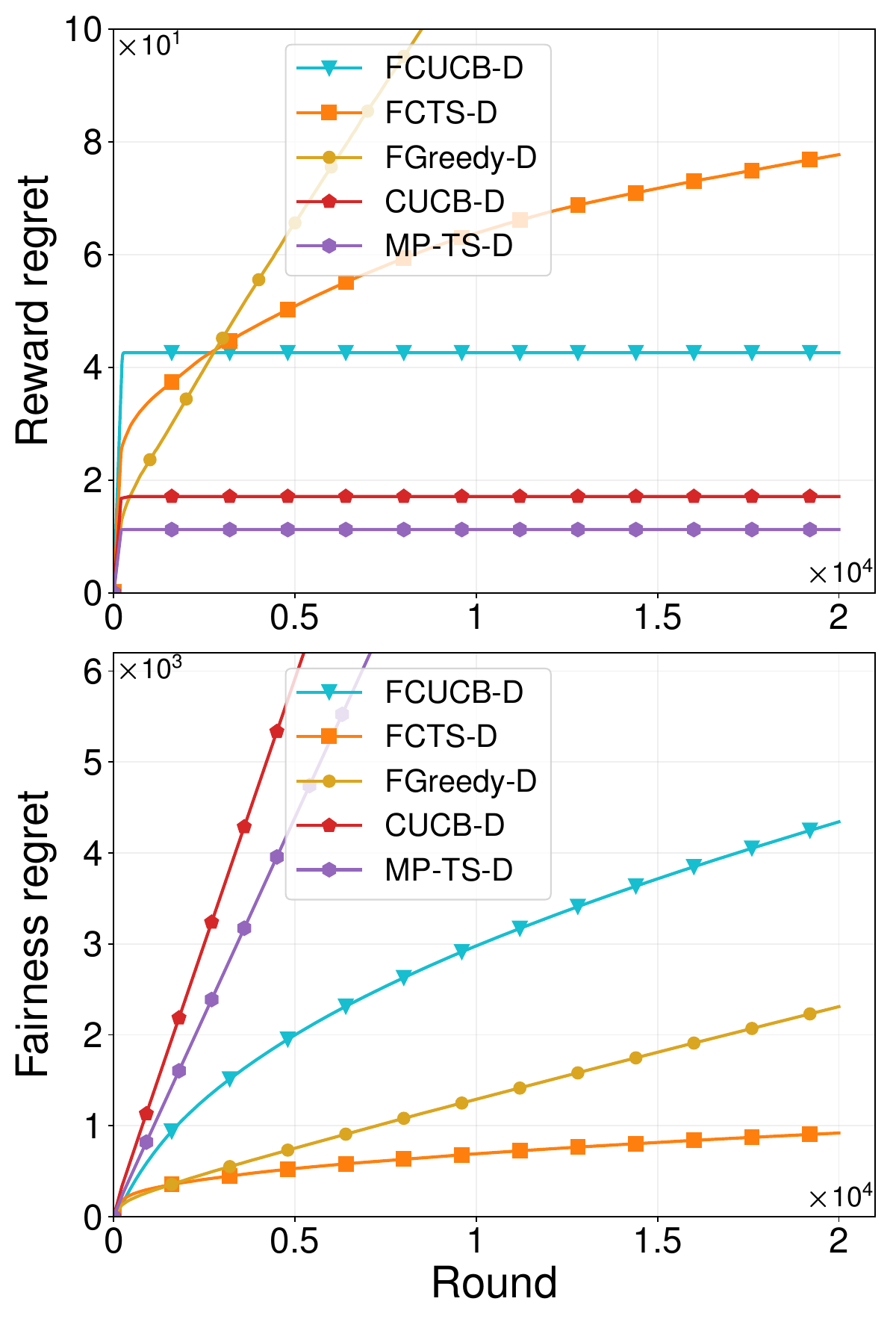}
		\label{fig:fixed_delays_slope_c=8}              
	}
	\caption{Experiment results of the different bandit algorithms using different merit
    functions algorithms under fixed feedback delays (200 rounds).}
	\label{fig:experiments_on_fixed_delays_slope}
\end{figure*}
\begin{figure*}[!t]
	\centering
	\subfigure[$c=2$]{
	\centering
	\includegraphics[width=0.23\textwidth]{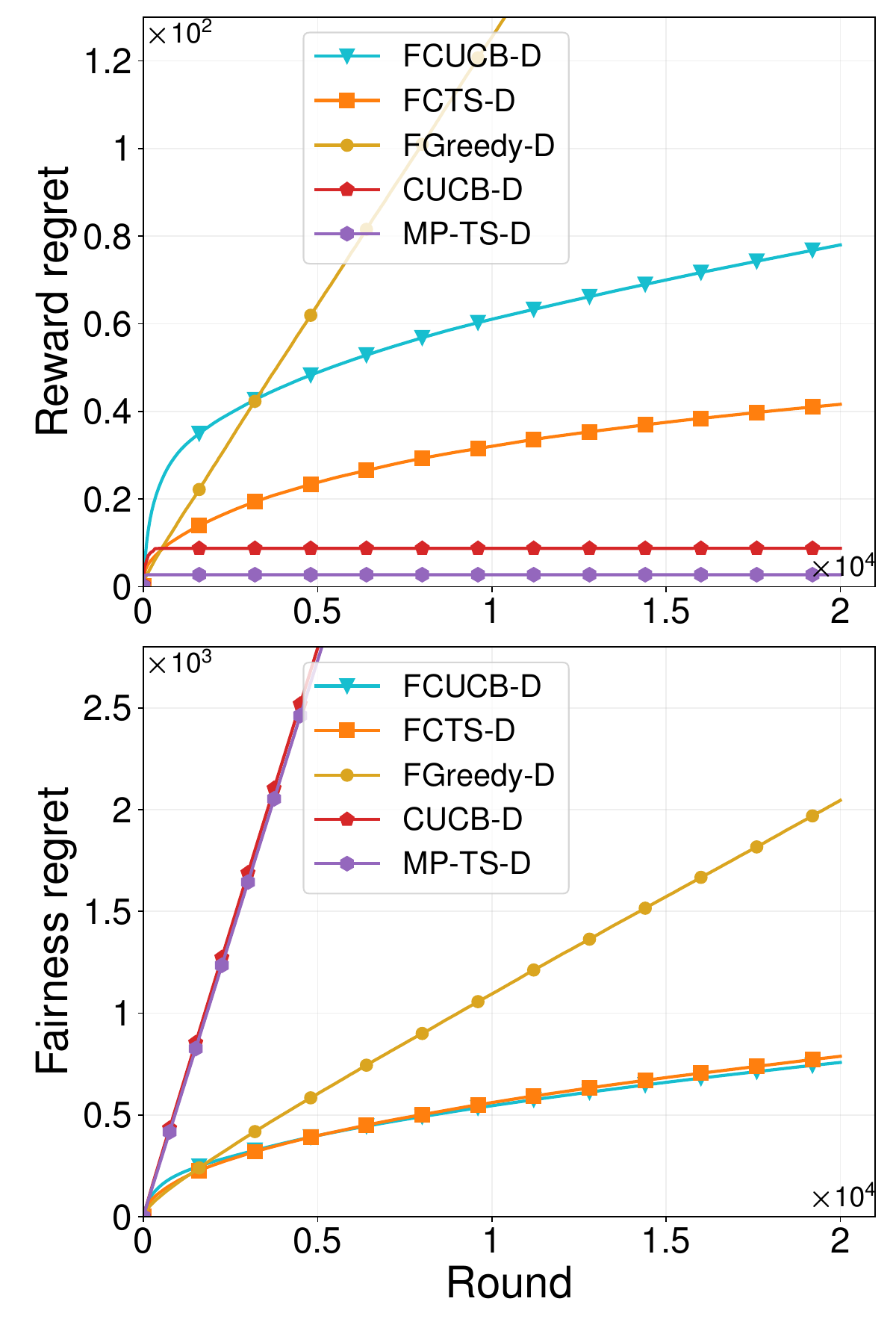}
	\label{fig:geometeric_delays_slope_c=2}
	}
	\subfigure[$c=4$]{
		\centering
		\includegraphics[width=0.23\textwidth]{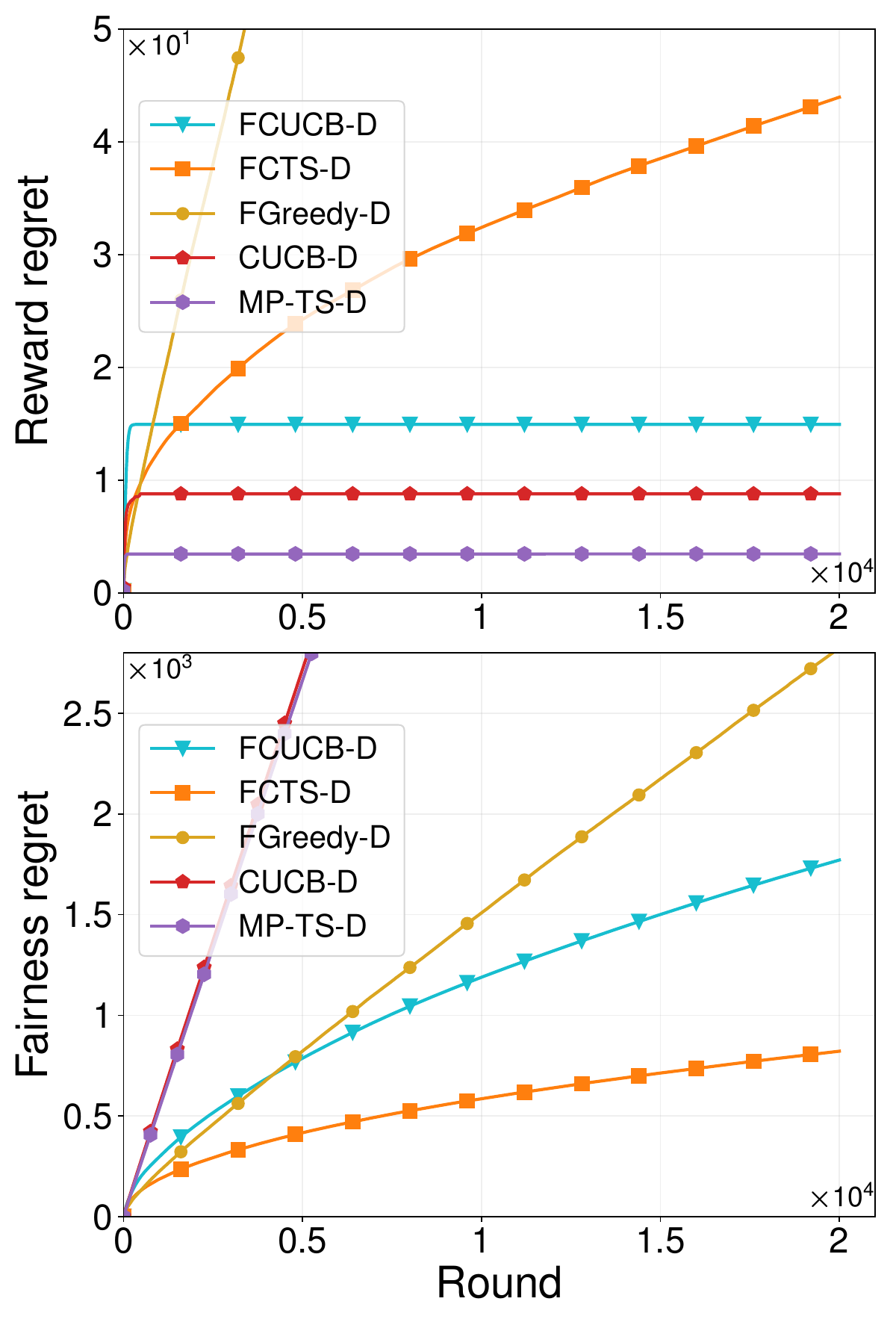}
		\label{fig:geometeric_delays_slope_c=4}              
	}
	\subfigure[$c=6$]{
		\centering
		\includegraphics[width=0.23\textwidth]{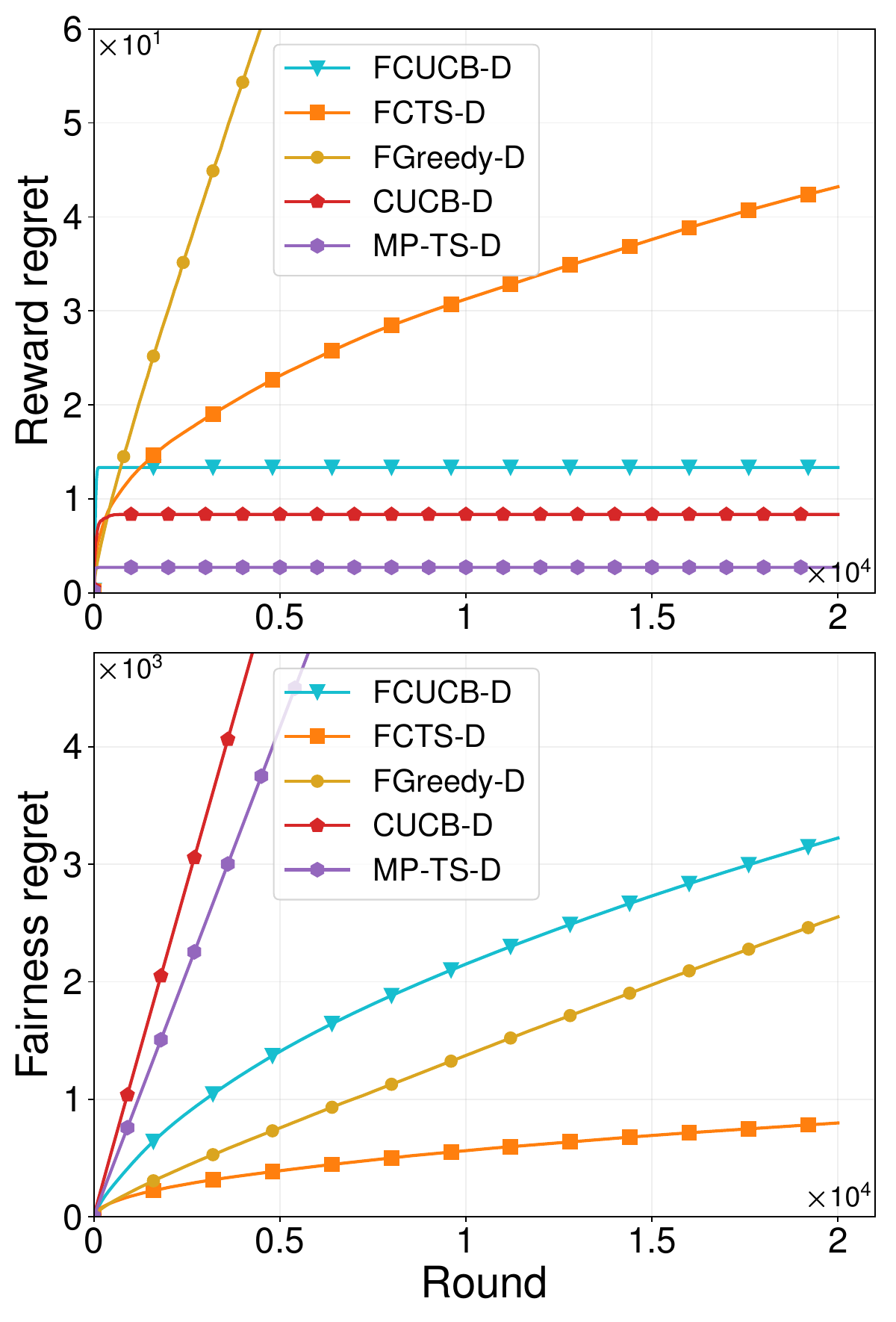}
		\label{fig:geometeric_delays_slope_c=6}
	}
 	\subfigure[$c=8$]{
		\centering
		\includegraphics[width=0.23\textwidth]{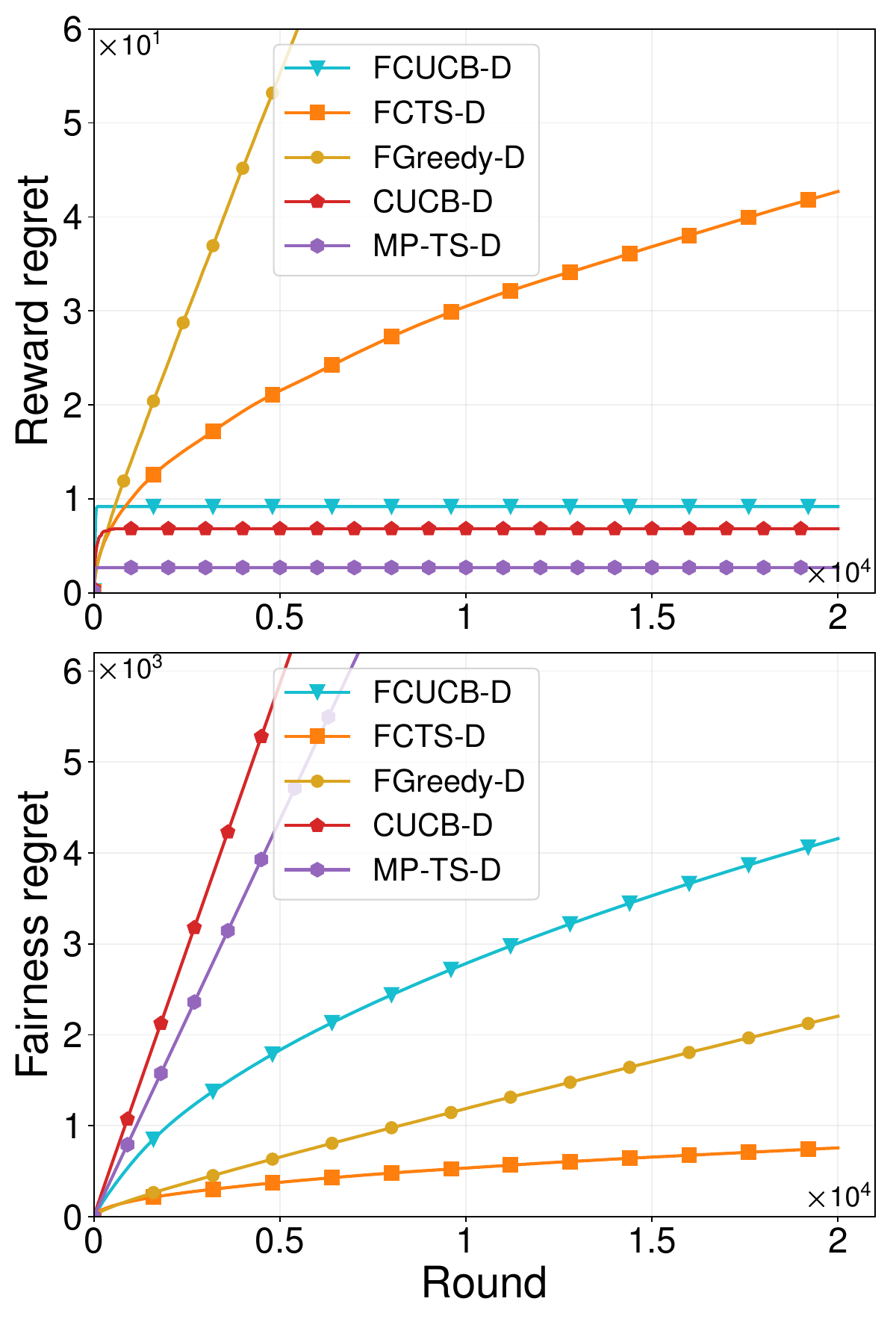}
		\label{fig:geometeric_delays_slope_c=8}              
	}
	\caption{Experiment results of the different bandit algorithms using different merit
    functions under geometric feedback delays.}
	\label{fig:experiments_on_geometeric_delays_slope}
\end{figure*}
\begin{figure*}[!t]
	\centering
	\subfigure[$c=2$]{
	\centering
	\includegraphics[width=0.23\textwidth]{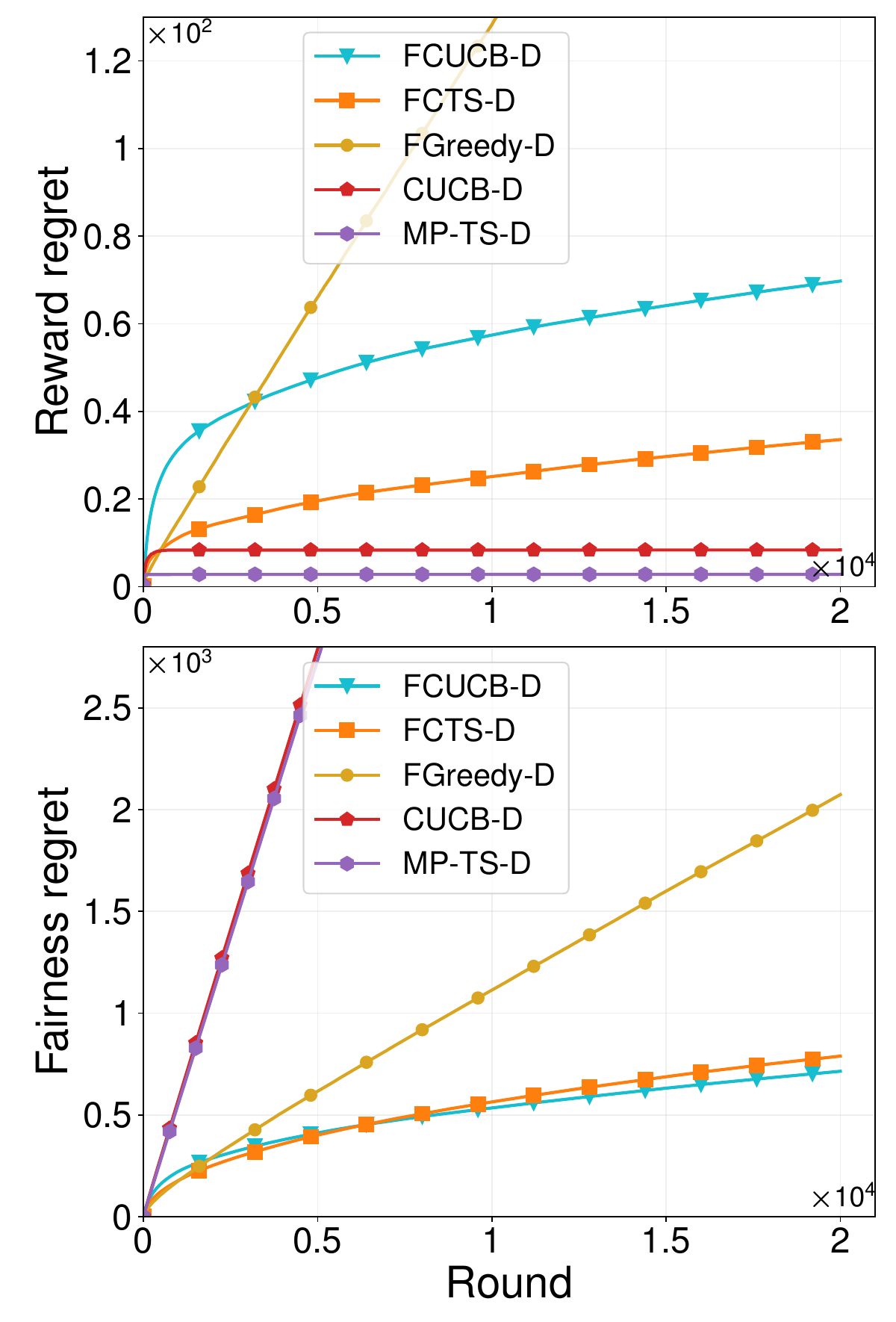}
	\label{fig:pareto_delays_slope_c=2}
	}
	\subfigure[$c=4$]{
		\centering
		\includegraphics[width=0.23\textwidth]{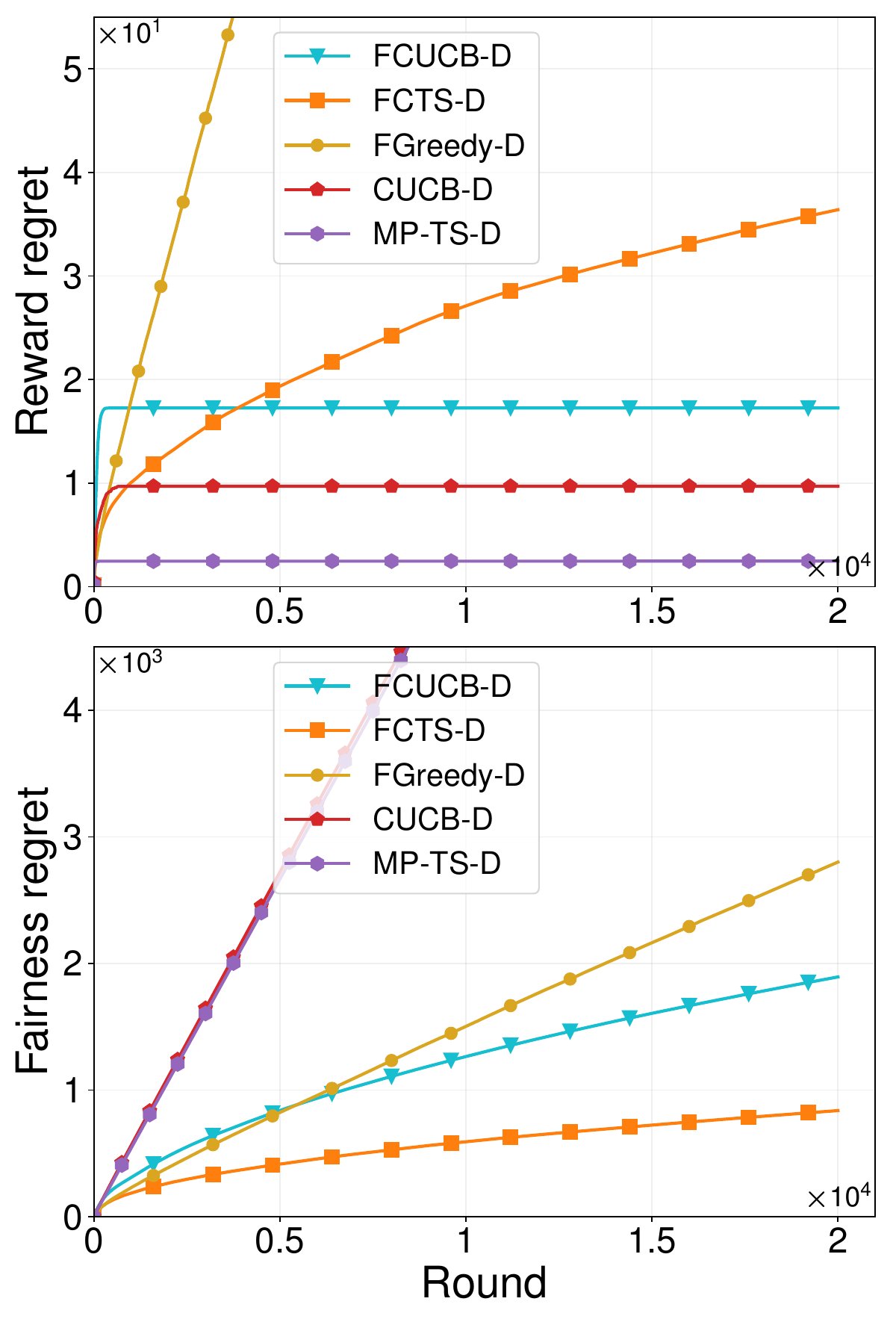}
		\label{fig:pareto_delays_slope_c=4}              
	}
	\subfigure[$c=6$]{
		\centering
		\includegraphics[width=0.23\textwidth]{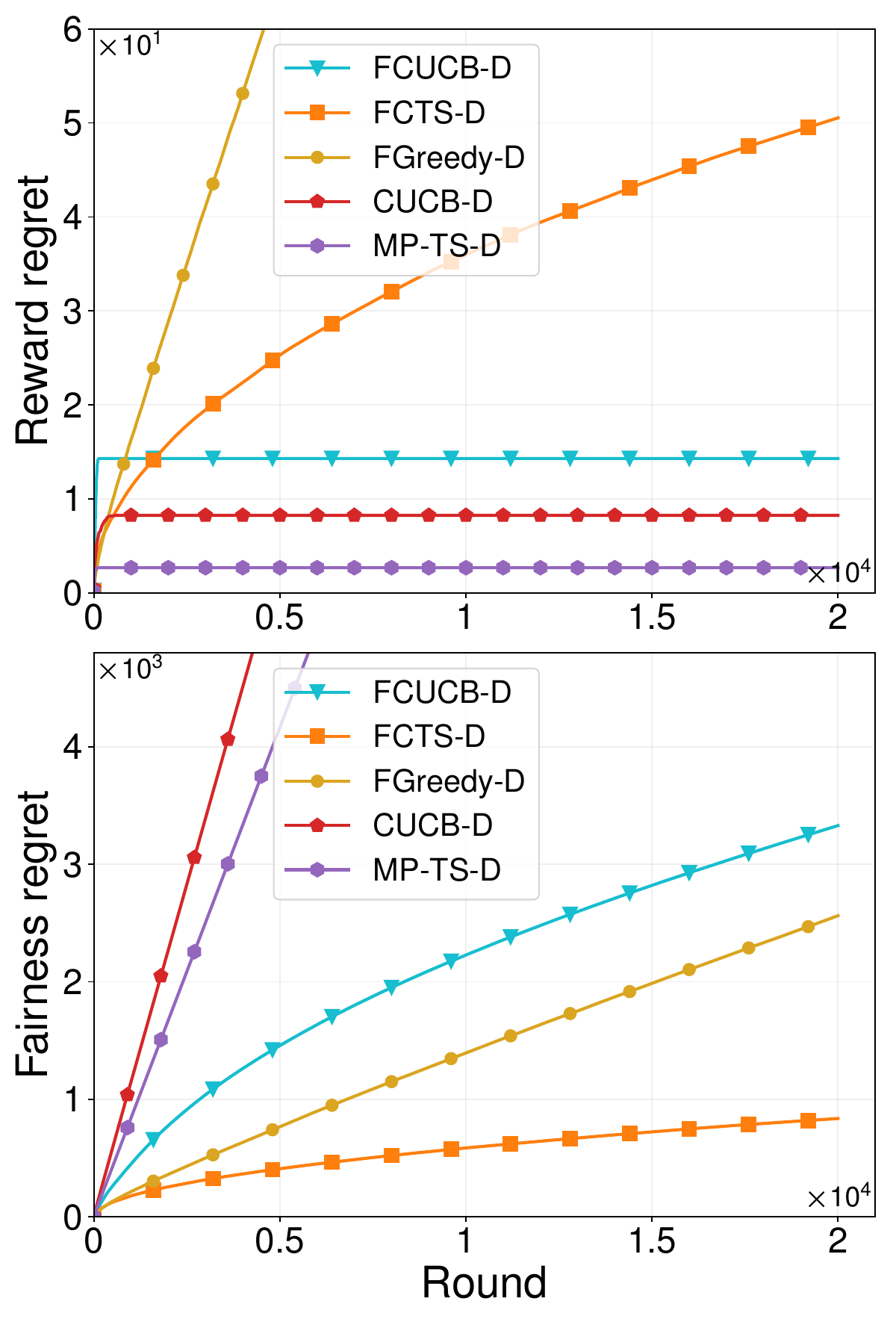}
		\label{fig:pareto_delays_slope_c=6}
	}
 	\subfigure[$c=8$]{
		\centering
		\includegraphics[width=0.23\textwidth]{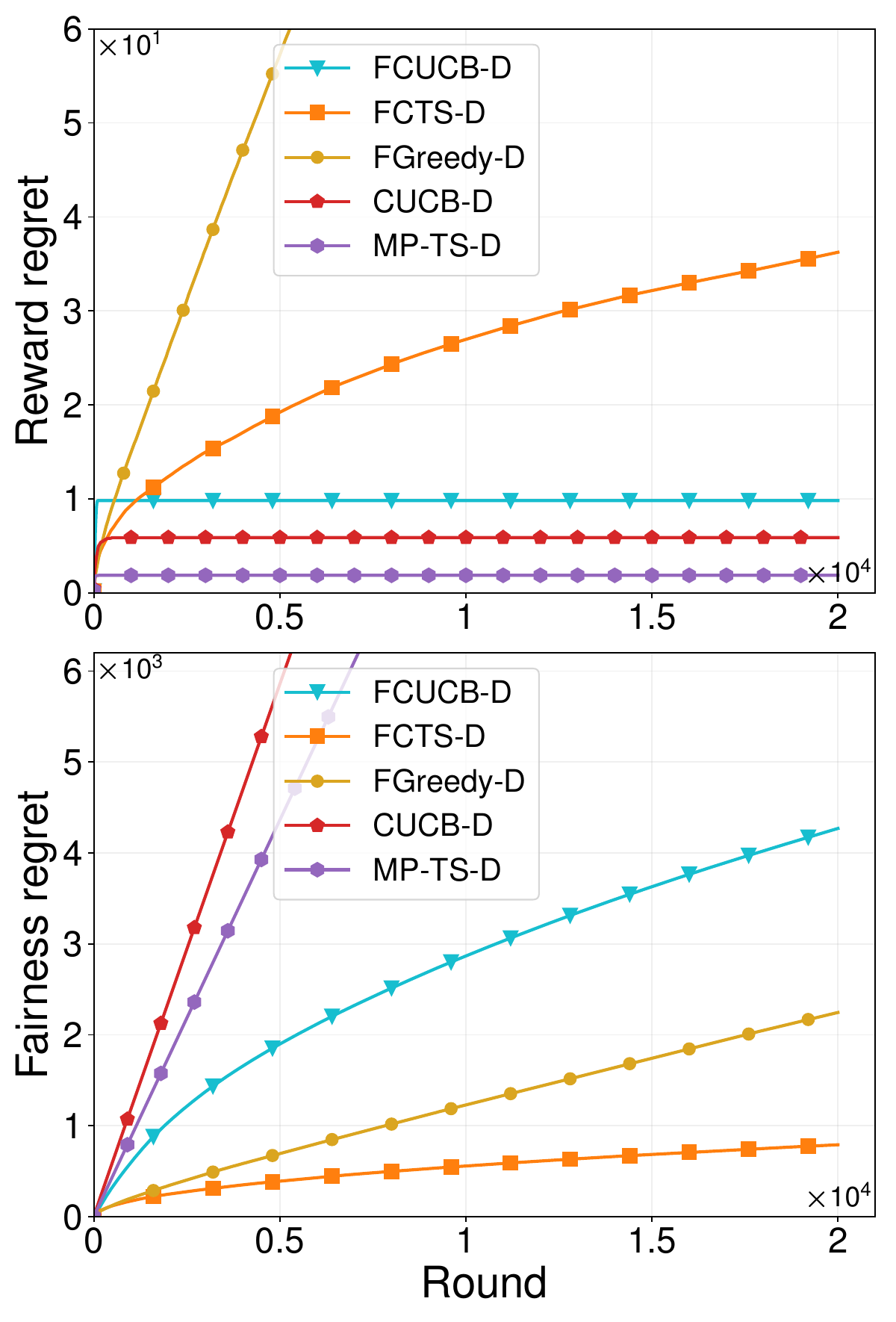}
		\label{fig:pareto_delays_slope_c=8}              
	}
	\caption{Experiment results of the different bandit algorithms using different merit
    functions under $\alpha$-Pareto feedback delays.}
	\label{fig:experiments_on_pareto_delays_slope}
\end{figure*}

\begin{figure*}[!t]
	\centering
	\subfigure[$c=2$]{
	\centering
	\includegraphics[width=0.23\textwidth]{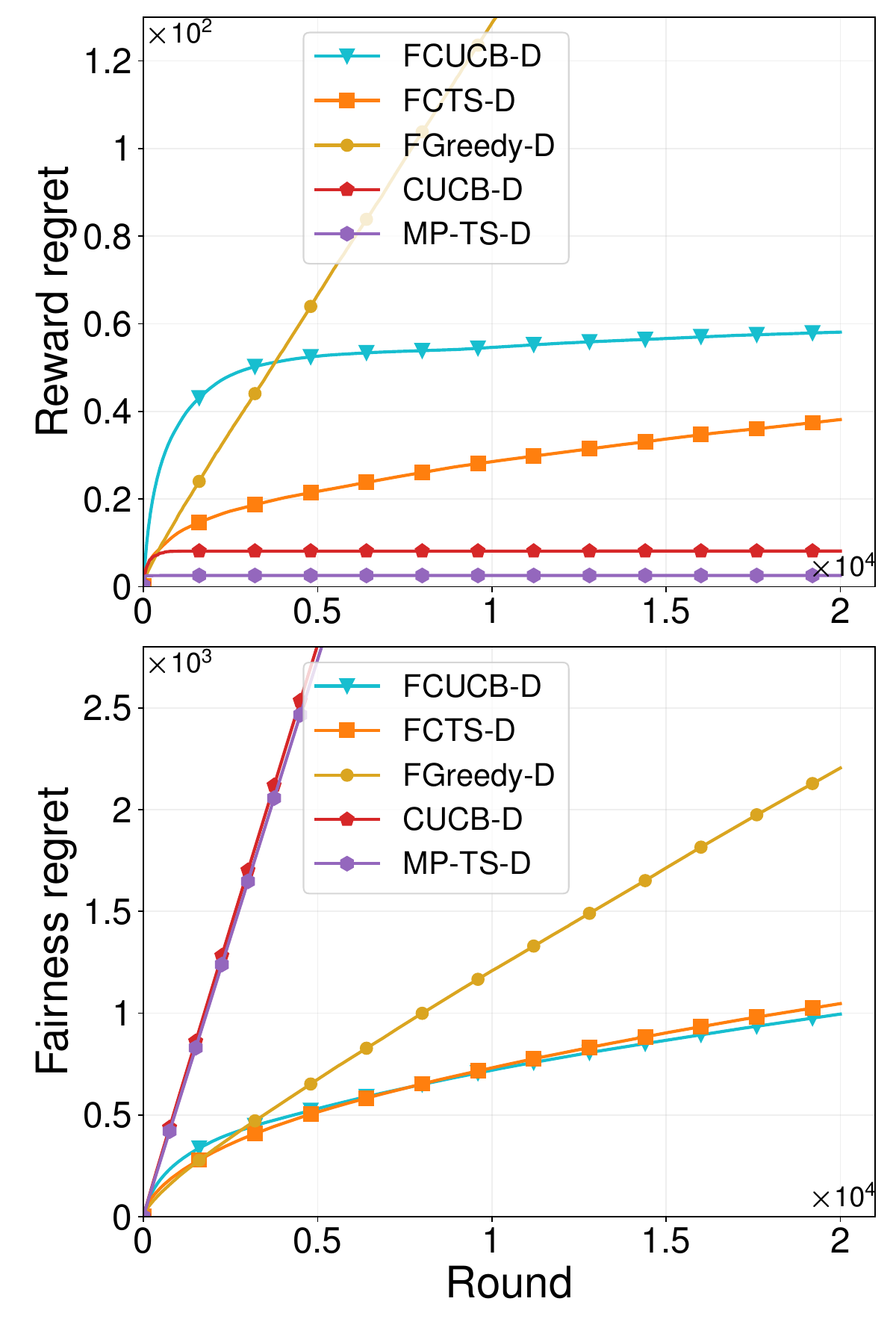}
 	\label{fig:packet_loss_delays_slope_c=2}
	}
	\subfigure[$c=4$]{
		\centering
		\includegraphics[width=0.23\textwidth]{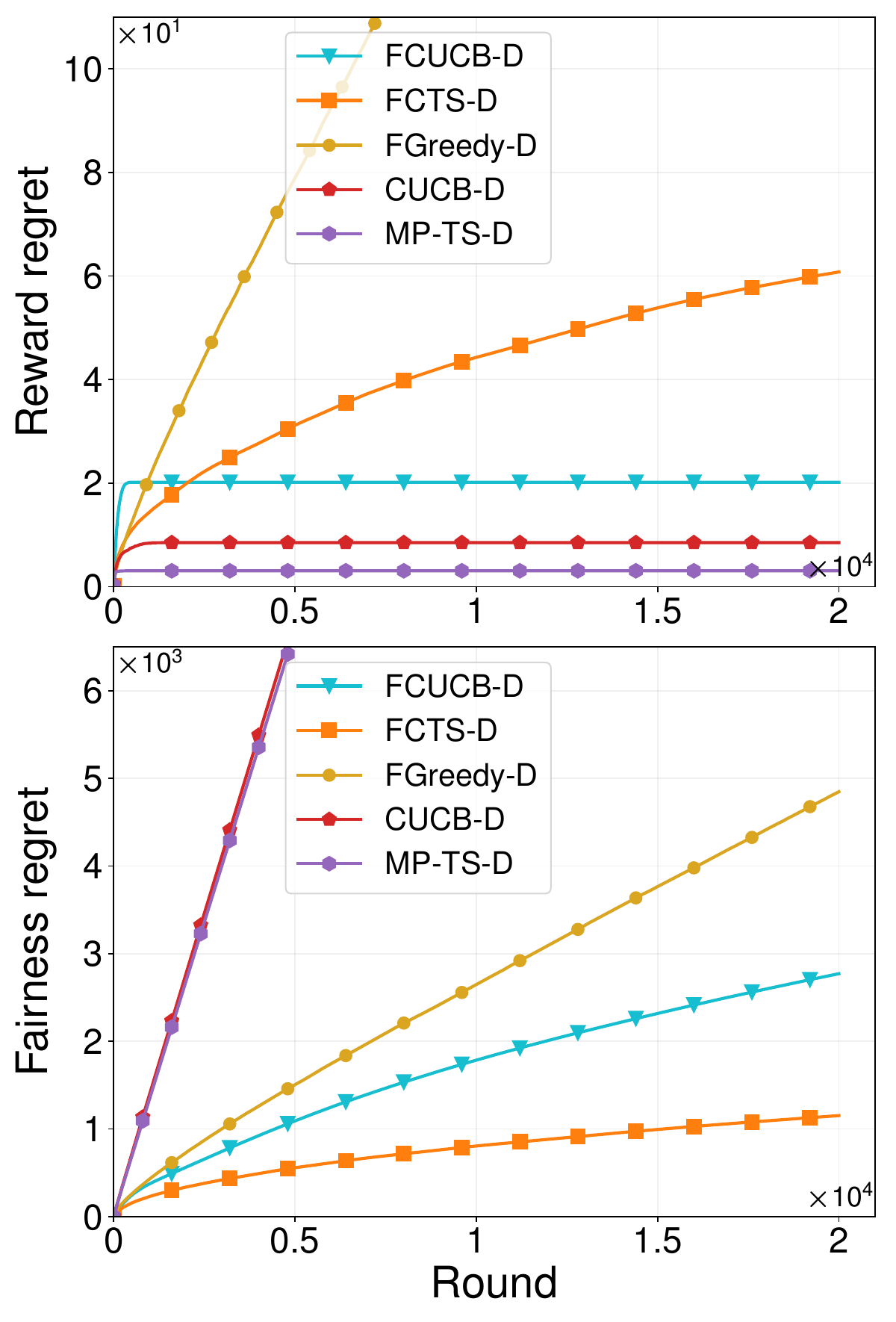}
		\label{fig:packet_loss_delays_slope_c=4}              
	}
	\subfigure[$c=6$]{
		\centering
		\includegraphics[width=0.23\textwidth]{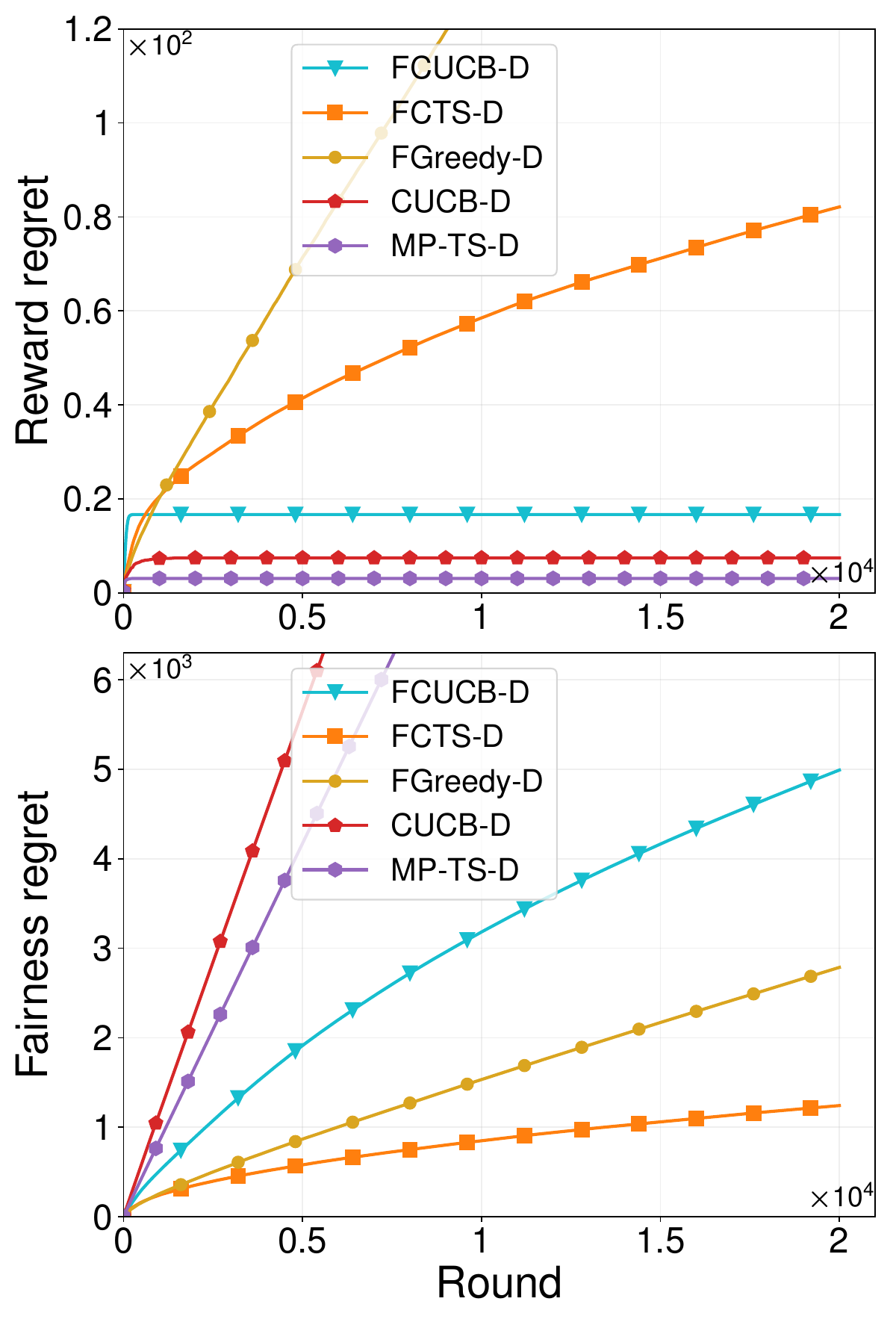}
		\label{fig:packet_loss_delays_slope_c=6}
	}
 	\subfigure[$c=8$]{
		\centering
		\includegraphics[width=0.23\textwidth]{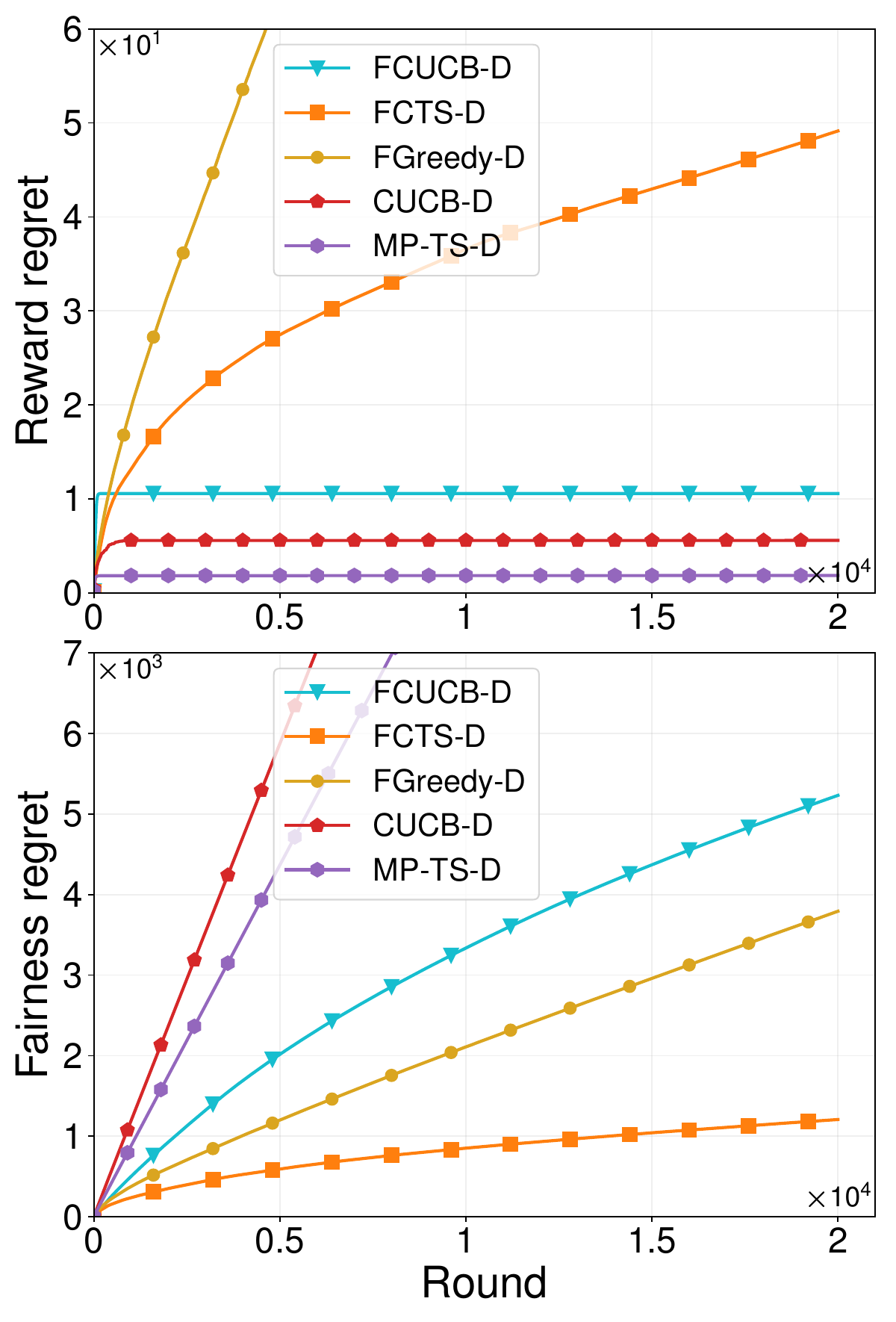}
		\label{fig:packet_loss_delays_slope_c=8}              
	}
  	\caption{Experiment results of the different bandit algorithms
    using different merit functions under packet-loss feedback delays.}
	\label{fig:experiments_on_packet_loss_delays_slope}
\end{figure*}
\begin{figure*}[!t]
	\centering
	\subfigure[$c=2$]{
	\centering
	\includegraphics[width=0.23\textwidth]{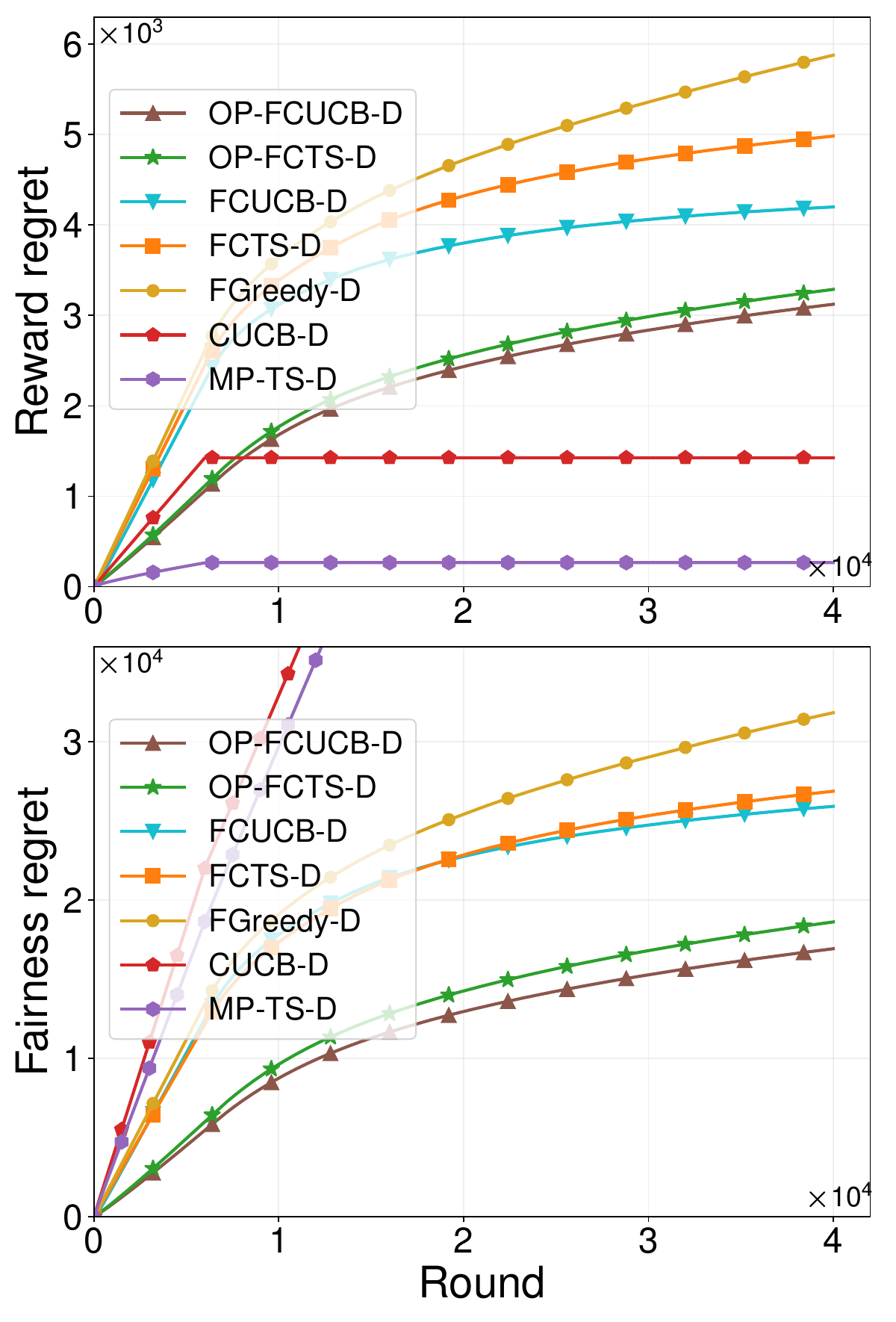}
	\label{fig:biased_delays_slope_c=2}
	}
	\subfigure[$c=4$]{
		\centering
		\includegraphics[width=0.23\textwidth]{Figures_arxiv/biased_delays_slope_c=4.pdf}
		\label{fig:biased_delays_slope_c=4}              
	}
	\subfigure[$c=6$]{
		\centering
		\includegraphics[width=0.23\textwidth]{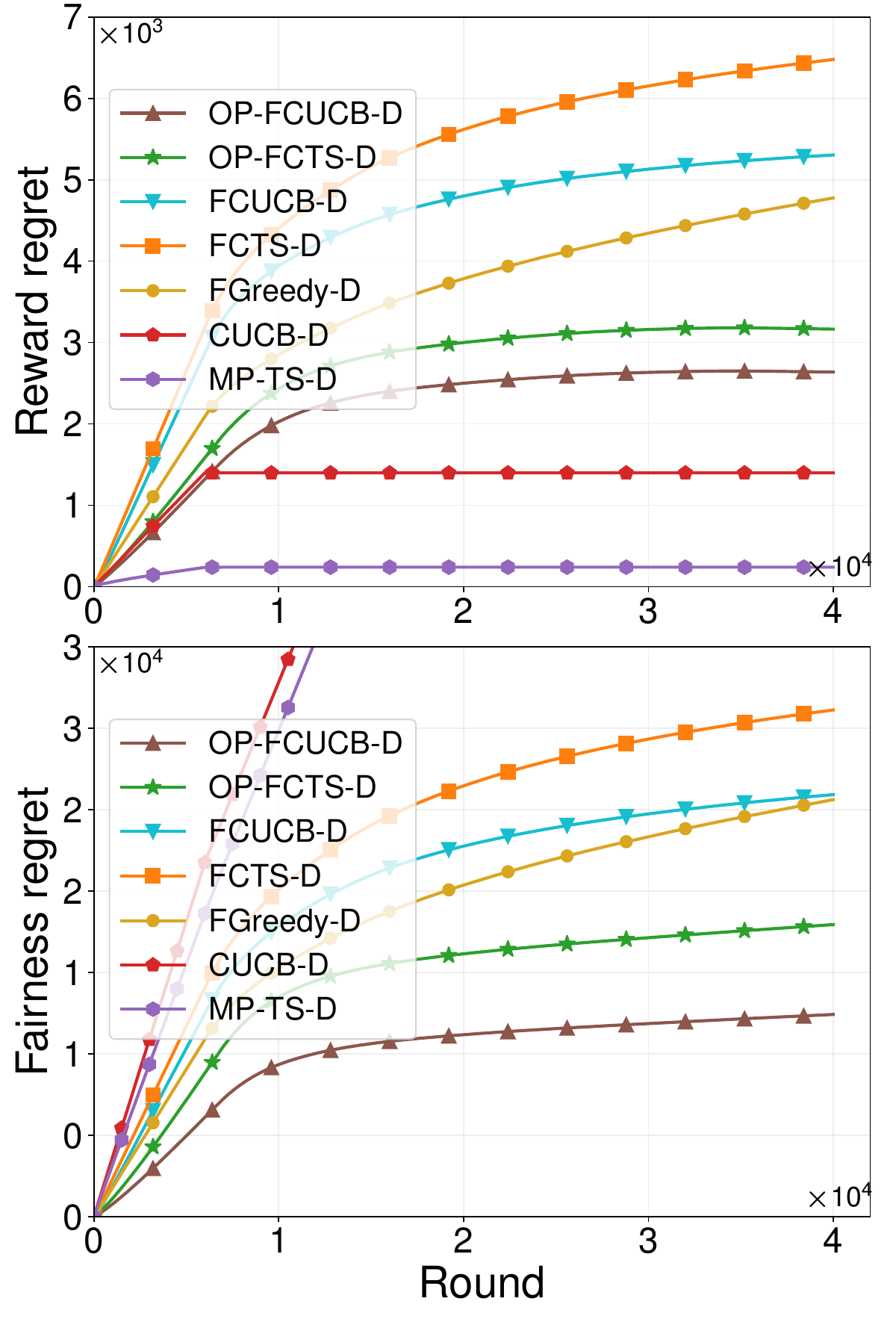}
		\label{fig:biased_delays_slope_c=6}
	}
 	\subfigure[$c=8$]{
		\centering
		\includegraphics[width=0.23\textwidth]{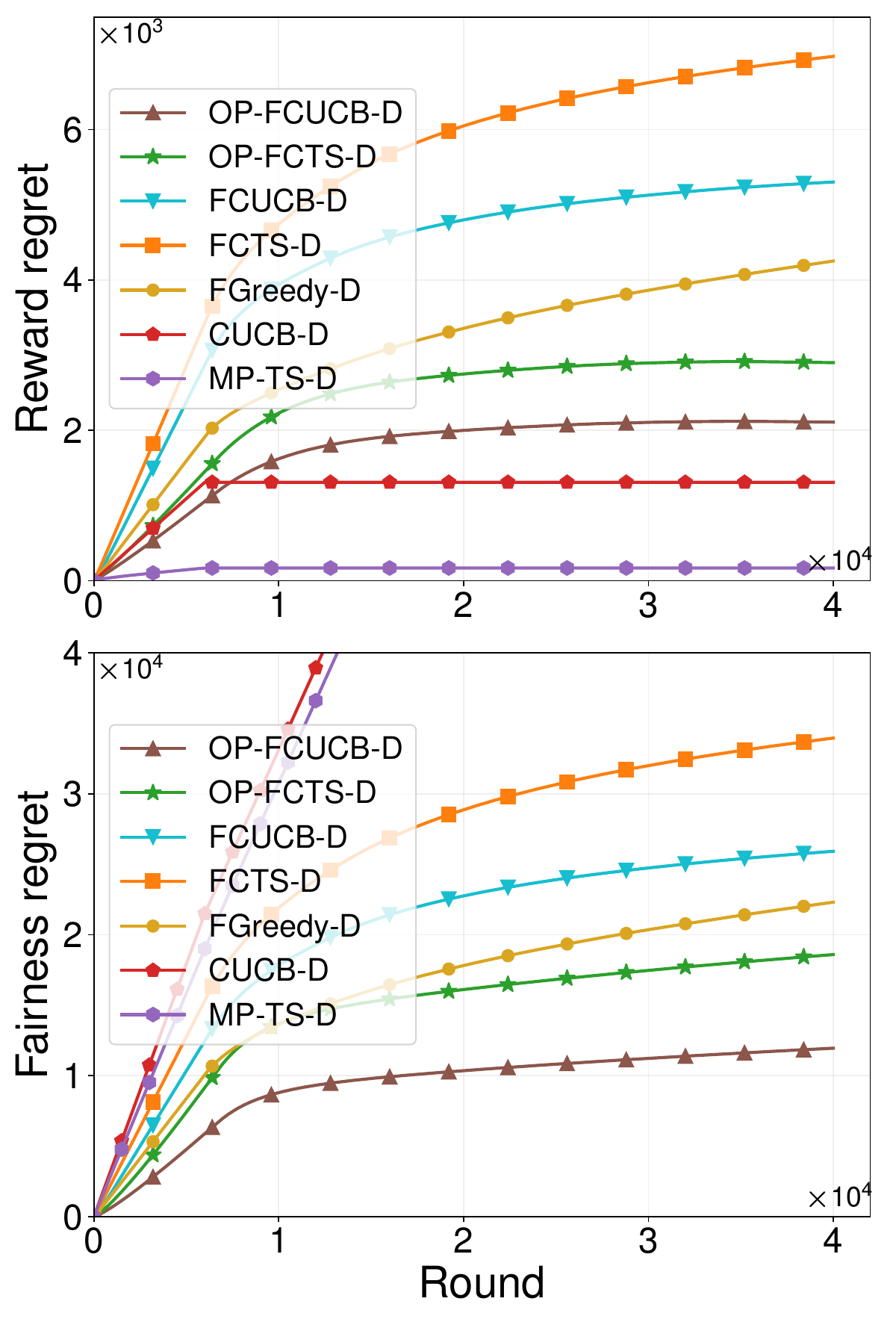}
		\label{fig:biased_delays_slope_c=8}              
	}
	\caption{Experiment results of the different bandit algorithms using different merit functions under biased feedback delays.}
	\label{fig:experiments_on_baised_delay_slope}
\end{figure*}

\begin{figure*}[!t]
\centering
\subfigure[Reward Regret]{
\centering
\includegraphics[width=0.35\textwidth]{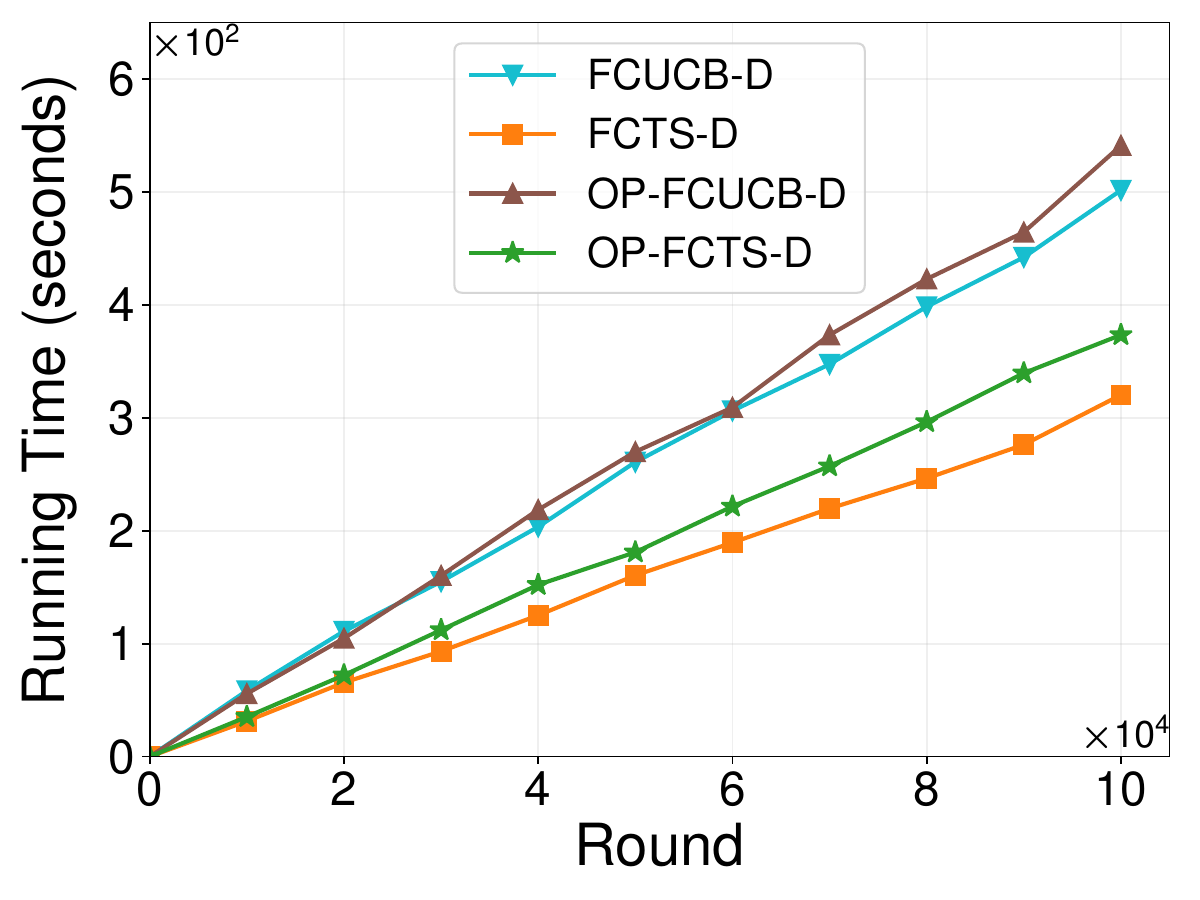}
	}	
	\caption{Running time of the different bandit algorithms.}
    		\label{fig:running-time}
\end{figure*}

\end{document}